\documentclass[11pt]{article}
\input{macros}
\usepackage{mathrsfs}
\usepackage{todonotes}
\allowdisplaybreaks

\begin{document}

\title{Replica Symmetry Breaking and
Algorithmic Thresholds in\\
Empirical Risk Minimization 
under Multi-Index Model}
\author{Andrea Montanari\thanks{Department of Mathematics and Department of Statistics, 
		Stanford University} \,\, and \,\, Kangjie Zhou\thanks{Department of Statistics, 
		Columbia University}}
\date{\today}
\maketitle

\begin{abstract}
Modern machine learning models are trained
by optimizing high-dimensional non-convex empirical risk functions.
Such cost functions can have a multitude of local
optima and yet, gradient-based optimization appears to converge to  
near-global optima.

Within a simple supervised learning setting,
we develop a precise picture of which parts of the empirical risk 
landscape are accessible by polynomial-time algorithms. 
We are given i.i.d. pairs
$\{(\xx_i,y_i):\; 1 \le i\le n\}$ with 
$\xx_i\in \R^d$ standard Gaussian feature vectors, and $y_i\in\R$ response 
variables that depend on $\xx_i$ through their projections
on an unknown $k$-dimensional subspace. We use empirical
risk minimization to learn a model that depends on an $m$-dimensional 
projection of the data (e.g., an $m$-neurons neural network).

We propose an incremental approximate message passing (IAMP) algorithm
and precisely characterize the training error it achieves, as well as the relation
between test and training error, in the high dimensional asymptotics $n,d\to\infty$,
with $n/d\to\alpha \in (0, +\infty)$. Based on earlier work in related models,
we expect that the performance achieved by our algorithm is optimal among polynomial-time algorithms. 
\end{abstract}

\tableofcontents

\section{Introduction}
Empirical risk minimization (ERM) is the primary guiding principle in the design of learning algorithms. However, its theoretical properties are still insufficiently understood in the regime where the model complexity\footnote{Characterized by the number of parameters or the Rademacher complexity \cite{shalev2014understanding}.} scales proportionally with the sample size, while this scaling is precisely the one privileged by current AI trends \cite{kaplan2020scaling,hoffmann2022training}.

In the case of convex empirical risk functions (namely,
linear models with convex losses), a rich theory allows us to derive sharp
 asymptotic results under simple data distributions 
 \cite{bayati2011lasso,thrampoulidis2015regularized,el2018impact,donoho2016high}. 
 Recent work has developed methods to handle non-convexity when the 
 empirical risk is locally convex in a neighborhood of the global minimizer, 
 and there is only one or a small number of near-global minima \cite{asgari2025local,vilucchio2025asymptotics,montanari2026topological}.
 However, the empirical risk landscape  of interest in machine learning is often
 qualitatively different, with many well-separated local minima.

An example of this type arises in the context of tensor PCA, 
where one tries to estimate a high-dimensional rank-one tensor, from
observations corrupted by Gaussian noise \cite{montanari2014statistical}.
In this case, the negative log-likelihood function is a Gaussian process 
$\hL_n(\ww)$ indexed by the parameter vector $\ww\in\S^{d-1}$,
and can present an exponential number of local minima 
\cite{arous2019landscape}. This type of cost functions is also at the core
of the theory of mean-field spin glasses \cite{auffinger2023optimization}.

  Since global optimization of empirical risk functions with
  many local minima is out of reach of polynomial-time algorithms
  in the worst case, 
  we instead address the following central question:
\begin{itemize}\label{Q:train_test_err}
   \item[{\sf Q}] \textit{Which values of the training and test error are achievable by 
   polynomial-time algorithms with high probability with respect to a given data distribution?}
\end{itemize}

To be definite, consider a simple setting where the observed i.i.d. data $\{ (\xx_i, y_i) \}_{i=1}^{n}$ follow a Gaussian single-index model:
\begin{equation}\label{eq:single_index_model_first}
    y_i= \sqrt{\lambda}\, \varphi(\ww_*^{\sT}\xx_i)+\eps_i\, ,\;\;\ \xx_i \sim \normal(\bzero, \id_d), \,\, \eps_i\sim\normal(0,1)\, ,
\end{equation}
where $\varphi$ is the link function, 
$\ww_* \in \S^{d-1}$ is the true signal, and 
$\lambda > 0$ represents the signal-to-noise ratio. We will be interested in the proportional scaling regime where $n, d \to \infty$ such that $n/d \to \alpha \in (0, \infty)$. While our main results consider more general models, the above setup is already rich enough to illustrate the key points.

To learn the true signal $\ww_*$, we propose to solve the following correlation maximization problem
over the unit sphere $\S^{d-1}$:
 \begin{equation}\label{eq:single_index_model_second}
	\begin{aligned}
    \mbox{maximize}\;\;\;&
	\hCorr_n(\ww) := \frac{1}{n \sqrt{\lambda}} 
	\sum_{i=1}^n y_i \sigma(\ww^{\sT}\xx_i),\\
	\mbox{subject to}\;\;\;& \ww \in \S^{d-1},
	\end{aligned}
 \end{equation}
 for a given activation function $\sigma$. Also in this respect, our main results
 are more general and we consider correlation loss here only as an illustration.
 On the other hand, we do not assume $\sigma=\varphi$
 because---in general---the data generating process is unknown.
 We are particularly interested in characterizing the training 
 and test errors achievable by efficient iterative algorithms, 
 such as gradient-based methods and approximate message passing (AMP). 

We pause to note that many special cases of the model~\eqref{eq:single_index_model_first}, and ERM methods analogous to \eqref{eq:single_index_model_second}
have been extensively investigated in the literature. See, e.g., \cite{chen2017solving,mondelli2018fundamental,aubin2018committee,barbier2019optimal}. Recall that $\alpha = \lim_{n, d \to \infty} n / d$ is the limiting aspect ratio. Existing results in literature \cite{vilucchio2025asymptotics,montanari2026topological,maillard2026topological} 
suggest the following qualitative picture of the landscape of $\hCorr_n(\ww)$:
\begin{enumerate}
	\item For $\alpha>\alpha_{\str}(\lambda,\varphi,\sigma)$, the empirical correlation 
	$\hCorr_n(\ww)$ admits a unique global maximizer $\hww \in \S^{d-1}$. Furthermore, $\hww$
	 is positively correlated with the true signal $\ww_*$, with 
	$\hCorr_n(\ww)$ being strongly concave in a neighborhood of $\hww$. This local maximum is accessible by
	suitable polynomial-time 
	algorithms that only use gradient or Hessian information of $\hCorr_n$, such as AMP.
	\item For $\alpha<\alpha_{\str}(\lambda,\varphi,\sigma)$, either 
	 all near-global maximizers are roughly orthogonal to $\ww_*$,
    or a global maximizer exists and is positively correlated with 
	$\ww_*$, but it cannot be found by polynomial-time algorithms 
	or does not have a strongly concave neighborhood. 
\end{enumerate}
It is also possible that $\alpha_{\str}=0$ or $\alpha_{\str}=\infty$.
We refer to the phase transition at $\alpha_{\str}$ as the `trivialization' phase
transition.  In statistical physics terminology, when the landscape presents multiple
well-separated near-optima, with roughly independent random values
(as is the case for $\alpha<\alpha_{\str}$), one speaks of `replica symmetry breaking.'

 Existing results successfully characterize the Bayes optimal 
 estimation error or analyze ERM for $\alpha>\overline{\alpha}_{\str}$ 
 (with $\overline{\alpha}_{\str}$ an upper bound on the trivialization
 threshold $\alpha_{\str}$) \cite{barbier2019optimal,vilucchio2025asymptotics,asgari2025local,montanari2026topological}. However, these works fall short of answering Question~{\sf Q} 
 in the most challenging regime $\alpha<\alpha_{\str}$.

\begin{figure}[!ht]
    \centering
    \includegraphics[width=0.9\linewidth]{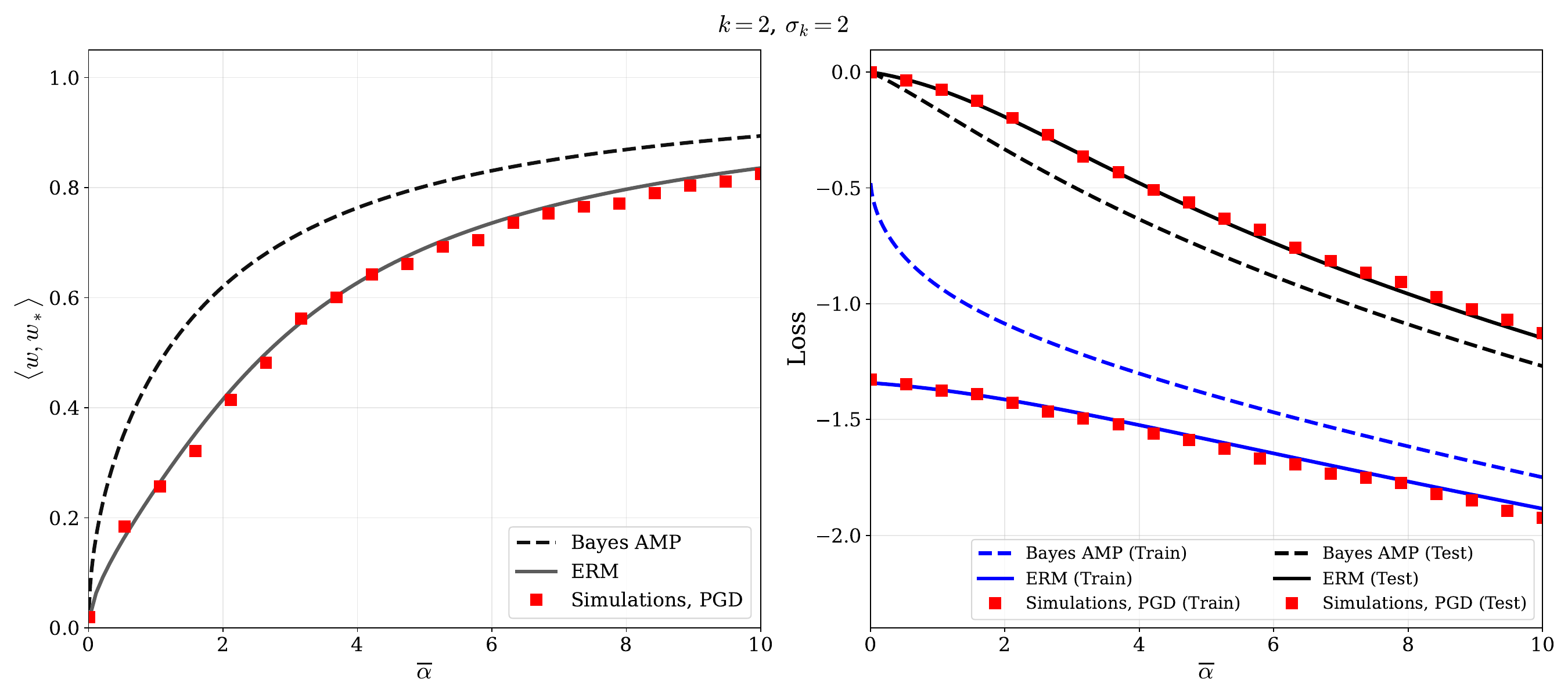}
    \caption{\textbf{Left panel:} The limiting overlap with $\ww_*$ for Bayes AMP, 
    ERM (value achieved by our two-stage algorithm),
    and projected gradient descent (PGD). \textbf{Right panel:} The limiting training error (normalized empirical correlation) and test error (normalized population correlation) for these three methods. These values are computed under the limit where $\alpha \to \infty$ and $\lambda \to 0$ such that $\alpha \lambda \to \oalpha$. See \cref{sec:Applications} for a comprehensive description of the experimental setup and implementation details.}
    \label{fig:misspecified_intro}
\end{figure}

To illustrate the content of  Question~{\sf Q}, 
Figure~\ref{fig:misspecified_intro} reports the results of a numerical simulation. 
We consider the ERM problem \eqref{eq:single_index_model_second} with
data generated using $\varphi(x)= \operatorname{ReLU} (x) = \max (x, 0)$ and learning with the misspecified activation function 
$\sigma(x) = x + \sqrt{2} (x^2-1)$.
We evaluate the performance of three algorithms: 
$(i)$ Bayes AMP; $(ii)$ ERM computed via our two-stage algorithm, 
which we expect to be optimal among polynomial-time algorithms
(see \cref{sec:main_results} for details); and $(iii)$~Projected gradient 
descent (PGD) minimizing $-\hCorr_n$ (maximizing $\hCorr_n$) over the unit sphere. 
For $(i)$ we use the theoretical predictions from  \cite{barbier2019optimal},
while for $(ii)$ we plot the predictions from the present work.
Finally, for $(iii)$ we plot the results of numerical experiments.

We plot suitably normalized versions of train and test errors
(see \cref{sec:single_index_example} for details), as well as the correlation 
with the true signal $\ww_*$ attained by each of the three methods.
We observe that numerical simulations with PGD are closely matched by the theoretical 
predictions for the optimal two-stage algorithm. In other words, our theory appears
to capture the behavior of suitably tuned gradient descent.
On the other hand, a substantial gap exists between the test error of
ERM based methods (either PGD or our two-stage algorithm) and that of the Bayes optimal AMP.

\subsection{Technical setting: Multi-index model and ERM}
In the remainder of this paper, we aim to answer Question~{\sf Q}
 by delineating the computationally feasible training and test errors for a broad class of 
 iterative algorithms, under the following Gaussian multi-index model.
\begin{ass}\label{ass:linear_signal}
	$\{ (\xx_i, y_i) \}_{i \in [n]}$ are i.i.d. data such that for each $i \in [n]$, 
	$\xx_i \sim \normal(\bzero, \id_d)$, $y_i = \varphi (\WW_*^\sT \xx_i, \veps_i)$, 
	where $\WW_* \in \R^{d \times k}$ is a deterministic signal matrix and $\varphi: \R^k \times \R \to \R$ is a deterministic link function. Without loss of generality, we assume that $\WW_* \in O(d, k)$, the space of $d\times k$ orthogonal matrices (cf. \cref{lem:Reduction_Signal}). Further, denoting $\yy = (y_i)_{i \in [n]}$, $\vveps = (\veps_i)_{i \in [n]}$ and $\XX = [\xx_1, \cdots, \xx_n]^\top$, we assume that the noise vector $\vveps$ is independent of the data matrix $\XX$, with $\veps_i \sim_{\iid} P_{\veps}$ such that the $y_i$'s are sub-Gaussian.
\end{ass}

Given $h: \R \times \R^m \to \R$ a bounded Lipschitz function, we aim to characterize
the achievable values of random optimization problems of the form
\begin{align}\label{eq:GeneralOpt_Supervised}
\mbox{maximize} \quad \frac{1}{n} \sum_{i=1}^{n} h \big( y_i, \WW^\sT \xx_i \big), \quad \mbox{subject to} \ \WW \in O(d, m),
\end{align}
in the proportional asymptotics where
$n/d \to \alpha \in (0, \infty)$ as $n, d \to \infty$, 
while $m, k = O(1)$. We will be interested both in the 
asymptotics of the maximum value (possibly subject to constraints on $\WW^{\sT}\WW_*$),
and on the maximum value achievable by polynomial-time algorithms.
Note that, although in \cref{eq:GeneralOpt_Supervised} we only optimize over orthogonal matrices, general inner product structures of $\WW$ can be encoded into the definition of $h$.

In statistical learning
applications, one normally takes $h = - \ell$ for some
 loss function $\ell$, and minimizes the empirical risk
\begin{equation}\label{eq:ERM_definition}
	\hat{L}_n (\WW) = \, \frac{1}{n} \sum_{i=1}^{n} \ell \big( y_i, \WW^\sT \xx_i \big), \quad \mbox{subject to} \ \WW \in O(d, m).
\end{equation}
Obviously, this minimization problem is entirely equivalent to the maximization problem~\eqref{eq:GeneralOpt_Supervised}.

\begin{rem}[Two-layer neural networks]
	As a special case of the ERM problem \eqref{eq:ERM_definition},
we consider training a two-layer neural network with squared loss. This example corresponds to minimizing the following empirical risk function
\begin{align}
\hat{L}_n (\UU;\aa) = 
\frac{1}{2n}\sum_{i=1}^n\left(y_i-\sum_{j=1}^m a_j \sigma(\< \uu_j, \xx_i \>) \right)^2 \, 
\end{align}
over parameters $a_1,\dots,a_m\in\R$ and $\uu_1,\dots,\uu_m\in\R^d$.
This can be performed in two steps: $(i)$~Optimization over $\WW\in O(d,m)$, an
orthogonal basis for the span of $\uu_1,\dots,\uu_m\in\R^d$; $(ii)$~Optimization over
$a_1,\dots,a_m\in\R$, as well as over the representation of $(\uu_j)_{j \le m}$ in the basis $\WW$,
namely over $\bb_1, \cdots, \bb_m \in \R^m$ such that $\uu_j = \WW \bb_j$ for $j = 1, \cdots, m$.

The inner optimization (optimizing over $\WW$ at fixed $\aa, \BB:=(\bb_j)_{j \le m}$) corresponds to
the setting of our paper with the following choice of $\ell$:
\begin{align}
	\ell\big( y;\zz \big) =\frac{1}{2} \left( y-\sum_{j=1}^m a_j \sigma(\<\bb_j, \zz\>) \right)^2
	\, .
\end{align}
 We note that the outer step is substantially easier both from an optimization viewpoint, because it is a constant-dimensional optimization problem, and from a 
statistical viewpoint. Indeed, by Gaussian concentration, the optimal cost at fixed $(\aa,\BB)$ converges uniformly to its expectation, as $n,d\to\infty$.
Hence, once we characterize the asymptotics of the inner optimization problem,
the overall result follows by optimizing over the low-dimensional parameters $\aa,\BB$.
\end{rem}

\subsection{Connection with projection pursuit}

The high-dimensional asymptotics of the random optimization problem 
\eqref{eq:GeneralOpt_Supervised} is intimately related to 
projection pursuit, which aims to find low-dimensional projections that are `informative'
or `atypical' \cite{friedman1974projection,diaconis1984asymptotics,montanari2022overparametrized}. 
In order to formalize the projection pursuit problem, and clarify its relation to
empirical risk minimization, let us define the set of feasible empirical joint distributions of the $y_i$'s and low-dimensional projections of the $\xx_i$'s as
\begin{equation}\label{eq:Feasible_Supervised}
\begin{split}
	\cuF_{m,\alpha,\varphi}:= \Big\{P \in \cuP (\R \times \R^{m}):\, & \exists
	\WW = \WW_{n, d} (\XX, \yy, \omega) \in O(d, m), \,\,\\
	& \mbox{s.t.} \, \, \frac{1}{n} \sum_{i=1}^{n} \delta_{(y_i, \WW^\sT \xx_i)} \stackrel{w}{\to} P\, \mbox{ in probability}
	\Big\}\, ,
\end{split}
\end{equation}
where $\omega$ represents some additional randomness independent of $(\XX, \yy)$ 
(without loss of generality, we take $\omega\sim\Unif([0,1])$),
and $\cuP (\R \times \R^{m})$ is the set of all 
probability measures on $\R \times \R^m$. 

We emphasize one important feature of the feasible set $\cuF_{m,\alpha,\varphi}$ defined in \eqref{eq:Feasible_Supervised},
namely $\WW= \WW_{n, d} (\XX, \yy, \omega)$ does not depend on $\WW_*$ and hence is
a true statistical estimator.

To establish the connection between the random optimization problem~\eqref{eq:GeneralOpt_Supervised} and the feasible set $\cuF_{m,\alpha,\varphi}$, we define the dual functional $\VH_{m,\alpha,\varphi} (\cdot)$ of $\cuF_{m,\alpha,\varphi}$ as
\begin{align}\label{eq:Dual_Functional}
	\VH_{m, \alpha, \varphi}(h) := \sup_{P \in \cuF_{m, \alpha, \varphi}} \left\{ \int_{\R \times \R^m} h(y, z) P(\d y, \d z) \right\}, \quad \forall h \in C_{b} (\R \times \R^m),
\end{align}
where $C_{b} (\R \times \R^m)$ denotes the collection of bounded continuous functions on $\R \times \R^m$. According to \cite[Proposition 4.1]{montanari2022overparametrized}, for any $h \in C_{b} (\R \times \R^m)$,
\begin{equation}\label{eq:Opt_PP_Equivalence}
\pliminf_{n/d \rightarrow \alpha} \max_{\WW \in O(d, m)} \frac{1}{n} \sum_{i=1}^{n} h \left( y_i,  \WW^\sT \xx_i \right) = \VH_{m, \alpha, \varphi}(h).
\end{equation}
Therefore, characterizing the feasible set $\cuF_{m,\alpha,\varphi}$
would allow us to determine $\VH_{m,\alpha,\varphi}(h)$, the asymptotics of the global maximum for all problems of the form 
\eqref{eq:GeneralOpt_Supervised}. On the other hand, a direct application of geometric Hahn-Banach theorem implies that, determining $\VH_{m,\alpha,\varphi}(h)$ for all $h \in C_{b} (\R \times \R^m)$ provides a complete characterization of the closed convex hull of $\cuF_{m,\alpha,\varphi}$ (cf. \cite[Theorem 1.1]{montanari2024exceptional}).

As emphasized above, we are equally interested in the values of the optimization problem \eqref{eq:GeneralOpt_Supervised} that can be achieved in polynomial time. 
Correspondingly, we define $\cuF_{m,\alpha,\varphi}^{\salg}$, the subset of $\cuF_{m,\alpha, \varphi}$ that can be realized via polynomial-time computable projections, namely
\begin{equation}\label{eq:Feasible_Supervised_Alg}
\begin{split}
	\cuF^{\salg}_{m, \alpha, \varphi} := \Big\{P \in \cuP (\R \times \R^{m}):\, & \exists
	\WW = \WW_{n, d} (\XX, \yy, \omega) \in O(d, m) \mbox{ polytime computable}, \\
 & \mbox{s.t.} \, \,
	\frac{1}{n} \sum_{i=1}^{n} \delta_{(y_i, \WW^\sT \xx_i)} \stackrel{w}{\to} P\, \mbox{ in probability}
	\Big\}\, .
\end{split}
\end{equation}
In the above display, $\WW_{n, d} (\XX, \yy, \omega)$ is `polytime computable' if there exists an algorithm accepting $(\XX, \yy, \omega)$ (or its finite-precision approximation) as input and computing $\WW_{n, d} (\XX, \yy, \omega)$ in time at most polynomial in $n, d$. The associated dual functional is then defined as
\begin{align}\label{eq:Dual_Functional_Supervised}
	\VH_{m, \alpha, \varphi}^{\salg} (h) := \sup_{P \in \cuF_{m, \alpha, \varphi}^{\salg}} \left\{ \int_{\R \times \R^m} h(y, z) P(\d y, \d z) \right\}, \quad \forall h \in C_{b} (\R \times \R^m).
\end{align}
Analogous to \cref{eq:Opt_PP_Equivalence}, there is a direct correspondence between 
precisely characterizing the algorithmically feasible set 
$\cuF^{\salg}_{m, \alpha, \varphi}$ and exactly determining the values of 
the optimization problem~\eqref{eq:GeneralOpt_Supervised} attainable via 
polynomial-time algorithms.

\subsection{General approach}

Over the last few years, a general approach has emerged to 
characterize the optimal value achieved by polynomial-time algorithms
in certain classes of random optimization problems. This approach 
has two components (see also \cite{auffinger2023optimization,montanari2024optimization} for reviews):
\begin{description}
\item[Achievability.] A rich class of first-order optimization 
algorithms can be constructed using 
approximate message passing (AMP) or related strategies 
\cite{celentano2020estimation,subag2021following,montanari2021optimization,el2021optimization,montanari2024exceptional}.
Designing the optimal algorithm in this class is equivalent to solving a certain
(low dimensional) stochastic optimal control  problem. This
optimal control problem is dual to a Parisi-style formula, thus 
connecting directly to fundamental concepts in spin glass theory. Solving 
the optimal control problem yields the predicted algorithmic threshold. 
\item[Impossibility.] Matching impossibility results are proven
for algorithms that act as Lipschitz continuous functions of their input 
(in our case, $(\XX,\yy)$), with a bounded Lipschitz constant under suitable normalization.
The proof technique relies on establishing that near-optima exhibit a certain geometric
property---the `overlap gap property'---that cannot be 
reproduced by the output of Lipschitz algorithms.
This approach was initiated in
\cite{gamarnik2014limits}, applied to spin glass models related to the present setting in 
\cite{gamarnik2021overlap}, and a sharp version 
based on the `branching overlap gap property' was recently developed in
\cite{huang2022tight,huang2024optimization}. 
\end{description}
In this paper, we carry out the first part of this general program for the problem of learning a multi-index model. As the first statistical application of these techniques, our setting presents several distinctive challenges. Most notably, we must account for the fact that we only have access to the empirical data $(\XX, \yy)$, and not the true underlying parameters $\WW_*$.

\subsection{Summary of main results and paper organization}
The main contributions of this paper are summarized below:
\begin{description}

\item[Conjectures from statistical physics.] In Section \ref{sec:replica_method}, we use the non-rigorous replica method from statistical physics to derive a 
Parisi-type variational principle, which jointly predicts the feasible set 
$\cuF_{m, \alpha, \varphi}$ and $\VH_{m, \alpha, \varphi} (h)$, the asymptotic value of~\eqref{eq:GeneralOpt_Supervised}. 
These predictions are formalized in 
Conjectures~\ref{conj:Parisi_formula_mdim} and \ref{conj:Parisi_formula_1dim},
and serve as a useful benchmark for our rigorous results.

\item[Algorithmic achievability.] In Section \ref{sec:main_results}, we present our two-stage AMP algorithm 
and rigorously characterize its asymptotics when applied to the 
problem \eqref{eq:GeneralOpt_Supervised}. This analysis yields an explicit 
inner bound on the set of feasible distributions
$\cuF^{\salg}_{m, \alpha, \varphi}$ (\cref{thm:AMP_Feasibility_general}),
which we expect to be tight over all polynomial-time algorithms.

\item[Asymptotics of test and train error.] In particular, we precisely characterize the asymptotic training and
   test errors achievable by the two-stage AMP algorithm for the 
   ERM of Eq.~\eqref{eq:ERM_definition}, in the high-dimensional proportional limit 
   $n, d \to \infty$ with $n / d \to \alpha$ (see \cref{rem:test_error_AMP_ERM}).

\item[Large $n/d$ and tensor PCA equivalence.] As a special case, 
in Section \ref{sec:Applications} we derive simple expressions
for the algorithmically achievable training and test errors in 
the single-index model with correlation loss (cf. \cref{eq:single_index_model_first}), under the double
asymptotics $\alpha\to\infty$ and $\lambda\to 0$.
Interestingly, our results for the single-index model admit a natural
 interpretation via an equivalent generalized tensor PCA problem, which 
 we elaborate in \cref{sec:Gaussian_model_setup}.
 \item[Duality and Parisi-style formulas.] Our algorithmic inner bound on $\cuF^{\salg}_{m, \alpha, \varphi}$ yields a corresponding lower bound on the dual functional $\VH_{m, \alpha, \varphi}^{\salg} (h)$, which we denote by $\VH_{m, \alpha, \varphi}^{\sAMP} (h)$. Namely, $\VH_{m, \alpha, \varphi}^{\sAMP} (h)$ is the limiting optimal value of the optimization problem~\eqref{eq:GeneralOpt_Supervised} achieved by our two-stage AMP algorithm. 
 In the case $k = m = 1$, we
  establish a Parisi-type variational principle for 
  $\VH_{1, \alpha, \varphi}^{\sAMP} (h)$ over an extended function space 
  (\cref{thm:two_stage_strong_duality}). 
\end{description}

\cref{sec:iamp} formally defines the two-stage AMP algorithm and provides the proof of \cref{thm:AMP_Feasibility_general}, establishing our general AMP achievability results for both $\cuF^{\salg}_{m, \alpha, \varphi}$ and $\VH_{m, \alpha, \varphi}^{\salg} (h)$. In \cref{sec:char_opt_ctrl}, we present the necessary technical preliminaries regarding the Parisi functional and prove \cref{thm:two_stage_strong_duality}, the extended variational principle for $\VH_{m, \alpha, \varphi}^{\sAMP} (h)$. The proofs of all auxiliary lemmas and additional technical details are deferred to the appendices.

\subsection{Definitions and notation}


We will follow the convention of using boldface letters for matrices or vectors whose
dimensions diverge as $n,d\to\infty$, and normal fonts otherwise.
We denote the standard scalar product between two vectors $\uu,\vv$ by $\<\uu,\vv\>$,
and the matrix scalar product by $\<\AA,\BB\> = \Tr(\AA^\sT\BB)$. We use $\norm{\cdot}_2$ to denote the Euclidean norm of a vector. We use $\norm{\cdot}_{L^p}$ to denote the standard $L^p$ norm of a function for $p \in [1, +\infty]$.

We denote by $\Sym_+^m$ the convex cone of $m\times m$ positive semi-definite matrices. For $d \ge m$, we denote by $O(d, m)$ the set of all $d \times m$ orthogonal matrices. For a subset $S$ in a topological space, we denote by $\close(S)$ its closure.
For $p \ge 1$, we denote by $C^k (\R^p)$ the collection of all functions that have continuous $k$-th derivatives in $\R^p$. We also denote by $C_b (\R^p)$ the set of all bounded continuous functions on $\R^p$, and by $C_{c}^{\infty} (\R^p)$ the set of all infinitely differentiable functions with compact supports. We use $\cuP (\R^p)$ to denote the set of all probability measures on $\R^p$ equipped with the topology of weak convergence, unless otherwise stated. For a sequence of probability measures $(\nu_n)_{n \ge 1}$ and a probability measure $\nu$, we write $\nu_n \stackrel{w}{\to} \nu$ if $\nu_n$ converges weakly to $\nu$. We use $\Law (X)$ to denote the law of a random variable $X$. For two random variables $X$ and $Y$, we write $X \perp\!\!\!\perp Y$ if $X$ is independent of $Y$.

For $d \in \mathbb{N}$ and $r \ge 0$, $B_d (r) = \{ x \in \R^d: \norm{x}_2 \le r \}$. For a function $h$, we denote by $\operatorname{conc} (h)$ the (upper) concave envelope of $h$. Namely, $\operatorname{conc} (h)$ is the pointwise minimum of all concave functions that dominate $h$. For $l, k \ge 1$, and a differentiable mapping $F: \R^k \to \R^l$, we denote by $J_F \in \R^{l \times k}$ the Jacobian matrix of $F$, namely for $x \in \R^k$: $J_F(x)_{ij} = \partial F_i / \partial x_j$. We occasionally use $J_F$ as a shorthand for $J_F (x)$ whenever the variable $x$ is clear from the context. We say that a function $\psi:\R^{p}\to\R$ is pseudo-Lipschitz if there exists a constant $C$ such that, 
for all $x,y\in \R^p$, 
$$
|\psi(x)-\psi(y)|\le C(1+\|x\|_2+\|y\|_2)\|x-y\|_2.
$$
Let $\{ B_t \}_{t \in [0, 1]}$ be an $m$-dimensional standard Brownian motion, and let $\{ \cF_t \}_{t \in [0, 1]}$ be its canonical filtration. For $s \le t$, we denote by $D[s, t]$ the space of all admissible controls on the interval $[s, t]$, i.e., the collection of all progressively measurable processes $\{ \Phi_r \}_{s \le r \le t}$ satisfying
\begin{equation}
	\sigma(\Phi_r) \subset \cF_r, \ \forall r \in [s, t], \ \mbox{and} \ \E \left[ \int_{s}^{t} \Phi_r \Phi_r^\sT \d r \right] < \infty.
 \label{eq:Ddef}
\end{equation}

%
%
\section{Conjectures from statistical physics}\label{sec:replica_method}

In this section, we derive a prediction for $\VH_{m, \alpha, \varphi} (h)$  using physicists' replica method, with detailed calculations deferred to Appendix~\ref{append:replica}.
Recall that $\VH_{m,\alpha,\varphi}(h)$ is defined in Eq.~\eqref{eq:Dual_Functional}.
\begin{conj}[Replica prediction for $\VH_{m,\alpha,\varphi}(h)$]\label{conj:Parisi_formula_mdim}
	For any $h \in C_b (\R \times \R^m)$, almost surely
	\begin{equation*}
		\lim_{n \to \infty} \max_{\WW \in O(d, m)} \frac{1}{n} \sum_{i=1}^{n} h \left( y_i, \WW^\sT \xx_i \right) = \VH_{m,\alpha,\varphi}(h)\, .
	\end{equation*}
	Further, $\VH_{m,\alpha,\varphi}(h)$ has the following variational representation:
	\begin{equation*}
		\VH_{m,\alpha,\varphi}(h) = \sup_{R \in \mathscr{B}_{k, m}} \inf_{(\mu, M, C) \in \mathscr{U} \times \mathscr{I}_m (R) \times \Sym_+^m} \mathsf{F}_m (\mu, M, C, R).
	\end{equation*}
	In the above display, $\Sym_+^m$ denotes the convex cone of $m \times m$ positive semi-definite matrices, and
	\begin{align*}
		\mathscr{B}_{k, m} &:= \left\{ R \in \R^{k \times m} : R^\sT R \preceq I_m \right\}, \\
		\mathscr{I}_m (R) &:= \left\{ M: [0, 1) \to \Sym_+^m : \int_{0}^{1} M(t) \,\d t = I_m - R^\sT R \right\},\\
	    \mathscr{U} &:= \left\{ \mu: [0, 1) \to \R_{\ge 0}: \ \mu \ \mbox{{\rm non-decreasing}}, \ \int_{0}^{1} \mu(t) \d t < \infty \right\}.
	\end{align*}
	The $m$-dimensional Parisi functional $\mathsf{F}_m$ is defined as
	\begin{equation*}
		\mathsf{F}_m (\mu, M, C, R) = \, \E_{Y, G} \left[ f_{Y, \mu} \left( 0, R^\sT G \right) \right] + \frac{1}{2 \alpha} \int_{0}^{1} \Tr \left( M (t) \left( C + \int_{t}^{1} \mu(s) M (s) \d s \right)^{-1} \right) \d t,
	\end{equation*}
	where the random vector $(Y, G)$ is distributed as follows:
	\begin{equation*}
		G \sim \normal(0, I_k), \quad Y = \varphi(G, \veps), \, \veps \sim P_{\veps}, \, \veps \perp\!\!\!\perp G.
	\end{equation*}
	Further, for any fixed $(y, \mu)$, $f_{y, \mu}: [0, 1] \times \R^m \to \R$ solves the following $m$-dimensional Parisi PDE:
	\begin{equation}\label{eq:Parisi_mdim}
	\begin{split}
		\partial_t f_{y, \mu} (t, x) + \, & \frac{1}{2} \mu (t)\<\nabla_{x} f_{y, \mu} (t, x), M (t) \nabla_{x} f_{y, \mu} (t, x)\> +\frac{1}{2} \operatorname{Tr}\left( M (t) \nabla_{x}^2 f_{y, \mu} (t, x) \right) = \, 0, \\
		f_{y, \mu} (1, x) = \, & \sup_{u \in \R^m} \left\{ h \left( y, \, x + u \right) - \frac{1}{2} \<u, C^{-1}u\> \right\}\, .
	\end{split}	
	\end{equation}	
    (For notational simplicity, we omit the dependence of $f_{y, \mu}$ on $(M, C)$.)
\end{conj}
For $m = 1$, the replica formulas in \cref{conj:Parisi_formula_mdim} reduce to the following simpler form:
\begin{conj}[Replica prediction for $\VH_{1,\alpha,\varphi}(h)$]\label{conj:Parisi_formula_1dim}
	For any $h \in C_{b} (\R \times \R)$, almost surely
	\begin{equation*}
		\lim_{n \to \infty} \max_{\ww \in \S^{d-1}} \frac{1}{n} \sum_{i=1}^{n} h \left( y_i, \langle \ww, \xx_i \rangle \right) = \VH_{1,\alpha,\varphi}(h) \, .
	\end{equation*}
	Further, $\VH_{1,\alpha,\varphi}(h)$ has the following variational representation:
	\begin{align*}
    	\VH_{1,\alpha,\varphi}(h) = \sup_{r \in B_k(1)} \inf_{(\mu, c) \in \mathscr{U} \times \R_{>0}} \mathsf{F}_1 (\mu, c, r).
	\end{align*}
	In the above display, the Parisi functional $\mathsf{F}_1: \mathscr{U} \times \R_{>0} \times B_k (1) \to \R$ is defined as follows: Let $(Y, G)$ be such that $G \sim \normal(0, I_k)$, $Y = \varphi(G, \veps)$ where $\veps \sim P_{\veps}$ is independent of $G$. For any fixed $y$ and $\mu$, let $f_{y, \mu} (t, x)$ be the solution to the following Parisi PDE:
	\begin{equation}\label{eq:ParisiPDE_1dim_reformulated}
		\begin{split}
			&\partial_t f_{y, \mu} (t, x)+\frac{1}{2} \mu (t) (\partial_x f_{y, \mu} (t, x))^2 +  \frac{1}{2} \partial_x^2 f_{y, \mu} (t, x) = \, 0, \\
			&f_{y, \mu} (1, x) = \,  \sup_{u \in \R} \left\{ h \left( y, x + u \right) - \frac{u^2}{2c}  \right\}.
		\end{split}
	\end{equation}
	Note that $f_{y, \mu} (t, x)$ also depends on $c$, although we suppress this dependence in its definition to avoid heavy notation. Finally, we define
	\begin{equation}\label{eq:ParisiFunc_1dim_reformulated}
		\mathsf{F}_1 (\mu, c, r) = \E_{Y, G} \left[ f_{Y, \mu} \left( \norm{r}_2^2, r^\sT G \right) \right] + \frac{1}{2 \alpha} \int_{\norm{r}_2^2}^{1} \frac{\d t}{c + \int_{t}^{1} \mu(u) \d u }\, .
	\end{equation}
\end{conj}
In the following sections, we will drop the subscript ``$1$" and use $\mathsf{F} (\mu, c, r)$ instead of $\mathsf{F}_1 (\mu, c, r)$.

%
%
\section{Main results (I): Algorithmic achievability}\label{sec:main_results}
In this section, we present our main results regarding $\cuF_{m, \alpha, \varphi}^{\salg}$, the set of  computationally
feasible probability distributions in $\cuF_{m, \alpha, \varphi}$. 
In Section \ref{sec:IAMP}, we describe the class of two-stage AMP algorithms to be analyzed, with detailed analysis deferred to \cref{sec:iamp}.
Section \ref{sec:FeasibleIAMP} presents the characterization of the set of
probability distributions in $\cuF_{m, \alpha, \varphi}^{\salg}$ that can be
realized using this class of algorithms. 
Section \ref{sec:VH_statement} then states our main achievability result for 
$\VH^{\salg}_{m,\alpha,\varphi} (h)$, the dual functional of $\cuF_{m, \alpha, \varphi}^{\salg}$.
%
%
\subsection{Overview of the algorithm}
\label{sec:IAMP}
We give a brief description of our AMP algorithm, and refer to Section~\ref{sec:iamp} for further details\footnote{The notation used here is primarily for simplicity of exposition, and may differ slightly from that in Section~\ref{sec:iamp}.}. Our algorithm has two stages. The first stage consists of $T_1$ iterations with fixed step size, followed by an incremental stage of
$T_2$ iterations with small step sizes, where $T_1$ and $T_2$ are two positive 
integers (independent of $n$ and $d$) to be determined.

The first stage of our algorithm consists of $T_1$ AMP iterations, where the $t$-th iteration only depends on the $(t-1)$-th one: for $t = 1, \cdots, T_1$, we update $\VV^t\in\R^{n\times m}$, $\WW^t\in\R^{d\times m}$ according to
%
\begin{align*}
    \WW^{t + 1} & = \frac{1}{\sqrt{n}} \XX^\sT F_t (\VV^{t}, \yy) - \WW^{t} K_{t}^\sT, \\
	\VV^{t} & = \frac{1}{\sqrt{n}} \XX \WW^{t} - \frac{d}{n} F_{t-1} (\VV^{t-1}, \yy)\, ,
\end{align*}
with $\WW^1 = \XX^\sT F_{0} (\yy) / \sqrt{n}$.
Here, $F_t:\R^m\times \R\to\R^m$ is understood to be applied row-wise.
Namely, for $\VV\in\R^{n\times m}$ (with rows $\vv_i$), $\yy\in\R^n$
(with entries $y_i$), 
$F_t (\VV, \yy)\in\R^{n\times m}$ is the matrix whose $i$-th row is $F_t (\vv_i, y_i)$.
Further, the Onsager correction term $K_t\in\R^{m\times m}$ is given by
\begin{equation*}
	K_{t} = \frac{1}{n} \sum_{i=1}^{n} \frac{\partial F_t}{\partial \vv_i^t} (\vv_i^t, y_i),
\end{equation*}
where $\vv_i^t$ is the $i$-th row of $\VV^t$, and $\frac{\partial F_t}{\partial \vv_i^t} (\vv_i^t, y_i)$ denotes the Jacobian matrix of $F_t$ with respect to $\vv_i^t$. We will show (using general tools from the analysis of AMP algorithms) 
that, with proper choices of the non-linearities and in the limit of large $T_1$ after $n,d\to \infty$, $n/d \to \alpha$,
this iteration converges to an approximate fixed point,
which will be the starting point of the second stage of our algorithm.

In the second stage, we allow each iterate to depend on all previous ones. Denote $\WW^{\le t} = (\WW^s)_{1 \le s \le t}$ and $\VV^{\le t} = (\VV^s)_{1 \le s \le t}$. We iterate, for $T_1 + 1 \le t \le T_1 + T_2$:
\begin{equation}\label{eq:amp_iter_2nd}
\begin{split}
    \WW^{t + 1} & = \frac{1}{\sqrt{n}} \XX^\sT F_t (\VV^{\le t}, \yy) - \sum_{s=1}^{t} G_s (\WW^{\le s}) K_{t, s}^\sT, \\
	\VV^{t} & = \frac{1}{\sqrt{n}} \XX G_t(\WW^{\le t}) - \sum_{s=1}^{t} F_{s-1} (\VV^{\le s-1}, \yy) D_{t, s}^\sT,
\end{split}
\end{equation}
where $F_t:(\R^{m})^t\times \R\to\R^m$ and 
$G_t:(\R^{m})^t\to\R^m$ are also understood to act row-wise, and
\begin{equation*}
    K_{t, s} = \frac{1}{n} \sum_{i=1}^{n} \frac{\partial F_t}{\partial \vv_i^s} \left( \vv_i^{1}, \cdots, \vv_i^t, y_i \right), \ D_{t, s} = \frac{1}{n} \sum_{i=1}^{d} \frac{\partial G_t}{\partial \ww_i^s} \left( \ww_i^{1}, \cdots, \ww_i^t \right), \quad t \ge s.
\end{equation*}
As before, we will overload the notations and let $F_t$ and $G_t$ operate on their argument matrices row-wise. 
We further assume that $F_t$ and $G_t$ take the following specific structure:
\begin{align*}
    F_t \left( \vv^{1}, \cdots, \vv^{t}, y \right) = \, & \vv^t \Phi_{t-1} \left( \vv^{1}, \cdots, \vv^{t-1}, y \right), \\
    G_t \left( \ww^{1}, \cdots, \ww^{t} \right) = \, & \ww^t \Psi_{t-1} \left( \ww^{1}, \cdots, \ww^{t-1} \right),
\end{align*}
where $\Phi_{t-1}$ and $\Psi_{t-1}$ 
are matrix-valued mappings that satisfy certain moment constraints,
which we spell out in \cref{sec:iamp}. The second stage of our algorithm involves $T_2$ 
iterations with the above choices of $F_t$ and $G_t$.

Finally, the output of our two-stage AMP algorithm is a linear combination of $\WW^{T_1}$ and the incremental AMP iterations in the second stage. To be concrete, we will show that 
\begin{equation*}
    \plim_{n,d\to\infty} \frac{1}{n} (\WW^{T_1})^{\sT} \WW^{T_1} = Q,
\end{equation*}
where $Q \in \Sym_+^m$ is a deterministic $m\times m$ matrix satisfying $0 \preceq Q \preceq I_m$, which will be characterized in \cref{sec:fixed_pt_amp}. Let $Q_1, \cdots, Q_{T_2}$ be $T_2$ deterministic $m\times m$ matrices such that 
$$
\sum_{t=1}^{T_2} Q_t^\sT Q_t = I_m - Q,
$$
and define
\begin{equation*}
    \WW_Q = \frac{1}{\sqrt{n}}  \WW^{T_1} + \frac{1}{\sqrt{n}} \sum_{t=1}^{T_2} G_{T_1 + t+1} \left( \WW^{\le T_1 + t + 1} \right) Q_t \, 
\end{equation*}
and set the final output of our algorithm to be $\hat{\WW}_n^{\sAMP} = \WW_Q (\WW_Q^\sT \WW_Q)^{-1/2}$, which is guaranteed to be a $d \times m$ orthogonal matrix. The set of $(\alpha, m)$-feasible distributions achieved by our algorithm will be studied in the next section.
%
%
\subsection{Achievability results for AMP}\label{sec:FeasibleIAMP}

We begin by presenting our main AMP achievability result for general $m$
 and $k$. We assume $m \ge k$ without loss of generality;
  if $m < k$, any projection matrix $\WW \in O(d, m)$ can be 
  embedded into a $d \times k$ orthogonal matrix by considering 
  only its first $m$ columns. 
\begin{defn}[Bayes AMP feasible region]\label{defn:Bayes_AMP_region_general}
	Let the random vector $(Y, G, Z) \in \R \times \R^k \times \R^m$ be such that
	\begin{equation}\label{eq:Triple_distribution_general}
		\begin{split}
			& (Y, G) \perp\!\!\!\perp Z, \, Z \sim \normal(0, I_m), \, G \sim \normal(0, I_k), \\
			& Y = \varphi(G, \veps), \, \veps \sim P_{\veps}, \, \veps \perp\!\!\!\perp G.
		\end{split}
	\end{equation}
    Assume that $\E \left[ G \vert Y \right]$ is non-zero and
    define $C_{\sBayes}\in\R^{k \times m}$  via
    \begin{align}
    \cuC_{\sBayes} &:= 
    \Big\{ \C\in \Sym_+^{k}:\; \C=  \alpha \E \left[ \E \Big[ G - \C^{1/2} Z \Big\vert (\C^{1/2})^{\sT} G + Z, Y \Big] \E \Big[ G - \C^{1/2} Z \Big\vert (\C^{1/2})^{\sT}G + Z, Y \Big]^\sT \right]\Big\}\, ,\nonumber\\
    &C_{\sBayes}:= \C_{\sBayes}^{1/2}\, ,\;\;\; \;\;\;   \C_{\sBayes}:=\min{}_{\preceq} \cuC_{\sBayes}\, ,\label{eq:CBayesDef}
    \end{align}
    where $\min{}_{\preceq}$ denotes the minimum with respect to the semidefinite (Loewner) order. It is understood that, for $\C\in\R^{k\times k}$,
    $C=\C^{1/2}\in\R^{k\times m}$ is any matrix such that $CC^{\sT}=\C$
    (by \cref{lem:property_recursion_general}, the specific choice of $C$ is immaterial to our results).
    
	We then define the \emph{Bayes AMP feasible region} as
	\begin{equation*}
		A_{\sBayes} := \, \left\{ (R, Q) \in \R^{k \times m} \times \Sym_{+}^{m}: \, R^\sT R \preceq Q \preceq I_m, \, R (Q - R^\sT R)^{-1} R^\sT \preceq C_{\sBayes} C_{\sBayes}^\sT \right\}.
	\end{equation*}
	If $\cuC_{\sBayes} = \emptyset$, the inequality constraint $R (Q - R^\sT R)^{-1} R^\sT \preceq C_{\sBayes} C_{\sBayes}^\sT$ in the above definition is vacuous and therefore removed.
\end{defn}
\begin{rem}
While the existence of the minimum in \eqref{eq:CBayesDef} is not immediately obvious, it is guaranteed by the following argument. 
Define the sequence $\{ \C_t \}_{t=0}^{\infty} \subset \Sym_+^k$ 
recursively via:
	\begin{equation}\label{eq:Bayes_recursion_general}
		\begin{split}
			\C_0 = \, & 0, \\
			\C_{t+1} = \, & \alpha \E \left[ \E \left[ G - \C^{1/2}_t Z \big\vert (\C^{1/2}_t)^{\sT} G + Z, Y \right] \E \left[ G - \C^{1/2}_t Z \big\vert (\C^{1/2}_t)^{\sT} G + Z, Y \right]^\sT \right] =: \SEmap(\C_t).
		\end{split}	
	\end{equation}
    Then, \cref{lem:property_recursion_general} in the appendix implies
    that $\C_* = \lim_{t \to \infty} \C_t$ exists, though it can be infinite. 
    Further, $\C_* \in \cuC_{\sBayes}$ by continuity of the mapping
    $\C\mapsto \SEmap(\C)$. Finally, for any $\C_*'\in \cuC_{\sBayes}$,
    we obviously have $\C_0\preceq  \C_{*}'$ and therefore (by monotonicity
    of $\SEmap$ with respect to the Loewner order established in \cref{lem:property_recursion_general})
    $\SEmap^t(\C_0) \preceq \SEmap^t(\C'_*)$ for all $t$, i.e.,  
	$\C_t\preceq \C'_*$ for
    all $t$. By taking the limit $t\to\infty$, we conclude that 
	$\C_{*}\preceq \C'_*$, whence $\C_*=\min_{\preceq}\cuC_{\sBayes}$. 
	We can therefore identify $\C_{\sBayes}=\C_*$. 
\end{rem}
We now introduce a family of Lipschitz functions, termed $(R, Q)$-contractions, which are used in the first stage of our AMP algorithm. The contraction property of these functions ensures that, during this stage, the AMP iterations converge to the solution of a set of fixed-point equations.
\begin{defn}\label{defn:mu_Q_contraction}
	Let $(Y, G, Z)$ be as described in \cref{eq:Triple_distribution_general}, and define 
	\begin{equation*}
		Z_{R, Q} = R^\sT G + ( Q - R^\sT R )^{1/2} Z.
	\end{equation*}
	We say that a Lipschitz function $F: \R^m \times \R \to \R^m$ is an $(R, Q)$-contraction, if
	\begin{align}
		\label{eq:R-Constraint}
		R^\sT = \, \E \left[ \frac{\partial F}{\partial G} \left( Z_{R, Q}, \varphi (G, \veps) \right) \right], \quad Q = R^\sT R + \frac{1}{\alpha} \E \left[ F \left( Z_{R, Q}, \varphi (G, \veps) \right) F \left( Z_{R, Q}, \varphi (G, \veps) \right)^\sT \right],
	\end{align}
	and there exists some $S \in \Sym_+^m$, $S \succ 0$, such that
	\begin{equation}\label{eq:R_Q_contraction_condition}
		\frac{1}{\alpha} \E \left[ \frac{\partial F}{\partial Z_{R, Q}} \left( Z_{R, Q}, \varphi (G, \veps) \right)^\sT S \,\frac{\partial F}{\partial Z_{R, Q}} \left( Z_{R, Q}, \varphi (G, \veps) \right) \right] \preceq \, S.
	\end{equation}
	(Here, the partial derivatives are defined by treating $F(Z_{R,Q},\varphi(G,\eps))$
	as a function of the variables $(Z_{R,Q},G)\in \R^m\times\R^k$.)
\end{defn}

We are now in position to state our main achievability results for 
AMP algorithms, whose proof is presented in \cref{sec:iamp}. For the reader's convenience, in Appendix \ref{app:Special_m=1}
we provide the explicit formulation of Theorem \ref{thm:AMP_Feasibility_general} for the special case $m=k=1$. 
\begin{thm}
\label{thm:AMP_Feasibility_general}
	For any $(R, Q) \in A_{\sBayes}$, let $(Y, Z_{R, Q})$ be as described in \cref{defn:mu_Q_contraction}, and $F$ be an $(R, Q)$-contraction.
	Let $(B_t)_{0 \le t \le 1}$ be an $m$-dimensional standard Brownian motion independent of $(Y, Z_{R, Q})$. Define the filtration $\{ \cF_t \}_{0 \le t \le 1}$ by
		$\cF_t = \sigma \left( Z_{R, Q}, Y, (B_s)_{0 \le s \le t} \right)$
    for all $ 0 \le t \le 1$.
	
	Assume $Q(t) \in L^2 ([0, 1] \to \R^{m \times m})$ satisfies
	\begin{equation}
		\int_{0}^{1} Q(t) Q(t)^\sT \d t = I_m - Q,\label{eq:Q-Constraint}
	\end{equation}
	and $\{ \Phi_t \}_{0 \le t \le 1}$ is an $m \times m$ matrix-valued progressively measurable stochastic process with respect to the filtration $\{ \cF_t \}_{0 \le t \le 1}$, satisfying
	\begin{equation}\label{eq:Phi-Constraint}
		\E \left[ \Phi_t \Phi_t^\sT \right] \preceq \frac{I_m}{\alpha}, \;\;\;\; \forall 0 \le t \le 1\, .
	\end{equation}
	Then, there exists a two-stage AMP algorithm (described in \cref{sec:IAMP} and further detailed in \cref{sec:iamp}) that outputs $\hat{\WW}_n^{\sAMP} \in O(d, m)$, such that
    \begin{itemize}
        \item [(a)] $\lim_{n \to \infty} \WW_*^\sT \hat{\WW}_n^{\sAMP} = R$ almost surely.
        \item [(b)] Defining
            \begin{equation}\label{eq:Udef}
                U = \, Z_{R, Q} + \frac{1}{\alpha} F \left( Z_{R, Q}, Y \right) + \int_{0}^{1} Q(t) \left( I_m + \Phi_t \right) \d B_t,
            \end{equation}
            then the following holds almost surely as $n \to \infty$:
            \begin{equation*}
            \begin{split}
                \frac{1}{n} \sum_{i=1}^{n} \delta_{( y_i, \, (\hat{\WW}_n^{\sAMP})^\sT \xx_i ) } \stackrel{w}{\to} \, &  \operatorname{Law} ( Y, U ), \\
                \frac{1}{n} \sum_{i=1}^{n} h \big( y_i, \, \big( \hat{\WW}_n^{\sAMP} \big)^\sT \xx_i \big) \stackrel{w}{\to} \, & \E [ h ( Y, U ) ], \quad \forall h\in C_b (\R \times \R^m).
            \end{split}
            \end{equation*}     
            Consequently, we have $\operatorname{Law} (Y, U) \in \cuF_{m, \alpha, \varphi}^{\salg}$. 
    \end{itemize}
\end{thm}
\begin{rem}[Test error]\label{rem:test_error_AMP_ERM}
    Let $h = - \ell$ for a given loss function $\ell$, and consider the 
    corresponding empirical risk minimization problem \eqref{eq:ERM_definition}. 
    We define the training and test errors of $\hat{\WW}_n^{\sAMP}$ as follows:
    %
    %
    \begin{align*}
    \hL_n(\hat{\WW}_n^{\sAMP})&:=\frac{1}{n} \sum_{i=1}^n
    \ell \Big( y_i, \, \big( \hat{\WW}_n^{\sAMP} \big)^\sT \xx_i \Big)\, ,
    \\
    L(\hat{\WW}_n^{\sAMP})&:=
    \E \Big[ \ell \Big( y_{\stest}, \, \big( \hat{\WW}_n^{\sAMP} \big)^\sT \xx_{\stest} \Big)\Big| \yy,\XX\Big]\, ,
    \end{align*}
    where $(\xx_{\rm test}, y_{\rm test})$ is a test data point 
	independent of the training set $\{ (\xx_i, y_i) \}_{i=1}^{n}$.
    \cref{thm:AMP_Feasibility_general} then implies that, for any $(R,Q)\in A_{\sBayes}$
	and $(Q(t), \Phi_t)$ satisfying 
	Eqs.~\eqref{eq:Q-Constraint} and \eqref{eq:Phi-Constraint}, the following pair of training and test errors are achievable in the proportional limit 
    $n/d \to \alpha$ (with $U$ defined as per Eq.~\eqref{eq:Udef}):
    \begin{align*}
		\lim_{n \to \infty} \hL_n (\hat{\WW}_n^{\sAMP})& 
		=   \E[\ell(Y,U)]\, ,\\
        \lim_{n \to \infty} L(\hat{\WW}_n^{\sAMP})&= \E \left[ \ell \left( Y, \, R^{\sT}G +(I_m-R^{\sT}R)^{1/2}Z\right) \right] \, .
    \end{align*}
    In particular, the limiting test error only depends on $R$, which captures the
	asymptotic correlation between $\hat{\WW}_n^{\sAMP}$ and the ground truth $\WW_*$.
\end{rem}

\begin{rem}\label{rem:weak_diff_contraction}
    Since $\cuF_{m, \alpha, \varphi}^{\salg}$ is closed under 
	weak limits (cf. \cite[Lemma E.8]{montanari2022overparametrized}), we may allow an $(R, Q)$-contraction $F$ to be only weakly 
	differentiable. In this case, $\partial F / \partial G$ and 
	$\partial F / \partial Z_{R, Q}$ in \cref{defn:mu_Q_contraction} 
	should be interpreted as weak derivatives. Furthermore, by Stein's identity, the first equality constraint in 
	Eq.~\eqref{eq:R-Constraint}, namely
	\begin{align*}
		R^\sT = \, \E \left[ \frac{\partial F}{\partial G} \left( Z_{R, Q}, \varphi (G, \veps) \right) \right]\, , 
	\end{align*}
    is equivalent to
    \begin{equation}\label{eq:contraction_Stein}
        R = \, ( I_k- RQ^{-1}R^{\sT} )^{-1} 
		\E \big[\left( G - R Q^{-1}Z_{R, Q} \right)\,
		 F \left( Z_{R, Q}, \varphi(G, \veps) \right)^{\sT}
		 \big]\, .
    \end{equation}
\end{rem}

Theorem~\ref{thm:AMP_Feasibility_general} 
(see also Theorem \ref{thm:AMP_Feasibility_Super} for the case $m=k=1$)
characterizes the asymptotic value of the optimization 
problem~\eqref{eq:GeneralOpt_Supervised} achieved by our two-stage AMP algorithm. Optimizing this value involves a variational problem over $F$ and $\Phi$, 
which will be examined in greater detail in the following sections.

%
%

\subsection{Dual value $\VH_{m, \alpha, \varphi}^{\salg}(h)$ and stochastic optimal control}
\label{sec:VH_statement}
We now discuss the implications of \cref{thm:AMP_Feasibility_general} 
on the optimization problem~\eqref{eq:GeneralOpt_Supervised}, 
which we restate here for future reference:
\begin{equation*}
    \mbox{maximize} \quad H_{n, d}(\WW) := \, \frac{1}{n} \sum_{i=1}^{n} h \big( y_i, \WW^\sT \xx_i \big), \quad \mbox{subject to} \ \WW \in O(d, m).
\end{equation*}
To characterize its optimum achieved by our two-stage AMP algorithm, 
define the dual functional
\begin{equation}\label{eq:SOC_First}
\begin{split}
	& \VH_{m, \alpha, \varphi}^{\sAMP} (h) := \, \sup \, \E \left[ h \left( Y, \, Z_{R, Q} + \frac{1}{\alpha} F \left( Z_{R, Q}, Y \right) + \int_{0}^{1} Q(t) \left( I_m + \Phi_t \right) \d B_t \right) \right], \\
	& \mbox{subject to} \,\, (R, Q) \in A_{\sBayes}, \, \text{$F$ is an $(R, Q)$-contraction}, \\
    & \quad \quad \quad \quad \,\, \int_{0}^{1} Q(t) Q(t)^\sT \d t = I_m - Q, \, \text{and} \,\, \E \left[ \Phi_t \Phi_t^\sT \right] \preceq \frac{I_m}{\alpha}, \, \forall t \in [0, 1].
\end{split}
\end{equation}
Note that for any fixed choice of $(R, Q)$, $F$ and $\{ Q(t) \}_{t \in [0, 1]}$, \cref{eq:SOC_First} describes a stochastic optimal control problem for the control process $\{ \Phi_t \}_{t \in [0, 1]}$. 
The following theorem, which is a direct consequence of \cref{thm:AMP_Feasibility_general}, characterizes the asymptotic maximum of $H_{n, d} (\WW)$ achieved by our two-stage AMP algorithm, in terms of $\VH_{m,\alpha,\varphi}^{\sAMP} (h)$.
\begin{thm}[Optimal value achieved by AMP algorithms] \label{thm:ValH}
For any $h: \R \times \R^m \to \R$ continuous, bounded from above, and of at most linear growth at infinity, the following holds:
\begin{itemize}
\item[(a)] \emph{Upper bound.} If $\hat{\WW}_n^{\sAMP}$ is the output of any AMP algorithm as described in 
  Section \ref{sec:IAMP} (and further elaborated in \cref{sec:iamp}), then almost surely,
	\begin{equation*}
		\lim_{n \to \infty} H_{n,d} \big( \hat{\WW}_n^{\sAMP} \big) \le \, \VH_{m,\alpha,\varphi}^{\sAMP} (h).
	\end{equation*}
 \item[(b)] \emph{Achievability.}
	For any $\veps > 0$, there exists a two-stage AMP algorithm outputting $\hat{\WW}_n^{\sAMP}$, such that almost surely,
	\begin{equation*}
		\lim_{n \to \infty} H_{n,d} \big( \hat{\WW}_n^{\sAMP} \big) \ge \, \VH_{m,\alpha,\varphi}^{\sAMP} (h) - \veps \, .
	\end{equation*}
 \end{itemize}
\end{thm}
Of course, since AMP is a polynomial-time algorithm, we conclude that
\begin{align*}
\VH_{m,\alpha,\varphi}^{\salg} (h) \ge \, \VH_{m,\alpha,\varphi}^{\sAMP} (h)\, .
\end{align*}
Namely, $\VH_{m,\alpha,\varphi}^{\sAMP} (h)$ is an algorithmic lower bound for the value of the optimization problem \eqref{eq:GeneralOpt_Supervised}.
%
%

Henceforth, we focus on the case $m = k = 1$. In this setting, the function parameter 
$Q(t)=q(t)$ in \cref{thm:AMP_Feasibility_general} (see also its specialization to $m=k=1$
in Appendix \ref{app:Special_m=1}) can be eliminated, as any choice of $q(t)$ reduces to $q(t)=1$ via a reparametrization of the time variable $t$. As a consequence, the stochastic control problem~\eqref{eq:SOC_First} can be further simplified, as stated below.
\begin{prop}\label{thm:AMP_Feasibility_Simple}
For $m = k = 1$, we have
\begin{align}
	\VH_{1, \alpha, \varphi}^{\sAMP} (h)  = &\sup\Big\{ \VH_{1, \alpha, \varphi}^{\sAMP} (r,q,h):
	(r, q) \in A_{\sBayes}\Big\}\, ,\label{eq:SOC_First_simple}\\
	 \VH_{1, \alpha, \varphi}^{\sAMP} (r,q,h) := &\; \sup \; \E \left[ h \left( Y, \, Z_{r, q} + \frac{1}{\alpha} F \left( Z_{r, q}, Y \right) + \int_{q}^{1} \left( 1 + \phi_t \right) \d B_t \right) \right], \label{eq:SOC_First_1dim}\\
	& \mbox{\rm subject to} \,\, \text{$F$ is an $(r, q)$-contraction}, \, \text{and} \,\, \sup_{t \in [q, 1]} \E \left[ \phi_t^2 \right] \le \frac{1}{\alpha}\, , \nonumber
\end{align}
where we recall that $F$ is an $(r, q)$-contraction if and only if
\begin{align*}
&r = \, \E \left[ \frac{\partial F}{\partial G} \left( Z_{r, q}, \varphi(G, \veps) \right) \right], \,\, q = \, r^2 + \frac{1}{\alpha} \E \left[ F \left( Z_{r, q}, \varphi(G, \veps) \right)^2 \right], \\
& \quad \quad \frac{1}{\alpha} \E \left[ \frac{\partial F}{\partial Z_{r, q}} \left( Z_{r, q}, \varphi(G, \veps) \right)^2 \right] \le 1 \, .
\end{align*}
\end{prop}
The proof of \cref{thm:AMP_Feasibility_Simple} closely follows that of \cite[Proposition 3.2]{montanari2024exceptional}, so we omit it here for simplicity. Later in \cref{sec:Parisi_formula}, we will present an extended variational principle for 
$\VH_{1,\alpha,\varphi}^{\sAMP} (h)$, which parallels the replica prediction for $\VH_{1,\alpha,\varphi} (h)$ established in \cref{sec:replica_method}.

%
%
%
\section{Replica symmetric regime for $m = k = 1$}\label{sec:heuristic}
At sufficiently large aspect ratio $\alpha$, the variational principles described in the previous sections simplify significantly, along with the 
ERM landscape. In statistical physics, this is usually referred to as a 
`replica symmetric' phase.
In this section, we state the resulting
simplified formulas, deferring their proofs to Appendix~\ref{append:heuristic}.

We begin with considering the variational principle for the optimizer of $\mathsf{F} (\mu, c, r) := \mathsf{F}_1 (\mu, c, r)$ given in Conjecture \ref{conj:Parisi_formula_1dim}. Define $r_* \in [-1, 1]$ as
\begin{equation*}
 r_* = \, \arg \max_{r \in [-1, 1]} \inf_{(\mu, c) \in \mathscr{U} \times \R_{> 0}} \mathsf{F} (\mu, c, r).
\end{equation*}
The optimizer $r_*$ can be interpreted as the asymptotic limit of the inner product
$\< \hww_n^{\sERM}, \ww_* \>$, where $\hww_n^{\sERM}$ is the global maximizer of \eqref{eq:GeneralOpt_Supervised} (or equivalently, the global minimizer of \eqref{eq:ERM_definition} with $\ell = - h$). The next proposition establishes that for sufficiently large $\alpha$, the Parisi variational principle is replica symmetric at $r_*$.
\begin{prop}
	\label{prop:ERM_asymptotics}
    Assume $\E [ (\partial_x h(Y, G))^2 ] > 0$, and for 
	$r \in [-1, 1]$, define the function
	\begin{equation*}
        v (r) = \, \E \left[ h \left( Y, r G + \sqrt{1 - r^2} Z \right) \right], \quad Z \sim \normal(0, 1), \,\, Z \perp \!\!\! \perp (Y, G).
    \end{equation*}
	Further assume that $r_1 := \arg \max_{r \in [-1, 1]} v (r) \in \{ -1, 1 \}$, and $r_1 v'(r_1) > 0$. Then, there exist constants $C, \alpha_0 > 0$, such that for all $\alpha > \alpha_0$, the following happens:
    \begin{itemize}
        \item [(i)] If $1 - r^2 \le C / \alpha$, then the Parisi variational problem $\inf_{(\mu, c) \in \mathscr{U} \times \R_{> 0}} \mathsf{F} (\mu, c, r)$ has a replica symmetric solution, i.e., the infimum is achieved at $(\mu_* = 0, c_*)$.
        \item [(ii)] $r_*$ satisfies $1 - r_*^2 \le C / \alpha$.
    \end{itemize}
	In this case, $r_*$ can be equivalently defined as
	\begin{align}\label{eq:r_*_RS}
		r_*=\arg \max_{r\in [-1,1]} \sup_{\E [U^2] \le 1 / \alpha} \E \left[ h \left( Y, r G + \sqrt{1 - r^2} (Z + U) \right) \right].
		\end{align}
\end{prop}
\begin{rem}
    The conditions $r_1 \in \{ -1, 1 \}$ and $r_1 v' (r_1) > 0$ ensure that ERM achieves perfect recovery in the limit $\alpha = \infty$. This requirement is essential. Without this assumption, 
    the estimation error is bounded away from zero as $\alpha\to\infty$
    (the empirical risk minimizer $\hww_n^{\sERM}$ is an inconsistent estimator of $\ww_*$),
    thus changing the nature of the large-$\alpha$ asymptotics (also making it less interesting from a statistical  viewpoint).
\end{rem}

Recent works \cite{vilucchio2025asymptotics,montanari2026topological} derived high-dimensional asymptotics
for the empirical risk minimizer $\hww_n^{\sERM}$ in the regime $\alpha > \alpha_0$ for some constant $\alpha_0$ (independent of $n$ and $d$), under certain technical conditions. Combining their results with \cref{prop:ERM_asymptotics}, we obtain the following corollary.
\begin{cor}
Let $\ell = - h$. Under the setting of Proposition \ref{prop:ERM_asymptotics}, as well as the assumptions and 
rate trivialization conditions in \cite[Theorem 3]{montanari2026topological}, there exists $\alpha_0 < \infty$ such that 
Conjecture \ref{conj:Parisi_formula_1dim} holds
for all $\alpha > \alpha_0$. In particular, we have almost surely
$$
\lim_{n \to \infty} \left\< \hww_n^{\sERM},\ww_* \right\>=r_*
$$ 
with $r_*$ defined as per \cref{eq:r_*_RS}. Further, the limiting training and test errors of $\hww_n^{\sERM}$ are given by:
\begin{align}
	\lim_{n \to \infty} \hL_n \left( \hww_n^{\sERM} \right) 
	= \, & \inf_{\E [U^2] \le 1 / \alpha}\E \left[ \ell \left( Y, r_* G + \sqrt{1 - r_*^2} (Z + U) \right) \right]\,,\label{eq:RS-asymptotics-train} \\
	\lim_{n \to \infty} L \left( \hww_n^{\sERM} \right) = \, & \E \left[ \ell \left( Y, \, r_*G + \sqrt{1-r_*^2} Z\right) \right] \, .\label{eq:RS-asymptotics-test}
\end{align}
\end{cor}
Finally, we show that for sufficiently large $\alpha$, these training and test errors are also achievable by our two-stage AMP algorithm.
\begin{prop}\label{prop:rs_amp_achievability}
Let $r_*$ be defined by Eq.~\eqref{eq:r_*_RS}. Under the settings of \cref{prop:ERM_asymptotics}, there exists $\alpha_0 < \infty$ such that for all $\alpha > \alpha_0$, 
one can construct a two-stage AMP algorithm as in Section \ref{sec:FeasibleIAMP} outputting $\hww_n^{\sAMP}$, such that
\begin{equation*}
    \lim_{n \to \infty} \left\langle \hww_n^{\sAMP}, \ww_* \right\rangle = r_*.
\end{equation*}
Further, Eqs.~\eqref{eq:RS-asymptotics-train} and \eqref{eq:RS-asymptotics-test} 
hold with $\hww_n^{\sERM}$ replaced by $\hww_n^{\sAMP}$.
\end{prop}

%

\section{Intermezzo: Large $n/d$ and tensor PCA equivalence}
\label{sec:Applications}


In this section, we apply the general theory developed in previous sections to a single-index model ($m=k=1$) with additive noise, trained by optimizing a correlation loss objective. We obtain a straightforward characterization of the algorithmically achievable training and test errors, in the large-$\alpha$ limit. We then show that this characterization admits a natural interpretation, in terms of an equivalent tensor PCA problem \cite{montanari2014statistical}. Proofs of the results in this section are deferred to \cref{append:single_index}.

\subsection{Single-index model}
\label{sec:SingleIndex}

Throughout this section, we assume the data
$\{ (\xx_i, y_i) \}_{i=1}^{n}$ follow the single-index model
of Eq.~\eqref{eq:single_index_model_first}, and consider maximizing 
the empirical correlation objective of Eq.~\eqref{eq:single_index_model_second} to learn the true signal $\ww_*$.
Without loss of generality, we will assume 
$\E [\sigma (G)] = 0$ for $G \sim \normal (0, 1)$; otherwise, one can simply consider the 
centered activation function $\sigma(z) - \E [\sigma (G)]$.

 We first take the limit $n,d\to\infty$, $n/d\to\alpha$,
  and then 
 \begin{align} \label{eq:DoubleLimit}
  \alpha\to\infty, \;\;\;\;\lambda\to 0, \;\;\;\;\alpha\lambda\to\oalpha\, ,
 \end{align}
   for some fixed $\oalpha \in (0, +\infty)$.
This turns out to be the correct scaling to avoid classical
low-dimensional behavior, or a trivial information-theoretic limit.

We will establish that---in this double limit--- 
the single-index model~\eqref{eq:single_index_model_first} is 
asymptotically equivalent 
to a Gaussian process model with asymptotically matching first and 
second moments. Namely, we let $\hCorr^g_d (\,\cdot\, )$ be 
a Gaussian process on $\S^{d-1}$ with mean and covariance
\begin{equation}\label{eq:GaussianModel}
\begin{split}
	\E \big[ \hCorr^g_d(\ww) \big] = \, \xifs(\<\ww_*,\ww\>)\, , \quad \Cov\big(\hCorr^g_d(\ww_1), \hCorr^g_d(\ww_2)\big) = \, \frac{1}{d\oalpha}\, \xiss(\<\ww_1,\ww_2\>)\, ,
\end{split}
\end{equation}
where $\xiss, \xifs$, and $\xiff$ (for future reference) are defined as 
\begin{align*}
\xiss(q) = \E \big[\sigma(G_1)\,\sigma(G_q)\big]\, , \,\, \xifs(q) = \E \big[\varphi(G_1)\,\sigma(G_q)\big]\, , \,\, \xiff(q) = \E \big[\varphi(G_1)\,\varphi(G_q)\big] \, ,
\end{align*}
for $(G_1,G_q)$ jointly centered Gaussian with $\E[G_1^2]=\E[G_q^2]=1$
and $\E[G_1G_q]=q$. Denoting by $ \{ \varphi_k \}_{k=0}^{\infty}$ and $\{ \sigma_k \}_{k=0}^{\infty}$ the Hermite coefficients of functions $\varphi$ and $\sigma$, we have
\begin{equation}\label{eq:xi_expansion}
\xiff(t) = \sum_{k=0}^{\infty}\varphi^2_k t^k\, ,\;\;\;\;
\xifs(t) = \sum_{k=0}^{\infty}\varphi_k\sigma_k t^k\, ,\;\;\;\;
\xiss(t) = \sum_{k=0}^{\infty}\sigma^2_k t^k\, .
\end{equation}
By straightforward calculation, we can verify that $\hCorr^g_d$ and $\hCorr_n$ 
(the latter being defined in Eq.~\eqref{eq:single_index_model_second}) have the same mean, and their covariances are asymptotically equivalent in the double limit $n / d \to \alpha$, $\alpha \lambda \to \oalpha$.

In Section \ref{sec:Gaussian_model_setup}, we will see that the Gaussian process model $\hCorr^g_d$ has a natural statistical interpretation as a generalized tensor PCA problem.


\subsection{Replica prediction for global maximum}\label{sec:replica_single_index}
We now present the replica prediction for the asymptotic global
 maximum of $\hCorr_n (\ww)$ over the unit sphere,
 under the double limit \eqref{eq:DoubleLimit}.
 
  In this regime, the Parisi variational principle simplifies 
  dramatically, as stated next. 
  The following claim is a consequence of Conjecture \ref{conj:Parisi_formula_1dim}: we derive it from that conjecture
  in Appendix \ref{app:Derivation_replica_single_index}.
\begin{conj}\label{conj:replica_single_index}
	Recall the function space $\mathscr{U}$ from \cref{conj:Parisi_formula_1dim} and $\hCorr_n (\ww)$ from \cref{eq:single_index_model_second}. Under the double limit~\eqref{eq:DoubleLimit}, we have
\begin{align}\label{eq:replica_single_index}
	\lim_{\alpha \lambda \to \oalpha} \lim_{n/d \to \alpha} \max_{\ww \in \S^{d-1}} \hCorr_n (\ww) = \, \sup_{r \in [-1, 1]} \inf_{(\mu, c) \in \mathscr{U} \times \R_{>0}} \osF_{\oalpha} (\mu,c,r ),
\end{align}
	where
\begin{equation}\label{eq:parisi_func_single_index}
\osF_{\oalpha} (\mu, c, r) = \, \xifs (r) + \frac{1}{2} \int_{r^2}^{1} \mu(t) \xiss' (t) \d t + \frac{1}{2} c \xiss' (1) + \frac{1}{2 \oalpha} \int_{r^2}^{1} \frac{\d t}{c +  \int_{t}^{1} \mu(s) \d s } \, .
\end{equation}
\end{conj}
We also note that the minimization of 
$\osF_{\oalpha} (\mu, c, r)$ over the extended domain used in \cref{sec:Parisi_formula} below can be performed explicitly. We state the resulting minimum in the following proposition and omit its proof, which follows from a straightforward application of integration by parts.
\begin{prop}\label{prop:single_index_extend_minimize}
Define for $q \in [0, 1]$, the function spaces
\begin{align}
\label{eq:FS_def}
\FS&:= \left\{ (\mu, c) \in L^1 [0, 1] \times \R_{> 0}:\, \mu \vert_{[0, t]} \in L^{\infty} [0, t] \ \text{and} \ c + \int_{t}^{1} \mu(s)\, \d s > 0, \, \forall t \in [0, 1) \right\},\\
\label{eq:FS_q_def}
\FS (q) &:= \left\{ (\mu, c) \in \FS:\, \mu (t) = 0, \ \forall t \in [0, q] \right\}.
\end{align}
Further define for any $(q, r) \in [0, 1] \times [-1, 1]$ with $r^2 \le q \le 1$, the functional
\begin{align}\label{eq:Vstar_single_index}
\VH_{\oalpha}^*(r,q) := \xifs(r) + \frac{1}{\sqrt{\oalpha}} \sqrt{(q-r^2)\xiss'(q)}
+ \frac{1}{\sqrt{\oalpha}}\int_q^1\sqrt{\xiss''(t)}\, \de t\, .
\end{align}
Then, we have
\begin{equation}\label{eq:Vstar_single_index_min}
	\VH_{\oalpha}^*(r,q)= \inf_{(\mu,c)\in\FS(q)}\osF_{\oalpha} (\mu, c, r)\, .	
\end{equation}
\end{prop}

\subsection{Achievability via AMP}

In this section, we characterize the asymptotic values of 
$\hCorr_n (\ww)$ achieved by our two-stage AMP algorithm, in the double limit 
\eqref{eq:DoubleLimit}. Throughout the analysis, we assume 
$\xiff'(0) \neq 0$, namely $\varphi_1 :=\E[G\varphi(G)]\neq 0$. We begin by defining the Bayes AMP feasible region for the single-index model:
\begin{align}\label{eq:A_Bayes_single_index}
	A_{\sBayes} &= \{ (r, q) \in \R^2: q \in [r^2, 1], \, r^2 \le c_{\sBayes}^2 (q - r^2) \},\\
\label{eq:C_Bayes_Perceptron}
    c_{\sBayes} &:=\inf\left\{c>0:
	c^2= \oalpha \, \xiff' \left( \frac{c^2}{1+c^2} \right)\right\} \, .
\end{align}
 Next, for $(q, r) \in [0, 1] \times [-1, 1]$ with $r^2 \le q \le 1$, 
 define
\begin{equation}\label{eq:T0_change_of_variable}
 v_k =\sqrt{\frac{k\oalpha}{c^2}} \left(\frac{r^2}{q}\right)^{(k-1)/2} \varphi_k \, , \quad s_k = \sqrt{k}\sigma_k q^{(k-1)/2}\, , \,\, k \ge 1,
\end{equation}
and 
\begin{align}\label{eq:T0_def}
T_0(r, q) := \max\Big\{\<s,x\>:\; \; \|x\|_2 \le 1, \, \<x,v\>=1, \, \sum_{k\ge 1} (k-1)x_k^2\le \frac{q}{q-r^2}\Big\}\, .
\end{align}
Then, for a fixed pair $(r, q)$, we establish that the asymptotic value of $\hCorr_n (\ww)$ achieved by our two-stage AMP algorithm takes the simple form:
\begin{align}\label{eq:VAMP_single_index}
	\VH_{\oalpha}^{\sAMP} (r, q)&:=\xifs(r) +\sqrt{\frac{q-r^2}{\oalpha}} \, T_0(r,q)+ \frac{1}{\sqrt{\oalpha}}\int_q^1\sqrt{\xiss''(t)}\, \de t\, ,
\end{align}
see below for the precise statement.
\begin{prop}\label{prop:T0_equiv_def}
	For any 
	$(r, q) \in A_{\sBayes}$ and $(\alpha, \lambda) \in \R_+^2$, 
	there exists a two-stage AMP algorithm that returns 
	$\hbw_n^{\alpha, \lambda} \in \S^{d-1}$, such that
	\begin{align*}
	\lim_{\alpha \lambda \to \oalpha}\lim_{n/d \to \alpha} \< \ww_*, \hbw_n^{\alpha, \lambda} \> = r \, , \quad
	 \lim_{\alpha \lambda \to \oalpha} \lim_{n/d \to \alpha}  \hCorr_n \big( \hbw_n^{\alpha, \lambda} \big) = \VH_{\oalpha}^{\sAMP} (r,q)\, .
	\end{align*}
\end{prop}
Let $\VH_{\oalpha}^{\sAMP} = \sup_{(r, q) \in A_{\sBayes}} \VH_{\oalpha}^{\sAMP} (r, q)$ denote the asymptotic maximum value of $\hCorr_n (\ww)$ achieved by our two-stage AMP algorithm over all feasible choices of $(r, q)$. The following theorem establishes that, under certain optimality conditions, $\VH_{\oalpha}^{\sAMP}$ admits a simpler expression that coincides with the 
formula in Eq.~\eqref{eq:Vstar_single_index}.
\begin{thm}\label{thm:single_index_duality}
Define $\VH_{\oalpha}^*(r,q)$ as per Eq.~\eqref{eq:Vstar_single_index}.
Then the following holds:
\begin{itemize}
    \item [(a)] For all $(q, r)$ with $r^2 \le q \le 1$,
	 $\VH^{\sAMP}_{\oalpha}(r,q) \le \VH_{\oalpha}^*(r,q)$. 
    \item [(b)] For all $r \in [-1, 1]$, define
        \begin{align}\label{eq:def_q_star}
            q_*(r)= \arg\min_{q\in [r^2,1]} \left\{ \frac{\xiss'(q)}{q-r^2} \right\}.
        \end{align}
        Then, we have
        \begin{equation*}
            \sup_{q : (r, q) \in A_{\sBayes}} \VH_{\oalpha}^{\sAMP} (r, q) \le 
			 \VH_{\oalpha}^*(r, q_*(r)).
        \end{equation*}
    \item [(c)] Let
        \begin{equation}\label{eq:r_star_oalpha}
            r_* := \argmax_{r : (r, q_*(r)) \in A_{\sBayes}} \VH_{\oalpha}^*(r,q_*(r)) \, , \,\, \;\;q_* = q_* (r_*).
        \end{equation}
        Then, we have $ \VH_{\oalpha}^{\sAMP}(r_*, q_*) =  \VH_{\oalpha}^* (r_*, q_*)$, i.e., the value $ \VH_{\oalpha}^*(r_*, q_*)$ is achievable by our two-stage AMP algorithm. As a consequence,
		\begin{equation}\label{eq:optimality_single_index}
		    \max_{r : (r, q_*(r)) \in A_{\sBayes}} \VH_{\oalpha}^*(r,q_*(r)) 
            \le    \max_{(r, q) \in A_{\sBayes}}\VH_{\oalpha}^{\sAMP}(r, q)\le 
             \max_{r : (r, 1) \in A_{\sBayes}} \VH_{\oalpha}^*(r,q_*(r)).
		\end{equation}
		In particular, if the first and last terms in the above chain of inequalities coincide, then 
        $\VH_{\oalpha}^{\sAMP} = \VH_{\oalpha}^*(r_*, q_*)$.
\end{itemize}
\end{thm}
While it is not guaranteed that condition~\eqref{eq:optimality_single_index} holds universally, we demonstrate below that it is satisfied for sufficiently large $\oalpha$.
\begin{prop}\label{prop:optimality_large_oalpha}
	There exists a threshold $\oalpha_* (\varphi, \sigma)$ depending only on $\varphi$ and $\sigma$, such that \cref{eq:optimality_single_index} holds for all $\oalpha \ge \oalpha_* (\varphi, \sigma)$.
\end{prop}

\subsection{The equivalent Gaussian model}\label{sec:Gaussian_model_setup}
We next revisit the Gaussian process model defined by 
Eq.~\eqref{eq:GaussianModel}, and show that it has the following tensor PCA interpretation. Recall from \cref{sec:SingleIndex} that $\ww_*$ is the true signal, $\varphi_k$ is the $k$-th Hermite coefficient of $\varphi$, and $\sigma_k$ is the $k$-th Hermite coefficient of $\sigma$. For each $k\ge 1$, we observe $\YY^{(k)}\in(\R^d)^{\otimes k}$ as a rank-one, order-$k$ tensor perturbed by i.i.d. Gaussian noise:
\begin{align*}
\YY^{(k)} = \varphi_k \ww_*^{\otimes k} +\frac{1}{\sqrt{d\oalpha}}\GG^{(k)}\, ,
\end{align*}
where $\GG^{(k)} = (G^{(k)}_{i_1, \dots, i_k})_{1 \le i_1, \dots, i_k \le d}
 \in (\R^d)^{\otimes k}$ is a random tensor with i.i.d. $\normal (0, 1)$ entries. We propose to learn $\ww_*$ by maximizing the ``generalized tensor PCA'' objective:
\begin{align}
    \hCorr^g_d(\ww) = \, \sum_{k \ge 1} \sigma_k \<\YY^{(k)},\ww^{\otimes k}\> \, , \quad \ww \in \S^{d-1}\, ,\label{eq:TensorPCA-Cost}
\end{align}
The lemma below establishes the equivalence between this model and the Gaussian process model introduced in \cref{sec:SingleIndex}.
\begin{lem}\label{lem:tensor_pca_equivalence}
The stochastic process $\big\{ \hCorr^g_d (\ww) \big\}_{\ww \in \S^{d-1}}$ defined as per Eq.~\eqref{eq:TensorPCA-Cost} is a Gaussian process with mean and covariance given by Eq.~\eqref{eq:GaussianModel}.
\end{lem}
In Appendix~\ref{append:generalized_tensor_pca}, we show that this generalized tensor PCA model
 (and the equivalent Gaussian process model) share the same 
 asymptotic global maximum---as well as the same algorithmic values achieved by
  a two-stage AMP-based procedure---as the single-index model 
  analyzed in \cref{sec:SingleIndex}. 

\subsection{Examples}\label{sec:single_index_example}
We now present several illustrative examples of the single-index model and its tensor PCA equivalent. 
\begin{figure}[!ht]
    \centering
    \includegraphics[width=0.9\linewidth]{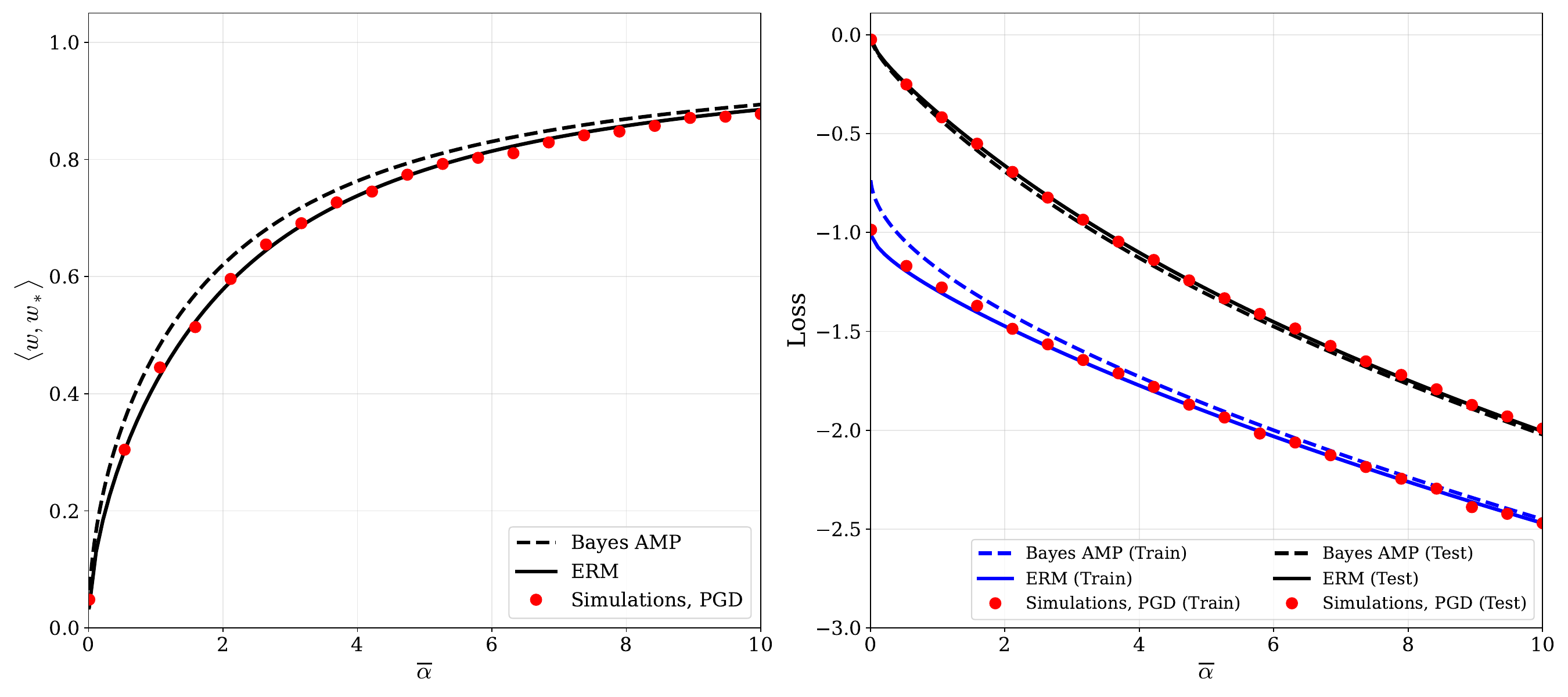}
    \caption{\textbf{Left panel:} The limiting overlap with $\ww_*$ for Bayes AMP, ERM and PGD. \textbf{Right panel:} The normalized limiting empirical and population correlation for these three methods. In our implementation of PGD, we set $d = 100$ and $\alpha = 200$, sweeping $\oalpha \in [0, 10]$ so that $\lambda$ ranges in $[0, 0.05]$. All reported numerical results are averaged over $100$ independent runs of PGD.}
    \label{fig:well_specified}
\end{figure}
\paragraph{Well-specified model.} We consider a simple well-specified setting, with $\sigma(z) = \varphi(z) = \max(0, z)$ being the standard ReLU activation function. In this setting, the learning task is equivalent to training a single-neuron ReLU network, and we have
\begin{align*}
	\xiss (x) = \xiff (x) = \xifs (x) = \, & \frac{x}{4} + \frac{x \arcsin x + \sqrt{1 - x^2}}{2\pi}.
\end{align*}
In Figure \ref{fig:well_specified},
we present theoretical predictions for both the Bayes optimal estimator, and the estimator that minimizes the training error within the class of two-stage AMP algorithms considered in Proposition \ref{prop:T0_equiv_def} (ERM). By direct calculation, we know that
\begin{equation*}
	x \mapsto (1 - x) \xiff' (x) = (1 - x) \left( \frac{1}{4} + \frac{1}{2 \pi} \arcsin x \right)
\end{equation*}
is decreasing on $[0, 1]$, which verifies condition $(i)$ in the statement of \cref{lem:single_index_FOC}. Therefore, the optimality condition~\eqref{eq:optimality_single_index} holds. As a consequence, maximizing 
$\VH_{\oalpha}^* \big( r, q_* (r) \big)$ over $(r, 1) \in A_{\sBayes}$ yields the theoretical predictions for ERM. For both Bayes AMP and ERM, we plot their correlations with the true parameter $\ww_*$, as well as the normalized training error
\begin{align}
	\hat{L}_n(\ww) := -\frac{1}{d \sqrt{\alpha\xiss(1)}} 
	\sum_{i=1}^n y_i \sigma(\ww^{\sT}\xx_i)
\end{align}
and the test error normalized in the same way. 

We also compare these theoretical predictions with 
numerical simulations for projected gradient descent (PGD) maximizing the empirical correlation $\hCorr_n (\ww)$ over the unit sphere.
We observe that the behavior of PGD closely matches our predictions for the optimal algorithm.
\begin{figure}[!ht]
    \centering
    \includegraphics[width=0.9\linewidth]{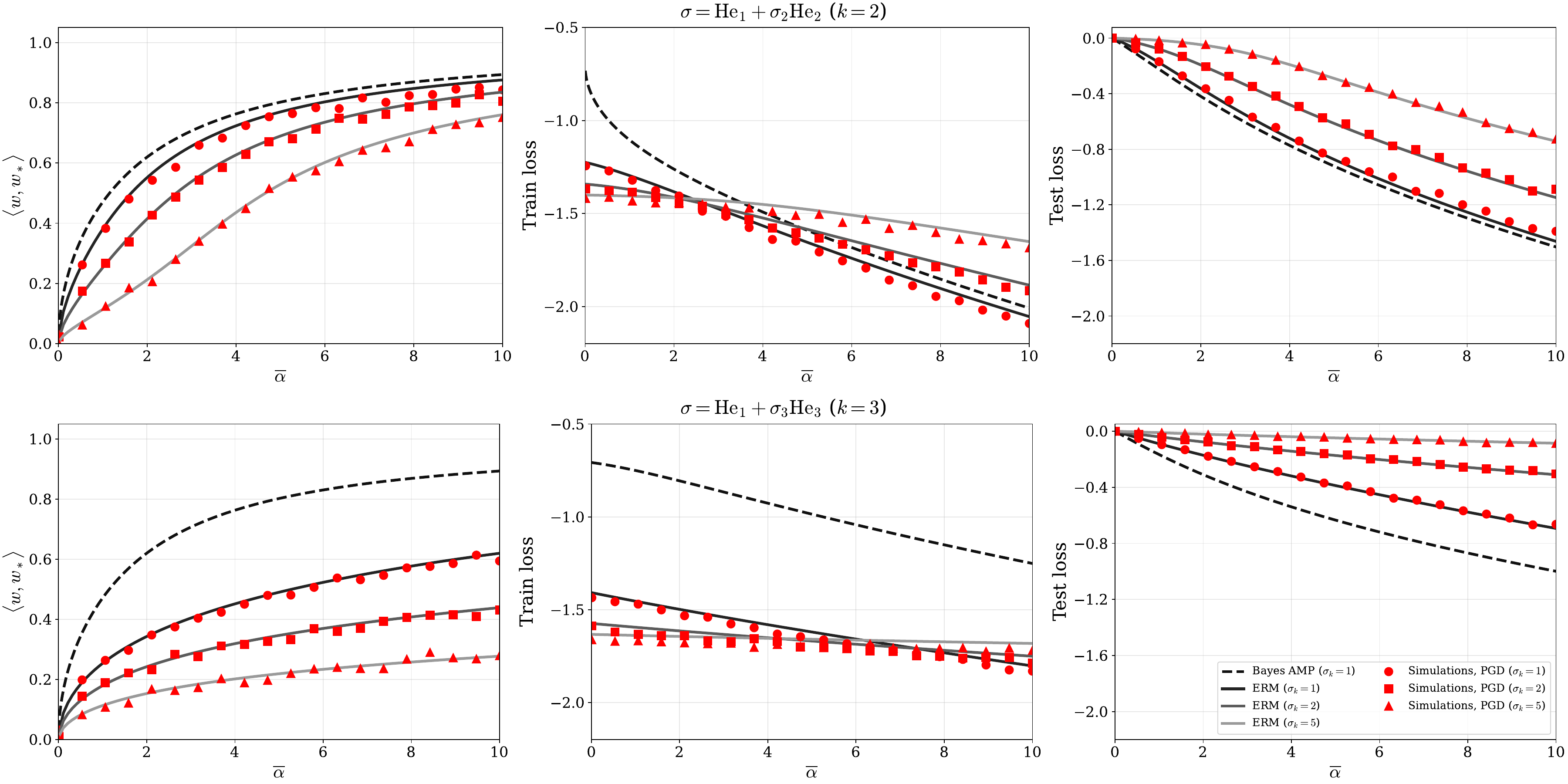}
    \caption{\textbf{Left panel:} The limiting overlap with $\ww_*$ for Bayes AMP, ERM and PGD, under different choices of the activation $\sigma$. \textbf{Right panel:} The normalized limiting empirical and population correlation for these three methods. In this experiment, we fix $\sigma_1 = 1$ and vary $k \in \{ 2, 3 \}$ and $\sigma_k \in \{ 1, 2, 5 \}$. The top and bottom rows present our results for $k=2$ and $k=3$, respectively. The remaining experimental setup is the same as in the well-specified case.}
    \label{fig:misspecified}
\end{figure}
\paragraph{A family of misspecified models.} 
In Figure~\ref{fig:misspecified}, we consider 
a misspecified scenario, 
where the true link function is still ReLU: $\varphi(z) = \max(0, z)$,
 yet the activation function takes the form 
 $\sigma(z) = \sigma_1 \He_1(z) + \sigma_k \He_k(z)$ for $k \ge 2$. Here, $\He_k$ denotes the $k$-th orthonormal Hermite polynomial. In this setting, we have $\xiss(t) = \sigma_1^2 t + \sigma_k^2 t^k$. Consequently, the optimization landscape becomes more complicated as we increase $k$.

%
%
Similar to the well-specified setting, maximizing 
$\VH_{\oalpha}^* \big( r, q_* (r) \big)$ over $(r, 1) \in A_{\sBayes}$ 
determines the limiting overlap with $\ww_*$, as well as the empirical 
and population correlation for the optimal ERM algorithm. We compute and plot these 
theoretical values across different choices of $k$ and 
$(\sigma_1, \sigma_k)$, and compare them with Bayes AMP and numerical
 simulations of PGD in Figure~\ref{fig:misspecified}.
 
 As in the well-specified setting, the theoretical predictions
 for the optimal ERM algorithm
 closely match the empirical results with PGD.
 However, the estimation and test errors of these methods are
 significantly larger than 
 those of Bayes AMP, especially when $\oalpha$ is small. 
 Furthermore, their performance deteriorates as we increase 
 $k$ or $\sigma_k$, indicating an increasingly severe overfitting.

%
%
%

\section{Main results (II): Parisi variational principle}
\label{sec:Parisi_formula}

In this section, we develop a variational principle that is dual to 
the stochastic optimal control problem~\eqref{eq:SOC_First_1dim}.
The resulting formula is closely related to the Parisi  
variational principle, which we derived heuristically in Section 
\ref{sec:replica_method} (namely,
the simplified Parisi formula for $m=1$ in \cref{conj:Parisi_formula_1dim}).

We begin with defining two relevant function spaces.
\begin{defn}[Space of functional order parameters]\label{defn:gamma_space}
Define the function spaces $\FS$ and $\FS(q)$
 as per Eqs.~\eqref{eq:FS_def} and~\eqref{eq:FS_q_def}, and 
 further define, for any $q \in [0, 1)$:
 \begin{align}
		\FSG&:= \Big\{\gamma: [0, 1] \to \R_{> 0} \mbox{ {\rm absolutely continuous}}:\, 
  (\mu, c) \in \FS \ \mbox{{\rm for }} \mu = \gamma' / \gamma^2, \, c = 1 / \gamma(1)
  \Big\}\, ,\\
	\FSG (q) &:=  \left\{ \gamma \in \FSG: \, \gamma \vert_{[0, q]} \ \mbox{is constant} \right\}.
\end{align}
Then, we see that $\gamma \in \FSG (q)$ if and only if $(\mu, c) \in\FS (q)$ with $(\mu, c) = (\gamma' / \gamma^2, 1 / \gamma(1))$.
\end{defn}
We now extend the definition of the Parisi functional in \cref{eq:ParisiFunc_1dim_reformulated} (we set $k = 1$ and drop the subscript $1$ for notational simplicity) from $\mathscr{U} \times \R_{> 0}$ to the larger function space $\FS$. For $(\mu, c) \in \FS$ and $r \in [-1, 1]$, define
\begin{equation*}
	\mathsf{F} (\mu, c, r) = \E_{Y, G} \left[ f_{Y, \mu} \left( r^2, r G \right) \right] + \frac{1}{2 \alpha} \int_{r^2}^{1} \frac{\d t}{c + \int_{t}^{1} \mu(u) \d u }\, ,
\end{equation*}
where for fixed $y \in \R$, $f_{y, \mu}$ is the solution to the  partial differential equation~\eqref{eq:ParisiPDE_1dim_reformulated}, which we copy here for the reader's convenience:
\begin{equation}\label{eq:ParisiEq-m1-bis}
\begin{split}
    &\partial_t f_{y, \mu} (t, x)+\frac{1}{2} \mu (t) (\partial_x f_{y, \mu} (t, x))^2 +  \frac{1}{2} \partial_x^2 f_{y, \mu} (t, x) = \, 0, \\
    &f_{y, \mu} (1, x) = \,  \sup_{u \in \R} \left\{ h \left( y, x + u \right) - \frac{u^2}{2c}  \right\}.
\end{split}
\end{equation}
Later in \cref{sec:tech_prelim}, we will establish existence, uniqueness and regularity of solutions to the Parisi PDE, thus ensuring that the Parisi functional is indeed well-defined on $\FS$.

We are now in position to state our main result establishing a Parisi-type variational principle for $\VH_{1, \alpha, \varphi}^{\sAMP} (r, q, h)$, defined via the stochastic optimal control problem~\eqref{eq:SOC_First_1dim}.
\begin{thm}\label{thm:two_stage_strong_duality}
    Assume that for all $y \in \R$, $h (y, \cdot)$ is $C^2$, Lipschitz continuous and bounded from above. Further assume that
    \begin{equation*}
        \E \left[ \norm{\partial_x h(Y, \cdot)}_{\infty}^2 \right] < \infty.
    \end{equation*}
    Then, the following holds.
 \begin{itemize} 
 \item[(a)] \emph{Variational formula.}
 Fix $(r, q) \in A_{\sBayes}$. For any $\gamma \in \FSG (q)$ and $(\mu, c) \in \FS (q)$ satisfying $\mu = \gamma' / \gamma^2$ and $c = 1 / \gamma(1)$, we have:
\begin{align*}
    \mathsf{F} (\mu, c, r) = \, \sup_{\substack{F: \R^2 \to \R \\ \phi \in D[q, 1] }} \E \Bigg[ & h \left( Y, \, Z_{r, q} + \frac{1}{\alpha} F \left( Z_{r, q}, Y \right) + \int_{q}^{1} \left( 1 + \phi_t \right) \d B_t \right) - \frac{1}{2} \int_{q}^{1} \gamma(t) \left( \phi_t^2 - \frac{1}{\alpha} \right) \d t \\
    & - \frac{\gamma(q)}{2 \alpha} \left( \frac{F (Z_{r, q}, Y)^2}{\alpha} - (q - r^2) \right) \Bigg].
\end{align*}
\item [(b)] \emph{Weak duality.} For any $(r, q) \in A_{\sBayes}$, we have
\begin{equation}\label{eq:weak_dual}
    \VH_{1,\alpha,\varphi}^{\sAMP} (r,q,h) \le \, \inf_{(\mu, c) \in \mathscr{L} (q)} \mathsf{F} (\mu, c, r).
\end{equation}
%
\item [(c)] \emph{Strong duality.} For any fixed $q \in [0, 1]$, denote $A_{\sBayes} (q) = \{ r \in [-1, 1]: (r, q) \in A_{\sBayes} \}$. Assume there exists $(\mu_*, c_*) \in \FS(q)$ and $r_* \in A_{\sBayes} (q)$, such that the following holds:
\begin{itemize}
    \item [(i)] The mapping $r \mapsto \inf_{(\mu, c) \in \FS (q)} \mathsf{F} (\mu, c, r)$ is differentiable at $r_*$;
    \item [(ii)] $(\mu_*, c_*, r_*)$ satisfies
    \begin{equation*}
        \mathsf{F} (\mu_*, c_*, r_*) = \inf_{(\mu, c) \in \FS (q)} \mathsf{F} (\mu, c, r_*) = \sup_{r \in A_{\sBayes} (q)} \inf_{(\mu, c) \in \FS (q)} \mathsf{F} (\mu, c, r).
    \end{equation*}
\end{itemize}
Then, there exists a feasible pair $(F^*, \phi^*)$ such that $F^*$ is an $(r_*, q)$-contraction, and
\begin{equation*}
    \E [(\phi_t^*)^2] \le \frac{1}{\alpha}, \ \forall t \in [q, 1].
\end{equation*}
Furthermore,
\begin{equation}\label{eq:strong_duality_final}
    \mathsf{F} (\mu_*, c_*, r_*) = \, \E \left[ h \left( Y, \, Z_{r_*, q} + \frac{1}{\alpha} F^* \left( Z_{r_*, q}, Y \right) + \int_{q}^{1} \left( 1 + \phi_t^* \right) \d B_t \right) \right] = \VH_{1,\alpha,\varphi}^{\sAMP} (r_*, q, h).
\end{equation}
\end{itemize}
\end{thm}
The proof of Theorem~\ref{thm:two_stage_strong_duality} is given in Section~\ref{sec:proof_general_q}, with certain technical details deferred to \cref{sec:proof_variation}. Note that, under the conditions of point $(c)$, we have the dual characterization 
\begin{align*}
    \VH_{1,\alpha,\varphi}^{\sAMP} (r_*, q, h) = \sup_{r \in A_{\sBayes} (q)} \inf_{(\mu, c) \in \FS (q)} \mathsf{F} (\mu, c, r)\, .
\end{align*}
Namely, the optimal value achievable by our two-stage AMP algorithm is characterized by the Parisi variational principle on an extended function space.

%
%

\section{Two-stage AMP algorithm: Proof of Theorem \ref{thm:AMP_Feasibility_general}}\label{sec:iamp}
This section will be devoted to establishing our general two-stage AMP algorithm and the proof of Theorem~\ref{thm:AMP_Feasibility_general}, with proofs of auxiliary results deferred to Appendix~\ref{sec:proof_iamp}.

Recall from \cref{ass:linear_signal} that $\{ (\xx_i, y_i) \}_{i \in [n]}$ are i.i.d. pairs satisfying
\begin{equation*}
	\xx_i \sim \normal( \bzero, \id_d), \quad y_i = \varphi (\WW_*^\sT \xx_i, \veps_i), \,\, \veps_i \sim P_{\veps}.
\end{equation*}
Using vector notation, we will write $\yy = \varphi (\XX \WW_*, \vveps)$ occasionally, 
where $\varphi$ is understood to act on its arguments row-wise.
We further note that there is no real loss of generality
in assuming $\WW_* \sim\Unif( O(d, k))$, the uniform distribution (Haar measure) on the Stiefel manifold $O(d, k)$ (see \cref{lem:Reduction_Signal} for the precise statement). Indeed, by rotation invariance
of the covariates $\xx_i$, $\Unif( O(d, k))$ is the least favorable prior distribution. 
Under this simplification, we will assume that the empirical distribution of the rows of 
$\sqrt{n} \WW_*$ converges in $W_2$ distance to $\normal(0, \alpha I_k)$ as 
$n, d \to \infty$, $n/d \to \alpha$.


\subsection{Approximate message passing}
Following \cite{bayati2011dynamics, javanmard2013state, celentano2020estimation}, we define the general AMP algorithm as an iterative procedure that generates two sequences of matrices $\{ \WW^t \}_{t \ge 1} \subset \R^{d \times m}$ and $\{ \VV^t \}_{t \ge 1} \subset \R^{n \times m}$ according to:
\begin{equation}\label{eq:amp_iter}
\begin{split}
    \WW^{t + 1} & = \frac{1}{\sqrt{n}} \XX^\sT F_t (\VV^{\le t}, \yy) - \sum_{s=1}^{t} G_s (\WW^{\le s}) K_{t, s}^\sT, \\
	\VV^{t} & = \frac{1}{\sqrt{n}} \XX G_t(\WW^{\le t}) - \sum_{s=1}^{t} F_{s-1} (\VV^{\le s-1}, \yy) D_{t, s}^\sT,
\end{split}
\end{equation}
where $\WW^{1} = \XX^\sT F_0 (\yy) / \sqrt{n}$, and $\{ F_t: \R^{mt + 1} \to \R^m \}_{t \ge 0}$ and $\{ G_t: \R^{mt} \to \R^m \}_{t \ge 1}$ are two sequences of Lipschitz functions. Moreover, we let $\WW^{\le t} = (\WW^{s})_{1 \le s \le t}$, $\VV^{\le t} = (\VV^{s})_{1 \le s \le t}$, and adopt the convention that the Lipschitz functions $G_t$ and $F_t$ apply row-wise to their arguments. The $m \times m$ matrices $D_{t, s}$ and $K_{t, s}$ are defined as
\begin{equation}\label{eq:onsager_terms}
	D_{t, s} = \frac{1}{n} \sum_{i=1}^{d} \frac{\partial G_t}{\partial \ww_i^s} (\ww_i^1, \cdots, \ww_i^t), \ \,\, K_{t, s} = \frac{1}{n} \sum_{i=1}^{n} \frac{\partial F_t}{\partial \vv_i^s} (\vv_i^1, \cdots, \vv_i^t, y_i),
\end{equation}
where $\ww_i^s$ is the $i$-th row of $\WW^s$ and $\vv_i^s$ is the $i$-th row of $\VV^s$, respectively.
\begin{rem}
We will refer to $D_{t, s}$ and $K_{t, s}$ in \cref{eq:onsager_terms} as ``Onsager coefficients".
	The population versions of these coefficients are used in some of the earlier
 literature, where the empirical average over $i$ is replaced by an expectation over the
 asymptotic distributions of the $\ww^t_i$'s and $\vv^t_i$'s. 
 By an induction argument in \cite{javanmard2013state}, the high-dimensional asymptotics of
 these two versions of the general AMP algorithm are the same. 
\end{rem}

As $n, d \to \infty$ and $n/d \to \alpha$, for any fixed $t \in \mathbb{N}$, the limiting joint distribution of the first $t$ AMP iterates is exactly characterized by the following proposition.
\begin{prop}[State evolution of AMP]\label{prop:state_evolution}
	Let $Y \in \R$ and $\overline{Z}_0, V \in \R^k$ be row vectors. Denote $\overline{Z}_{\le t} := (\overline{Z}_1, \cdots, \overline{Z}_t) \in \R^{mt}$ and $Z_{\le t} := (Z_1, \cdots, Z_t) \in \R^{mt}$, where each $Z_t, \overline{Z}_t \in \R^m$ is a row vector. Then, the distributions of the random row vectors $(\overline{Z}_0, \overline{Z}_{\le t}, Y) \in \R^{mt + k + 1}$ and $(Z_{\le t}, V) \in \R^{mt + k}$ are defined as follows: (i) $V \sim \normal (0, \alpha I_k)$ is independent of $(Z_i)_{i \ge 1}$, and $\veps \sim P_{\veps}$ is independent of $(\overline{Z}_i)_{i \ge 0}$, (ii) both $(\overline{Z}_0, \overline{Z}_{\le t} )$ and $Z_{\leq t}$ are multivariate Gaussian with zero mean, and their covariance structures are specified via the following recursion:
\begin{equation}\label{eq:covariance}
\begin{split}
    & \E \left[ \bar{Z}_i^\sT \bar{Z}_j \right] = \frac{1}{\alpha} \mathbb{E}\left[G_i \left( V R_{\leq i} + Z_{\leq i} \right)^\sT G_j \left( V R_{\leq j} + Z_{\leq j} \right)\right], \quad i, j \geq 1, \\
    & \E \left[ \bar{Z}_i^\sT \bar{Z}_0 \right] = \frac{1}{\alpha} \mathbb{E}\left[G_i \left( V R_{\leq i} + Z_{\leq i} \right)^\sT V \right], \quad \E \left[ \bar{Z}_0^\sT \bar{Z}_0 \right] = \frac{1}{\alpha} \mathbb{E}\left[ V^\sT V \right] = I_k, \quad i \geq 1, \\
    & \E \left[ Z_i^\sT Z_j \right] =\mathbb{E}\left[F_{i-1}\left(\overline{Z}_{\leq i-1} , Y\right)^\sT F_{j-1}\left(\overline{Z}_{\leq j-1} , Y \right)\right], \quad i, j \geq 1, \\
    & Y = \varphi \left( \bar{Z}_0, \veps \right), \ R_{t+1}=\mathbb{E}\left[ \frac{\partial F_t}{\partial \bar{z}_0} \left(\overline{Z}_{\leq t} , \varphi \left( \bar{Z}_0, \veps \right) \right)\right] \in \R^{k \times m}, \quad t \ge 0.
\end{split}
\end{equation}
	In the above display, we denote $R_{\le t} = (R_1, \cdots, R_t) \in \R^{k \times mt}$ for all $t \ge 1$, so that $V R_{\le t} + Z_{\le t} = (V R_1 + Z_1, \cdots, V R_t + Z_t) \in \R^{mt}$, and in the definition of $R_{t+1}$, $\partial F_t / \partial \overline{z}_0$ represents the Jacobian of $F_t$ with respect to $\overline{Z}_0 \in \R^{k}$. Under this specification, we have
	\begin{equation*}
		\plim_{n \to \infty} D_{t, s} = \frac{1}{\alpha} \E \left[ \frac{\partial G_t}{\partial w^s} \left( V R_{\le t} + Z_{\le t} \right) \right], \ \plim_{n \to \infty} K_{t, s} = \E \left[ \frac{\partial F_t}{\partial v^s} \left( \bar{Z}_{\le t}, Y \right) \right].
	\end{equation*}
	Furthermore, denoting by $\vv_i \in \R^k$ the $i$-th row of $\sqrt{n} \WW_* \in \R^{d \times k}$ for $i \in [d]$, we have for any pseudo-Lipschitz functions $\psi_1: \R^{mt + k} \to \R$ and $\psi_2: \R^{mt + 1} \to \R$:
\begin{equation}\label{eq:state_evolution}
\begin{split}   
    \lim_{n \to \infty} \frac{1}{d} \sum_{i=1}^{d} \psi_1 \left( \ww_i^1, \cdots, \ww_i^t, \vv_i \right) & = \E \left[ \psi_1 \left( V R_{\le t} + Z_{\le t}, V \right) \right], \\
    \lim_{n \to \infty} \frac{1}{n} \sum_{i=1}^{n} \psi_2 \left( \vv_i^1, \cdots, \vv_i^t, y_i \right) & = \E \left[ \psi_2 \left( \bar{Z}_{\le t}, Y \right) \right],
\end{split}
\end{equation}
almost surely as $n, d \to \infty$ and $n/d \to \alpha$.
\end{prop}
\begin{proof}
	This can be deduced from the results in \cite{javanmard2013state, celentano2020estimation, montanari2022statistically}.
\end{proof}

\begin{rem}\label{rmk:Randomness}
As mentioned in the introduction, we allow our two-stage AMP algorithm to be randomized. Within the framework of this section, randomization can be implemented by letting 
the functions $F_t$ (or $G_t$) depend on some additional randomness. For example, one can replace
$F_t(\vv_i^1,\dots,\vv^t_i,y_i)$ by 
$F_t(\vv_i^1,\dots,\vv^t_i,y_i,\omega_i)$ with $(\omega_i)_{i\ge 1}\sim_{\iid} \Unif [0,1]$, the uniform distribution on $[0, 1]$.
For simplicity of notation, we will leave this dependence implicit.
Expectations in the state evolution equations are understood to be taken with respect to
these random variables as well.
\end{rem}
\begin{cor}
	As $n, d \to \infty$, $n / d \to \alpha$, the empirical distribution of $(\vv_i^1, \cdots, \vv_i^t, y_i)_{1 \le i \le n}$ almost surely weakly converges to the law of $( \bar{Z}_{\le t}, Y )$. Similarly, the empirical distribution of $(\ww_i^1, \cdots, \ww_i^t, \vv_i)_{1 \le i \le d}$ almost surely weakly converges to the law of $( V R_{\le t} + Z_{\le t}, V )$.
\end{cor}
\begin{proof}
Identical to the proof of \cite[Corollary 4.2]{montanari2024exceptional}.
\end{proof}

\subsection{First stage: Fixed-point AMP}\label{sec:fixed_pt_amp}
In this section, we present the first stage of our general two-stage AMP algorithm, which consists of several vanilla AMP iterations. The AMP iterations in this stage eventually converge to a fixed point of the state evolution equations~\eqref{eq:covariance}, which will be the starting point of the pure incremental part in the second stage.

We now specify our choices of $F_t$ and $G_t$. For this stage, we let $F_t$ be a function that only depends on $(\VV^t, \yy)$, and let $G_t$ be the identity mapping of $\WW^t$. Consequently, $D_{t} = (d/n) I_m$ and the AMP iterations reduce to
\begin{equation}\label{eq:amp_iter_stage_1}
\begin{split}
    \WW^{t + 1} & = \frac{1}{\sqrt{n}} \XX^\sT F_t (\VV^{t}, \yy) - \WW^{t} K_{t}^\sT, \\
	\VV^{t} & = \frac{1}{\sqrt{n}} \XX \WW^{t} - \frac{d}{n}  F_{t-1} (\VV^{t-1}, \yy) ,
\end{split}
\end{equation}
where we still have $\WW^{1} = \XX^\sT F_0 (\yy) / \sqrt{n}$, and
\begin{equation}\label{eq:onsager_terms_stage_1}
	 K_{t} = \frac{1}{n} \sum_{i=1}^{n} \frac{\partial F_t}{\partial \vv_i^t} (\vv_i^t, y_i)
\end{equation}
are the Onsager terms. With this simple choice, the state evolution equations~\eqref{eq:covariance} reduce to
\begin{equation}\label{eq:covariance_stage_1}
\begin{split}
& \E \left[ \bar{Z}_i^\sT \bar{Z}_j \right] =  R_i^\sT R_j + \frac{1}{\alpha} \E \left[ Z_i^\sT Z_j \right], \quad i, j \geq 1, \\
& \E \left[ \bar{Z}_i^\sT \bar{Z}_0 \right] = R_i^\sT, \quad \E \left[ \bar{Z}_0^\sT \bar{Z}_0 \right] = I_k, \quad i \geq 1, \\
& \E \left[ Z_1^\sT Z_1 \right] =\mathbb{E}\left[F_0 \left( Y \right)^\sT F_0 \left( Y \right)\right], \\
& \E \left[ Z_i^\sT Z_1 \right] =\mathbb{E}\left[F_{i-1} \left(\overline{Z}_{i-1} , Y\right)^\sT F_0 \left( Y \right)\right], \quad i \geq 2, \\
& \E \left[ Z_i^\sT Z_j \right] =\mathbb{E}\left[F_{i-1} \left(\overline{Z}_{i-1} , Y\right)^\sT F_{j-1} \left(\overline{Z}_{j-1} , Y \right)\right], \quad i, j \geq 2, \\
& Y = \varphi \left( \bar{Z}_0, \veps \right), \ R_{1}=\mathbb{E}\left[ \frac{\partial F_0}{\partial \bar{z}_0} \left( \varphi \left( \bar{Z}_0, \veps \right) \right)\right] \in \R^{k \times m}, \\
& R_{t+1}=\mathbb{E}\left[ \frac{\partial F_t}{\partial \bar{z}_0} \left(\overline{Z}_{t} , \varphi \left( \bar{Z}_0, \veps \right) \right)\right] \in \R^{k \times m}, \quad t \ge 1.
\end{split}
\end{equation}

In order to specify our choices of the Lipschitz functions $\{ F_t \}$ in this stage, we first define two crucial quantities that track the state evolution of AMP:
\begin{equation*}
	R_t = \E \left[ \bar{Z}_0^\sT \bar{Z}_t \right] \in \R^{k \times m}, \quad Q_t = \E \left[ \bar{Z}_t^\sT \bar{Z}_t \right] \in \Sym_+^m.
\end{equation*}
Recall the Bayes AMP feasible region from \cref{defn:Bayes_AMP_region_general}:
\begin{equation*}
	A_{\sBayes} = \, \left\{ (R, Q) \in \R^{k \times m} \times \Sym_{+}^{m}: \, R^\sT R \preceq Q \preceq I_m, \, R (Q - R^\sT R)^{-1} R^\sT \preceq C_{\sBayes} C_{\sBayes}^\sT \right\}.
\end{equation*}
The proposition below establishes the connection between $A_{\sBayes}$ and the collection of $(R_t, Q_t)$'s achievable via vanilla AMP iterations.
\begin{prop}\label{prop:Bayes_feasible_region}
	Under state evolution~\eqref{eq:covariance_stage_1}, the following hold:
	\begin{itemize}
		\item [(a)] For all $t \in \mathbb{N}$, $(R_t, Q_t) \in \R^{k \times m} \times \Sym_+^m$ satisfies
			\begin{equation*}
				R_t \big( Q_t - R_t^\sT R_t \big)^{-1} R_t^\sT \preceq C_{\sBayes} C_{\sBayes}^\sT.
			\end{equation*}
			Further, $(R_t, Q_t) \in A_{\sBayes}$ if $Q_t \preceq I_m$.
		\item [(b)] For any $(R, Q) \in \operatorname{int} A_{\sBayes}$, there exists $T_{11} \in \mathbb{N}$ and Lipschitz functions $\{ F_t \}_{t \le T_{11} - 1}$, such that $(R_{T_{11}}, Q_{T_{11}}) = (R, Q)$.
	\end{itemize}
\end{prop}

Now for any $(R, Q) \in \operatorname{int} A_{\sBayes}$, we can find some $T_{11} \in \mathbb{N}$ and $\{ F_t \}_{t \le T_{11} - 1}$ such that $(R_{T_{11}}, Q_{T_{11}}) = (R, Q)$. Let $F$ be an $(R, Q)$-contraction as defined in \cref{defn:mu_Q_contraction}. We set $F_t = F$ for all $T_{11} \le t \le T_{1} - 1$, where $T_{1} \in \mathbb{N}$ is a sufficiently large integer to be determined.
\begin{lem}\label{lem:cov_mu_Q}
	Denote $P_t = \E [ \bar{Z}_{t+1}^\sT \bar{Z}_t ]$. Then for all $T_{11} + 1 \le t \le T_{1}$, we have $(R_t, Q_t) = (R, Q)$, and $\lim_{T_{1} \to \infty} P_{T_{1}-1} = Q$.
\end{lem}
As discussed in \cref{sec:IAMP}, this fixed-point AMP stage will end after $T_1$ steps, with $T_1 \in \mathbb{N}$ to be determined later.

%
%

\subsection{Second stage: Incremental AMP}

In this section, we describe the second stage of our algorithm, which is an incremental AMP (IAMP) procedure first introduced in \cite{montanari2021optimization}.
We will see that the asymptotics of this incremental stage admit a
stochastic integral representation under a suitable scaling limit. 

For this IAMP stage, the non-linear functions $\{ F_t \}_{t \ge T_1}$ and $\{ G_t \}_{t \ge T_1 + 1}$ depend on all previous AMP iterations, and are chosen to satisfy the following assumption:
\begin{ass}\label{ass:F_t_and_G_t}
Consider the random variables $((V^t)_{t \ge 1}, Y)$ and $((W^t)_{t \ge 1}, V)$ defined as follows:
\begin{equation*}
	\begin{split}
		& (V^t)_{1 \le t \le T_1} = (\bar{Z}_t)_{1 \le t \le T_1}, \ Y = \varphi \left( \bar{Z}_0, \veps \right), \ \bar{Z}_0 \sim \normal \left( 0, I_k \right), \ \veps \sim P_{\veps}, \ \bar{Z}_0 \perp\!\!\!\perp \veps, \\
		& (V^t)_{t \ge T_1 + 1} \sim_{\iid} \normal(0, I_m), \ (V^t)_{t \ge T_1 + 1} \perp\!\!\!\perp \left( (V^t)_{1 \le t \le T_1}, Y \right) , \\
		& (W^t)_{1 \le t \le T_1} = (Z_t)_{1 \le t \le T_1}, \ V \perp\!\!\!\perp (W^t)_{t \ge 1}, \ V \sim \normal (0, \alpha I_k) , \\
		& (W^t)_{t \ge T_1 + 1} \sim_{\iid} \normal(0, I_m), \ (W^t)_{t \ge T_1 + 1} \perp\!\!\!\perp \left( (W^t)_{1 \le t \le T_1}, V \right) .
	\end{split}
\end{equation*}
We impose the following second moment constraints on $\{ F_t \}_{t \ge T_1}$ and $\{ G_t \}_{t \ge T_1 + 1}$:
\begin{enumerate}
    \item $F_{T_1}$ is only a function of $Y = \varphi \left( \bar{Z}_0, \veps \right)$ with $\E [F_{T_1} (Y)^\sT F_{T_1} (Y)] = I_m$. Recall from \cref{eq:covariance} that $R_{T_1 + 1} = \E [\partial_{\overline{z}_0} F_{T_1} (\varphi(\overline{Z}_0, \veps))]$. We also require that $G_{T_1 + 1}$ is only a function of $V R_{T_1 + 1} + W^{T_1 + 1}$ satisfying
    \begin{equation*}
    	\E \left[ G_{T_1 + 1} \left( V R_{T_1 + 1} + W^{T_1 + 1} \right)^\sT G_{T_1 + 1} \left( V R_{T_1 + 1} + W^{T_1 + 1} \right) \right] = \alpha I_m.
    \end{equation*} 
    Furthermore, recalling $F$ and $(R, Q)$ from \cref{sec:fixed_pt_amp}, we require
    \begin{align*}
    	\E_{(R, Q)} \left[ F_{T_1} \left( \varphi \left( \bar{Z}_0, \veps \right) \right)^\sT F \left(\overline{Z} , \varphi \left( \bar{Z}_0, \veps \right) \right) \right] = \, & 0, \\
    	\E \left[ G_{T_1 + 1} \left( V R_{T_1 + 1} + W^{T_1 + 1} \right)^\sT V \right] = \, & 0,
    \end{align*}
    where we recall that $\E_{(R, Q)}$ represents the expectation taken under $(\bar{Z}, \bar{Z}_0)^\sT \sim \normal \left( 0, \begin{bmatrix}
		Q & R^\sT \\
		R & I_k 
		\end{bmatrix} \right)$, independent of $\veps \sim P_{\veps}$. 
    \item The functions $\{ F_t \}_{t \ge T_1 + 1}$ and $\{ G_t \}_{t \ge T_1 + 2}$ have the form:
    \begin{equation*}
        F_t (V^{\le t}, Y) = V^t \Phi_{t-1} (V^{\le t-1}, Y), \ G_t(W^{\le t}) = W^t \Psi_{t-1} (W^{\le t-1}),
    \end{equation*}
    where $\Phi_{t-1}: \R^{m(t-1) + 1} \to \R^{m \times m}$ and $\Psi_{t-1}: \R^{m(t-1)} \to \R^{m \times m}$ are $m \times m$ matrix-valued functions satisfying
    \begin{equation*}
    \begin{split}
    	& \E \left[ \Phi_{t-1} \left( V^{\le t-1}, Y \right)^\sT \Phi_{t-1} \left( V^{\le t-1}, Y \right) \right] = I_m, \\
    	& \E \left[ \Psi_{t-1} \left( ( V R_s + W^s)_{s=1}^{T_1 + 1}, (W^s)_{s = T_1 + 2}^{t-1} \right)^\sT \Psi_{t-1} \left( ( V R_s + W^s)_{s=1}^{T_1 + 1}, (W^s)_{s = T_1 + 2}^{t-1} \right) \right] = \alpha I_m.
    \end{split}
    \end{equation*}
\end{enumerate}
\end{ass}
It is straightforward to verify that there exist functions $\{ F_t \}_{t \ge T_1}$ and $\{ G_t \}_{t \ge T_1 + 1}$ satisfying the above assumption. Under such choices, the state evolution of the IAMP stage is characterized by the following:
\begin{prop}\label{prop:state_evolution_iamp}
	Under Assumption~\ref{ass:F_t_and_G_t}, $(Z_t)_{t \ge T_1 + 1} \sim_{\iid} \normal (0, I_m)$ is independent of $((Z_t)_{1 \le t \le T_1}, V)$, and $(\bar{Z}_t)_{t \ge T_1 + 1} \sim_{\iid} \normal (0, I_m)$ is independent of $((\bar{Z}_t)_{1 \le t \le T_1}, Y)$. Further, $R_{t} = 0$ for all $t \ge T_1 + 2$.
\end{prop}

The IAMP stage is run for $T_2$ steps after the fixed-point AMP stage, for some $T_2 \in \mathbb{N}$ to be determined. In the next section, we combine the two stages of our AMP algorithm to complete the proof of \cref{thm:AMP_Feasibility_general}.

\subsection{Combining the two stages}
Let $\WW_F = \WW^{T_1} / \sqrt{n}$ be the output of the fixed-point AMP stage. By state evolution equations~\eqref{eq:covariance_stage_1} and \cref{lem:cov_mu_Q}, we have
\begin{align*}
    \WW_F^\sT \WW_F = \, & \frac{1}{n} \left( \WW^{T_1} \right)^\sT \WW^{T_1} = \frac{1}{n} \sum_{i=1}^{d} \left( \ww_i^{T_1} \right)^\sT \ww_i^{T_1} \\
	\to \, & \frac{1}{\alpha} \E \left[ \left(V R + Z_{T_1} \right)^\sT \left( V R + Z_{T_1} \right) \right] \\
    = \, & R^\sT R + \frac{1}{\alpha} \E \left[ Z_{T_1}^\sT Z_{T_1} \right] = \E \left[ \bar{Z}_{T_1}^\sT \bar{Z}_{T_1} \right] = Q
\end{align*}
almost surely as $n \to \infty$. Let $Q_1, \cdots, Q_{T_2}$ be $T_2$ non-random $m \times m$ matrices such that $\sum_{t=1}^{T_2} Q_t^\sT Q_t = I_m - Q$. We then define
\begin{align*}
    \WW_I &= \frac{1}{\sqrt{n}} \sum_{t=1}^{T_2} G_{T_1 + t+1} \left( \WW^{\le T_1 + t + 1} \right) Q_t \,
\end{align*}
as the output of the IAMP stage. Following \cref{sec:IAMP}, we construct the final output of our two-stage AMP algorithm by letting $\WW_Q = \WW_F + \WW_I$, and setting $\hat{\WW}_n^{\sAMP} = \WW_Q (\WW_Q^\sT \WW_Q)^{-1/2}$. It is easy to see that $\hat{\WW}_n^{\sAMP} \in O(d, m)$, as we required in the definition of $\mathscr{F}_{m, \alpha, \varphi}^{\salg}$. The theorem below characterizes the set of $(\alpha, m)$-feasible distributions achievable by our algorithm:
\begin{thm}\label{thm:iamp_feasibility}
Let Assumption~\ref{ass:F_t_and_G_t} hold, and further assume that for all $t \ge 0$, $F_t$ is continuous, and for all $t \ge 1$, $G_t$ is Lipschitz continuous. If the weight matrix $\hat{\WW}_n^{\sAMP}$ is constructed as above, then $\lim_{n \to \infty} \WW_*^\sT \hat{\WW}_n^{\sAMP} = R$ almost surely. Further,
\begin{equation*}
    \frac{1}{n} \sum_{i=1}^{n} \delta_{\left( y_i, \, \xx_i^\sT \hat{\WW}_n^{\sAMP} \right)} \stackrel{w}{\to} \operatorname{Law} \left( Y, \ \bar{Z}_{T_1} + \frac{1}{\alpha} F \left( \bar{Z}_{T_1-1}, Y \right) + \sum_{t=1}^{T_2} \left( \bar{Z}_{T_1 + t+1} + F_{T_1+t} \left( \bar{Z}_{\le T_1+t}, Y \right) A_t \right) Q_t \right)
\end{equation*}
almost surely as $n \to \infty$, where $Y = \varphi (\bar{Z}_0 , \veps)$, and
\begin{equation*}
	A_t = \frac{1}{\alpha} \E \left[ \Psi_{T_1+t} \left( V R_{\leq T_1+1} + Z_{\leq T_1+1}, (Z_t)_{T_1+2 \le t \le T_1+t} \right) \right], \ 1 \le t \le T_2.
\end{equation*}
In fact, the above holds for any sequence of $m \times m$ matrices $(A_t)_{1 \le t \le T_2}$ satisfying $A_{t}^\sT A_{t} \preceq I_m / \alpha$, and functions $(F_{T_1 + t})_{1 \le t \le T_2}$ satisfying point 2 of Assumption~\ref{ass:F_t_and_G_t} (not necessarily continuous).
\end{thm}
We next take the limit $T_1 \to \infty$. Using \cref{lem:cov_mu_Q}, we know that
\begin{equation*}
    \lim_{T_1 \to \infty} \E \left[ \left\| \overline{Z}_{T_1} - \overline{Z}_{T_1-1} \right\|_2^2 \right]  = 0,
\end{equation*}
and that
$(\overline{Z}_{T_1}, \overline{Z}_0)^\sT$ has the same distribution as $(Z_{R, Q}, G)$ in \cref{defn:mu_Q_contraction}. 
Leveraging our assumptions on $F$ and $\{ F_{T_1 + t} \}_{t=1}^{T_2}$, the above characterization of AMP-feasible distributions can be further simplified. Note that in the theorem below, we recast $T_2$ as $T$ and the $\overline{Z}_t$'s as $V^t$'s, and transpose the vectors $V^t$ and matrices $(\Phi_t, Q_t)$ to align with the notation of \cref{thm:AMP_Feasibility_general}.
\begin{thm}\label{thm:FeasibleDiscrete}
    Let $(R, Q) \in A_{\sBayes}$, $(Y, Z_{R, Q})$ be as described in \cref{defn:mu_Q_contraction}, and $F$ be an $(R, Q)$-contraction. For any $T \in \mathbb{N}_+$, let $(V^t)_{1 \le t \le T} \sim_{\iid} \sN (0, I_m)$ be independent of $(Z_{R, Q}, Y)$. Define
\begin{equation}\label{eq:discrete_riemann_sum}
    U = Z_{R, Q} + \frac{1}{\alpha} F \left( Z_{R, Q}, Y \right) + \sum_{t=1}^{T} Q_t \left( V^{t+1} + \Phi_{t-1} \left( V^{\le t-1}, Z_{R, Q}, Y \right) V^{t} \right),
\end{equation}
where
\begin{equation*}
    \E \left[ \Phi_{t-1} \left( V^{\le t-1}, Z_{R, Q}, Y \right) \Phi_{t-1} \left( V^{\le t-1}, Z_{R, Q}, Y \right)^\sT \right] \preceq \frac{I_m}{\alpha}, \ \forall t \ge 1, \ \sum_{t=1}^{T} Q_t Q_t^\sT = I_m - Q.
\end{equation*}
Then, the same conclusion as in \cref{thm:AMP_Feasibility_general} holds with $U$ defined as per \cref{eq:discrete_riemann_sum}.
\end{thm}
Finally, we take the scaling limit $T \to \infty$ for the AMP-feasible distribution described in Theorem \ref{thm:FeasibleDiscrete}, thus establishing the stochastic integral representation of \cref{thm:AMP_Feasibility_general}. As the proof is nearly identical to the unsupervised setting in \cite{montanari2024exceptional}, we omit it here for ease of presentation.

%
%
\section{Dual characterization and Parisi formula: Proof of Theorem \ref{thm:two_stage_strong_duality}}\label{sec:char_opt_ctrl}

This section will be devoted to the proof
of Theorem \ref{thm:two_stage_strong_duality}, thus yielding a dual characterization of 
$\VH_{1,\alpha, \varphi}^{\sAMP}(h)$. We begin with some necessary technical preliminaries.

\subsection{Technical preliminaries}\label{sec:tech_prelim}
In this section, we gather several key technical results from \cite{montanari2024exceptional} that are essential for proving \cref{thm:two_stage_strong_duality}. In particular, we establish the existence and uniqueness of a (weak) solution to the Parisi PDE, develop the verification argument which connects the Parisi PDE to the Hamilton-Jacobi-Bellman (HJB) equation, and compute the first-order variation of the Parisi functional with respect to $(\mu, c, r)$.

\paragraph{Solution to the Parisi PDE.} Theorem 3.3 in \cite{montanari2024exceptional} implies the following:
\begin{thm}\label{thm:Parisi_solution_supervised}
    Fix $y \in \R$, assume that $h(y, \cdot)$ is $C^2$, Lipschitz and bounded from above. Then, for any $(\mu, c) \in \FS$ and $\gamma \in \FSG$ satisfying $\mu = \gamma' / \gamma^2$ and $c = 1 / \gamma(1)$, the Parisi PDE \eqref{eq:ParisiEq-m1-bis} admits a unique weak solution $f_{y, \mu}$ such that $f_{y, \mu} (t, \cdot) \in C^2 (\R)$ for all $t \in [0, 1]$. Further,
    \begin{align*}
        \norm{\partial_x f_{y, \mu} (t, \cdot)}_{L^{\infty} (\R)} \le \, & \norm{\partial_x h(y, \cdot)}_{L^{\infty} (\R)}, \, \forall t \in [0, 1], \\
        \partial_x^2 f_{y, \mu} (t, x) > \, & - \gamma(t), \, \forall (t, x) \in [0, 1] \times \R.
    \end{align*}
    If, additionally $\sup_{z \in \R} \partial_z^2 h(y, z) < 1/c$, then we have
    \begin{equation*}
        \partial_x^2 f_{y, \mu} (t, x) \le \, C, \, \forall (t, x) \in [0, 1] \times \R,
    \end{equation*}
    where the constant $C$ only depends on $(\mu, c)$ and $\sup_{z \in \R} \partial_z^2 h(y, z)$.
\end{thm}
According to \cref{thm:Parisi_solution_supervised}, we know that the Parisi functional is well-defined on $\FS$.

\paragraph{Verification argument.} 
Define for any $(t, y, z) \in [0, 1] \times \R \times \R$, the following value function:
\begin{equation}\label{eq:redefine_value_func}
    V_{\gamma} (t, y, z) = \, \sup_{\phi \in D [t, 1]} \E \left[ h \left( y, z+ \int_{t}^{1} \left( 1 + \phi_s \right) \d B_s \right) - \frac{1}{2} \int_{t}^{1} \gamma(s) \left( \phi_s^2 - \frac{1}{\alpha} \right) \d s \right].
\end{equation}
Then, applying \cite[Proposition 5.3]{montanari2024exceptional} for fixed $y \in \R$ leads to the following dual relationship between $V_{\gamma}$ and $f_{y, \mu}$:
\begin{thm}\label{prop:veri_arg_supervised}
    Under the conditions of \cref{thm:Parisi_solution_supervised}, we have for all $t \in [0, 1]$ and $y, x, z \in \R$:
    \begin{equation}\label{eq:two_stage_V_and_f}
	\begin{split}
		V_{\gamma} (t, y, z) =\, & \inf_{x \in \R} \left\{ f_{y, \mu} (t, x) + \frac{\gamma(t)}{2} (x - z)^2 \right\} + \frac{1}{2 \alpha} \int_{t}^{1} \gamma(s) \d s, \\
		f_{y, \mu} (t, x) =\, & \sup_{z \in \R} \left\{ V_{\gamma} (t, y, z) - \frac{\gamma(t)}{2} (z - x)^2 \right\} - \frac{1}{2 \alpha} \int_{t}^{1} \gamma(s) \d s.
	\end{split}
	\end{equation}
\end{thm}

\paragraph{Analysis of the variational problem.} Having established the existence and uniqueness of solutions to the Parisi PDE, and its dual relationship with the value function, we are now ready to compute the first-order variations of the Parisi functional. We summarize these results in the theorem below and defer its proof to \cref{sec:proof_variation}.
\begin{thm}\label{thm:variation_compute}
Under the assumptions of \cref{thm:two_stage_strong_duality}, let $G \sim \normal (0, 1)$, $Y = \varphi (G, \veps)$, and $(X_t)_{t \in [r^2, 1]}$ solve the SDE (existence and uniqueness of the solution will be established in the proof):
    \begin{equation*}
        X_{r^2} = \, r G, \quad \d X_t = \mu(t) \partial_x f_{Y, \mu} (t, X_t) \d t + \d B_t, \, t \in [r^2, 1]
    \end{equation*}
    for $(\mu, c) \in \FS (q)$ with $q \ge r^2$, where $(B_t)_{t \in [r^2, 1]}$ is a standard Brownian motion independent of $(Y, G)$. Furthermore, for $t \in [r^2, 1]$, define
    \begin{equation*}
        M_t = \, \frac{1}{\gamma (t)} \partial_x f_{Y, \mu} (t, X_t) + X_t,
    \end{equation*}
    then there exists $(\phi_t)_{t \in [r^2, 1]} \in D [r^2, 1]$ satisfying
    \begin{equation*}
        M_t = \, M_{r^2} + \int_{r^2}^{t} (1 + \phi_s) \d B_s, \, \forall t \in [r^2, 1].
    \end{equation*}
    Finally, define for $(x, y) \in \R^2$:
    \begin{equation*}
        F(x, y) = \, \frac{\alpha}{\gamma (q)} \partial_x f_{y, \mu} (q, x) = \frac{\alpha}{\gamma (r^2)} \partial_x f_{y, \mu} (q, x).
    \end{equation*}
    ($\gamma(q) = \gamma (r^2)$ since $\mu \equiv 0$ on $[r^2, q]$.) The following hold:
    \begin{itemize}
        \item [(i)] $(Y, G, X_q) \stackrel{d}{=} (Y, G, Z_{r, q})$, and
        \begin{equation*}
        \begin{split}
            \mathsf{F} (\mu, c, r) = \, \E \Bigg[ & h_c \left( Y, \, X_q + \frac{1}{\alpha} F \left( X_q, Y \right) + \int_{q}^{1} \left( 1 + \phi_t \right) \d B_t \right) - \frac{1}{2} \int_{q}^{1} \gamma(t) \left( \phi_t^2 - \frac{1}{\alpha} \right) \d t \\
            & - \frac{\gamma(q)}{2 \alpha} \left( \frac{F (X_q, Y)^2}{\alpha} - (q - r^2) \right) \Bigg],
        \end{split}
        \end{equation*}
        where
        \begin{equation*}
            h_c (y, x) = \, \operatorname{conc}  \left( h (y, x) - \frac{x^2}{2 c} \right) + \frac{x^2}{2 c}.
        \end{equation*}
        \item [(ii)] $\forall 0 \le s < t \le 1$, we have
        \begin{equation*}
            \E \left[ \left( \partial_x f_{Y, \mu} (t, X_t) \right)^2 \right] - \E \left[ \left( \partial_x f_{Y, \mu} (s, X_s) \right)^2 \right] = \, \int_{s}^{t} \gamma (u)^2 \E \left[ \phi_u^2 \right] \d u.
        \end{equation*}
        \item [(iii)] Define
        \begin{equation*}
             g_c (y, x) = \frac{\partial h_c (y, x)}{\partial c}.
        \end{equation*}
        Assume that $\delta: [r^2, 1] \to \R$ is in $L^1 [r^2, 1]$ and $L^{\infty} [r^2, t]$ for all $t \in [r^2, 1)$. Then,
    \begin{equation*}
    \begin{split}
        \frac{\d}{\d u} \mathsf{F} \left( \mu + u \delta, c, r \right) \bigg\vert_{u = 0} = \, & \frac{1}{2} \int_{r^2}^{1} \delta (t) \left( \E \left[ \left( \partial_x f_{Y, \mu} (t, X_t) \right)^2 \right] - \frac{1}{\alpha} \int_{r^2}^{t} \gamma(s)^2 \d s \right) \d t, \\
        \frac{\d}{\d c} \mathsf{F} \left( \mu, c, r \right) = \, & \E \left[ g_c \left( Y, M_1 \right) \right] + \frac{1}{2} \left( \E \left[ \left( \partial_x f_{Y, \mu} (1, X_1) \right)^2 \right] - \frac{1}{\alpha} \int_{r^2}^{1} \gamma(t)^2 \d t \right), \\
        \frac{\d}{\d r} \mathsf{F} \left( \mu, c, r \right) = \, & \frac{\gamma (r^2)}{\alpha} \left( \frac{1}{q - r^2} \E \left[ \left( q G - r Z_{r, q} \right) F \left( Z_{r, q}, \varphi(G, \veps) \right) \right] - r \right).
    \end{split}
    \end{equation*}
    \end{itemize}
\end{thm}

\subsection{Proof of \cref{thm:two_stage_strong_duality}}\label{sec:proof_general_q}
\paragraph{Proof of $(a)$: Variational formula.} Since $\mu = 0$ on $[r^2, q]$, we know that the Parisi PDE degenerates to a standard heat equation on this interval:
\begin{equation*}
    \partial_t f_{y, \mu} (t, x) + \frac{1}{2} \partial_x^2 f_{y, \mu} (t, x) = \, 0, \, (t, x) \in [r^2, q] \times \R.
\end{equation*}
Therefore,
\begin{equation*}
    \E_{Y, G} \left[ f_{Y, \mu} \left( r^2, r G \right) \right] = \, \E_{Y, G, Z} \left[ f_{Y, \mu} \left( q, r G + \sqrt{q - r^2} Z \right) \right],
\end{equation*}
where $Z \sim \normal (0, 1)$ is independent of $(Y, G)$. Using the notation of \cref{defn:mu_q_contraction}, we deduce that
\begin{align*}
    \mathsf{F} (\mu, c, r) = \, & \E_{Y, G} \left[ f_{Y, \mu} \left( r^2, r G \right) \right] + \frac{1}{2 \alpha} \int_{r^2}^{1} \frac{\d t}{c + \int_{t}^{1} \mu(u) \d u } \\
    = \, & \E_{Y, G, Z} \left[ f_{Y, \mu} \left( q, r G + \sqrt{q - r^2} Z \right) \right] + \frac{1}{2 \alpha} \int_{r^2}^{1} \gamma(t) \d t \\
    = \, & \E \left[ f_{Y, \mu} \left( q, Z_{r, q} \right) \right] + \frac{1}{2 \alpha} \int_{r^2}^{1} \gamma(t) \d t.
\end{align*}
Applying \cref{prop:veri_arg_supervised} yields that
\begin{equation*}
    f_{y, \mu} (q, x) = \, \sup_{z \in \R} \left\{ V_{\gamma} (q, y, z) - \frac{\gamma(q)}{2} (z - x)^2 \right\} - \frac{1}{2 \alpha} \int_{q}^{1} \gamma(s) \d s,
\end{equation*}
which leads to
\begin{align*}
    \mathsf{F} (\mu, c, r) = \, & \E \left[ f_{Y, \mu} \left( q, Z_{r, q} \right) \right] + \frac{1}{2 \alpha} \int_{r^2}^{1} \gamma(t) \d t \\
    = \, & \E \left[ \sup_{z \in \R} \left\{ V_{\gamma} (q, Y, z) - \frac{\gamma(q)}{2} \left( z - Z_{r, q} \right)^2 \right\} \right] + \frac{1}{2 \alpha} \int_{r^2}^{q} \gamma(t) \d t \\
    \stackrel{(i)}{=} \, & \E \left[ \sup_{u \in \R} \left\{ V_{\gamma} \left( q, Y, Z_{r, q} + u \right) - \frac{\gamma(q)}{2} \left( u^2 - \frac{q - r^2}{\alpha} \right) \right\} \right] \\
    = \, & \sup_{F: \R^2 \to \R} \E \left[ V_{\gamma} \left( q, Y, Z_{r, q} + \frac{1}{\alpha} F (Z_{r, q}, Y) \right) - \frac{\gamma(q)}{2 \alpha} \left( \frac{F (Z_{r, q}, Y)^2}{\alpha} - (q - r^2) \right) \right],
\end{align*}
where $(i)$ is because of $\gamma(t) = \gamma(q)$ for $t \in [r^2, q]$. For any fixed $F$, we know that (analogous to \cite[Lemma 5.2]{montanari2024exceptional})
\begin{align*}
    & \E \left[ V_{\gamma} \left( q, Y, Z_{r, q} + \frac{1}{\alpha} F (Z_{r, q}, Y) \right) \right] \\
    = \, & \sup_{\phi \in D[q, 1]} \E \left[ h \left( Y, \, Z_{r, q} + \frac{1}{\alpha} F \left( Z_{r, q}, Y \right) + \int_{q}^{1} \left( 1 + \phi_t \right) \d B_t \right) - \frac{1}{2} \int_{q}^{1} \gamma(t) \left( \phi_t^2 - \frac{1}{\alpha} \right) \d t \right].
\end{align*}
We thus obtain that
\begin{align*}
    \mathsf{F} (\mu, c, r) = \, \sup_{ \substack{F: \R^2 \to \R \\ \phi \in D[q, 1] } } \E \Bigg[ & h \left( Y, \, Z_{r, q} + \frac{1}{\alpha} F \left( Z_{r, q}, Y \right) + \int_{q}^{1} \left( 1 + \phi_t \right) \d B_t \right) - \frac{1}{2} \int_{q}^{1} \gamma(t) \left( \phi_t^2 - \frac{1}{\alpha} \right) \d t \\
    & - \frac{\gamma(q)}{2 \alpha} \left( \frac{F (Z_{r, q}, Y)^2}{\alpha} - (q - r^2) \right) \Bigg],
\end{align*}
completing the proof of part $(a)$.

\paragraph{Proof of $(b)$: Weak duality.}
Define
\begin{equation}\label{eq:SOC_First_1dim_upperbd}
\begin{split}
    & \oVH_{1, \alpha, \varphi}^{\sAMP} (r, q, h) := \, \sup \, \E \left[ h \left( Y, \, Z_{r, q} + \frac{1}{\alpha} F \left( Z_{r, q}, Y \right) + \int_{q}^{1} \left( 1 + \phi_t \right) \d B_t \right) \right], \\
    & \mbox{s.t.} \quad q \ge \, r^2 + \frac{1}{\alpha} \E \left[ F \left( Z_{r, q}, \varphi(G, \veps) \right)^2 \right], \, \text{and} \sup_{t \in [0, 1]} \E \left[ \phi_t^2 \right] \le \frac{1}{\alpha}.
\end{split}
\end{equation}
By definition, we know that $\oVH_{1,\alpha,\varphi}^{\sAMP} (r,q,h) \ge \VH_{1,\alpha,\varphi}^{\sAMP} (r,q,h)$. Hence,
\begin{align*}
    & \VH_{1,\alpha,\varphi}^{\sAMP} (r,q,h) \\ 
    \le \, & \sup_{ \substack{F: \R^2 \to \R \\ \phi \in D[q, 1] } } \inf_{\gamma \in \FSG (q)} \E \Bigg[ h \left( Y, \, Z_{r, q} + \frac{1}{\alpha} F \left( Z_{r, q}, Y \right) + \int_{q}^{1} \left( 1 + \phi_t \right) \d B_t \right) - \frac{1}{2} \int_{q}^{1} \gamma(t) \left( \phi_t^2 - \frac{1}{\alpha} \right) \d t \\
    & \hspace{8em} - \frac{\gamma(q)}{2 \alpha} \left( \frac{F (Z_{r, q}, Y)^2}{\alpha} - (q - r^2) \right) \Bigg] \\
    \stackrel{(i)}{\le} \, & \inf_{\gamma \in \FSG (q)} \sup_{ \substack{F: \R^2 \to \R \\ \phi \in D[q, 1] } } \E \Bigg[ h \left( Y, \, Z_{r, q} + \frac{1}{\alpha} F \left( Z_{r, q}, Y \right) + \int_{q}^{1} \left( 1 + \phi_t \right) \d B_t \right) - \frac{1}{2} \int_{q}^{1} \gamma(t) \left( \phi_t^2 - \frac{1}{\alpha} \right) \d t \\
    & \hspace{8em} - \frac{\gamma(q)}{2 \alpha} \left( \frac{F (Z_{r, q}, Y)^2}{\alpha} - (q - r^2) \right) \Bigg] \\
    \stackrel{(ii)}{=} \, & \inf_{(\mu, c) \in \FS (q)} \mathsf{F}(\mu, c, r),
\end{align*}
where $(i)$ follows from the max-min inequality, $(ii)$ follows from the variational formula of part $(a)$. This concludes the proof of part $(b)$.

\paragraph{Proof of $(c)$: Strong duality.} 
By \cref{thm:variation_compute} (iii) and the envelope theorem, we know that $(\mu_*, c_*, r_*)$ verifies the first-order conditions:
\begin{align}
    \label{eq:FOC_1}
    & \E \left[ \left( \partial_x f_{Y, \mu_*} (t, X_t^*) \right)^2 \right] - \frac{1}{\alpha} \int_{r_*^2}^{t} \gamma_* (s)^2 \d s = \, 0, \quad \forall t \in [r_*^2, 1], \\
    \label{eq:FOC_2}
    & \E \left[ g_{c_*} \left( Y, M_1^* \right) \right] + \frac{1}{2} \left( \E \left[ \left( \partial_x f_{Y, \mu_*} (1, X_1^*) \right)^2 \right] - \frac{1}{\alpha} \int_{r_*^2}^{1} \gamma_* (t)^2 \d t \right) = \, 0, \\
    \label{eq:FOC_3}
    & \frac{1}{q - r_*^2} \E \left[ \left( q G - r_* Z_{r_*, q} \right) F \left( Z_{r_*, q}, \varphi(G, \veps) \right) \right] = r_*.
\end{align}
Using \cref{thm:variation_compute} (ii), we immediately obtain that $\E [(\phi_t^*)^2] = 1 / \alpha$, thus establishing the feasibility of $\phi^*$. Since we already have \cref{eq:FOC_3}, to show that $F^*$ is feasible, it suffices to prove
\begin{equation*}
    q = \, r_*^2 + \frac{1}{\alpha} \E \left[ F \left( Z_{r_*, q}, \varphi(G, \veps) \right)^2 \right], \quad \frac{1}{\alpha} \E \left[ \frac{\partial F}{\partial Z_{r_*, q}} \left( Z_{r_*, q}, \varphi(G, \veps) \right)^2 \right] \le \, 1.
\end{equation*}
The first equation in the above display follows by taking $t = q$ in \cref{eq:FOC_1}, while the second one can be proved using a similar argument as in \cite[Section 8.2]{montanari2024exceptional}. Having established the feasibility of $(F^*, \phi^*)$, \cref{thm:variation_compute} (i) then implies that
\begin{equation*}
    \mathsf{F} (\mu_*, c_*, r_*) = \, \E \left[ h_{c_*} \left( Y, \, Z_{r_*, q} + \frac{1}{\alpha} F \left( Z_{r_*, q}, Y \right) + \int_{q}^{1} \left( 1 + \phi_t^* \right) \d B_t \right) \right] = \E \left[ h_{c_*} \left( Y, \, M_1^* \right) \right].
\end{equation*}
Note that \cref{eq:FOC_1,eq:FOC_2} together imply that $\E [g_{c_*} (Y, M_1^*)] = 0$. Applying Proposition C.4 in \cite{montanari2024exceptional}, we get that 
\begin{equation*}
    \E \left[ h (Y, M_1^*) \right] = \, \E \left[ h_{c_*} (Y, M_1^*) \right] = \mathsf{F} (\mu_*, c_*, r_*).
\end{equation*}
This completes the proof of part (c).

\subsection*{Acknowledgments}
AM was supported by the NSF through Award DMS-2031883, the Simons Foundation through
Award 814639 for the Collaboration on the Theoretical Foundations of Deep Learning, and the NSF
Award MFAI-2501597. KZ was supported by the Founder's Fellowship in Statistics at Columbia University.

\newpage

\bibliographystyle{alpha}
\bibliography{small_nn}

\newpage

\appendix

%
%
\section{The replica calculation}\label{append:replica}
In this section, we carry out calculations using the non-rigorous replica
method from statistical physics, to support Conjectures \ref{conj:Parisi_formula_mdim} and \ref{conj:Parisi_formula_1dim}.
We will focus on the case $m = 1$ (\cref{conj:Parisi_formula_1dim}), since the case of general $m$ is almost identical, but less transparent. We
refer to \cite{montanari2024friendly} for a friendly introduction to these techniques.  The derivation presented here is quite straightforward (from the perspective of statistical physics) and generalizes the replica calculation in \cite{gyorgyi2000beyond}.

For the case $m = 1$, we define the Hamiltonian
\begin{equation*}
	H_{n, d} (\ww) = \frac{1}{n} \sum_{i=1}^{n} h \left( y_i, \langle \xx_i, \ww \rangle \right)\, ,\,\,\, \ww \in \S^{d-1}\, .
\end{equation*}
Recall that $y_i = \varphi(\WW_*^\sT \xx_i, \veps_i)$ where $\WW_* \in O(d, k)$. By rotational invariance, we may assume without loss of generality that $\WW_* = [I_k, \bzero_{k \times (d-k)}]^\sT$, and $\xx_i = (g_i, \zz_i^\sT)^\sT$ where $g_i \sim \normal(0, I_k)$, and $\zz_i \sim \normal(\bzero, \id_{d-k})$. Denoting $\ww = (w_1, \ww_2^\sT)^\sT$ where $w_1 \in \R^k$ and $\ww_2 \in \R^{d-k}$, one can write
\begin{equation*}
	y_i = \varphi(g_i, \veps_i), \quad \langle \xx_i, \ww \rangle = \langle w_1, g_i \rangle + \langle \ww_2, \zz_i \rangle.
\end{equation*}
Under this reduction, we have
\begin{equation*}
	H_{n, d} (\ww) = \frac{1}{n} \sum_{i=1}^{n} h \left( y_i, \langle \xx_i, \ww \rangle \right) = \frac{1}{n} \sum_{i=1}^{n} h \left( y_i, \langle w_1, g_i \rangle + \langle \ww_2, \zz_i \rangle \right), \quad y_i = \varphi(g_i, \veps_i).
\end{equation*}
Since $\norm{w_1}_2^2 + \norm{\ww_2}_2^2 = \norm{\ww}_2^2 = 1$, it follows that
\begin{align*}
	\max_{\ww \in \S^{d-1}} H_{n, d} (\ww) = \, & \max_{w_1 \in B_k (1)} \max_{\norm{\ww_2}_2 = \sqrt{1 - \norm{w_1}_2^2}} \frac{1}{n} \sum_{i=1}^{n} h \left( y_i, \langle w_1, g_i \rangle + \langle \ww_2, \zz_i \rangle \right) \\
	\stackrel{(i)}{=} \, & \max_{r \in B_k (1)} \max_{\ww \in \S^{d-k-1}} \frac{1}{n} \sum_{i=1}^{n} h \left( y_i, \langle r, g_i \rangle + \sqrt{1 - \norm{r}_2^2} \ \langle \ww, \zz_i \rangle \right),
\end{align*}
where in $(i)$ we recast $w_1$ as $r$ and $\ww_2$ as $\sqrt{1 - \norm{r}_2^2} \ \ww$. Since $\lim_{n \to \infty} n / (d-k) = \alpha$, it then suffices to compute the asymptotics of (note that we also recast $d-k$ as $d$)
\begin{equation*}
	\max_{\ww \in \S^{d-1}} H_{n, d} (\ww, r) := \max_{\ww \in \S^{d-1}} \frac{1}{n} \sum_{i=1}^{n} h \left( y_i, \langle r, g_i \rangle + \sqrt{1 - \norm{r}_2^2} \ \langle \ww, \zz_i \rangle \right)
\end{equation*}
for each fixed $r \in B_k (1)$.

To this end, we define the partition function for $\beta > 0$:
\begin{equation*}
	Z_{\beta} = \int_{\S^{d-1}} e^{n\beta H_{n, d}(\ww, r)}\,
 \nu_0 (\d \ww),
\end{equation*}
where $\nu_0$ is the uniform measure on $\S^{d-1}$. Assume that
\begin{equation*}
	\lim_{\substack{n, d \to \infty \\ n/d \to \alpha}} \max_{\ww \in \S^{d-1}} H_{n, d} (\ww, r)
\end{equation*}
 exists almost surely, and concentrates around its expectation (this is true if $h$ is Lipschitz), then we would like to compute
\begin{align*}
	\lim_{n \to \infty} \E \left[ \max_{\ww \in \S^{d-1}} H_{n, d} (\ww, r) \right] = \lim_{n \to \infty} \E \left[ \lim_{\beta \to \infty} \frac{1}{n \beta} \log Z_{\beta} \right].
\end{align*}
Here, the limit $n \to \infty$ should be understood as $n, d \to \infty$ simultaneously with $n/d \to \alpha$. 
Within the scope of replica method,
we will interchange expectations and limits arbitrarily. It then suffices to compute the quantity
\begin{equation*}
	\lim_{\beta \to \infty} \lim_{n \to \infty} \frac{1}{n \beta} \E \left[ \log Z_{\beta} \right] = \lim_{\beta \to \infty} \lim_{n \to \infty} \lim_{k \to 0^+} \frac{1}{n \beta k} \log \E \left[ Z_{\beta}^k \right],
\end{equation*}
where we use the identity
\begin{equation*}
	\E \left[ \log Z \right] = \lim_{k \to 0^+} \frac{1}{k} \log \E \left[ Z^k \right]
\end{equation*}
for any non-negative random variable $Z$.

While the above interchange of limits is not justified in the present derivation, it is not the most problematic step in the replica calculation.
Indeed, the critical step is to first consider $k$ as an integer, and then
extrapolate to non-integer values of $k$. For $k \in \mathbb{N}$, we have
\begin{align*}
	\E \left[ Z_{\beta}^k \right] = \, & \E \left[ \int_{(\S^{d-1})^k} \exp \left( \beta \cdot \sum_{j=1}^{k} \sum_{i=1}^{n} h \left( y_i, r^\sT g_i + \sqrt{1 - \norm{r}_2^2} \ \langle \ww_j, \zz_i \rangle \right) \right) \cdot \prod_{j=1}^{k} \nu_0(\d \ww_j) \right] \\
	= \, & \int_{(\S^{d-1})^k} \E \left[ \exp \left( \beta \cdot \sum_{j=1}^{k} \sum_{i=1}^{n} h \left( y_i, r^\sT g_i + \sqrt{1 - \norm{r}_2^2} \ \langle \ww_j, \zz_i \rangle \right) \right) \right] \cdot \prod_{j=1}^{k} \nu_0(\d \ww_j) \\
	= \, & \int_{(\S^{d-1})^k} \E \left[ \exp \left( \beta \cdot \sum_{j=1}^{k} h \left( Y, r^\sT G + \sqrt{1 - \norm{r}_2^2} \ \langle \ww_j, \zz \rangle \right) \right) \right]^n \cdot \prod_{j=1}^{k} \nu_0 (\d \ww_j),
\end{align*}
where the inner expectation is taken under
\begin{equation*}
	(Y, G) \perp\!\!\!\perp \zz, \quad \zz \sim \normal(\bzero, \id_d), \, Y = \varphi(G, \veps), \, G \sim \normal(0, I_k), \, \veps \sim P_{\veps}, \, \veps \perp\!\!\!\perp G.
\end{equation*}
Denoting by $Q$ the overlap matrix of the $\ww_j$'s, namely $Q_{ij} = \langle \ww_i, \ww_j \rangle$ for $1 \le i, j \le k$, then we have for $\zz \sim \normal (\bzero, \id_d)$,
\begin{align*}
	& \E \left[ \exp \left( \beta \cdot \sum_{j=1}^{k} h \left( Y, r^\sT G + \sqrt{1 - \norm{r}_2^2} \ \langle \ww_j, \zz \rangle \right) \right) \right] \\
	= \, & \E \left[ \exp \left( \beta \cdot \sum_{j=1}^{k} h \left( Y, r^\sT G + \sqrt{1 - \norm{r}_2^2} \  Z_j \right) \right) \right],
\end{align*}
where $Z = (Z_j)_{j=1}^k \sim \normal(0, Q)$ is independent of $(Y, G)$.
For future convenience, we denote the above quantity as $f_{\beta, h} (Q)$, i.e.,
\begin{equation*}
	f_{\beta, h} (Q) = \E \left[ \exp \left( \beta \cdot \sum_{j=1}^{k} h \left( Y, r^\sT G + \sqrt{1 - \norm{r}_2^2} \  Z_j \right) \right) \right],
\end{equation*}
it then follows that
\begin{align*}
	\E \left[ Z_{\beta}^k \right] = \, & \int_{(\S^{d-1})^k} f_{\beta, h} (Q)^n \cdot \prod_{j=1}^{k} \nu_0 (\d \ww_j) \\
	= \, & \int_{\Sym_+^k (1)} f_{\beta, h} (Q)^n \exp \left( d I_d (Q) \right) \d Q,
\end{align*}
where $\Sym_+^k (1)$ denotes the space of all $k \times k$ positive semidefinite matrices with all ones on the diagonal, and $\d Q = \prod_{1 \le i < j \le k} \d Q_{ij}$ represents the uniform probability measure on this space. Moreover, we have for fixed $k$,
\begin{equation*}
	\lim_{d \to \infty} I_d (Q) = \frac{1}{2} \log \det Q,
\end{equation*}
thus leading to
\begin{align*}
	& \lim_{n \to \infty} \frac{1}{n} \log \E \left[ Z_{\beta}^k \right] = \, \max_{Q \in \Sym_+^k (1)} \left\{ \log f_{\beta, h} (Q) + \frac{1}{2 \alpha} \log \det Q \right\} \\
	= \, & \max_{Q \in \Sym_+^k (1)} \left\{ \log \E \left[ \exp \left( \beta \cdot \sum_{j=1}^{k} h \left( Y, r^\sT G + \sqrt{1 - \norm{r}_2^2} \  Z_j \right) \right) \right] + \frac{1}{2 \alpha} \log \det Q \right\},
\end{align*}
which we denote as $S_{\beta, h} (\alpha, k)$.
Assume again that we can interchange the limits arbitrarily, then we get that
\begin{align*}
	\lim_{\beta \to \infty} \lim_{n \to \infty} \frac{1}{n \beta} \E \left[ \log Z_{\beta} \right] = \, & \lim_{\beta \to \infty} \lim_{n \to \infty} \lim_{k \to 0^+} \frac{1}{n \beta k} \log \E \left[ Z_{\beta}^k \right] \\
	= \, & \lim_{\beta \to \infty} \frac{1}{\beta} \lim_{k \to 0^+} \frac{1}{k} \lim_{n \to \infty} \frac{1}{n} \log \E \left[ Z_{\beta}^k \right] \\
	= \, & \lim_{\beta \to \infty} \frac{1}{\beta} \lim_{k \to 0^+} \frac{1}{k} S_{\beta, h} (\alpha, k).
\end{align*}
To compute this limit, we resort to the full RSB (full replica symmetry breaking) ansatz described in Section 3 of \cite{gyorgyi2000beyond}. Following their calculation, the limiting free energy can be expressed as the extreme value of a variational problem. To be specific, we have
\begin{align*}
	&\frac{1}{\beta} \lim_{k \to 0^+} \frac{1}{k} S_{\beta, h} (\alpha, k) =
    \inf_{v \in \mathscr{U} [0, 1]} \sA(v,\beta)\, ,\\
    &\sA(v,\beta) := \E_{Y, G} \left[ f_{Y, G, r, v} (0, 0) \right] + \frac{1}{2 \alpha \beta} \int_{0}^{1} \left( \frac{1}{D_v (t)} - \frac{1}{1 - t} \right) \d t ,
\end{align*}
where $\mathscr{U} [0, 1]$ is the space of all non-decreasing functions $v: [0, 1] \to [0, 1]$,
\begin{equation*}
	D_v(t) = \int_{t}^{1} v(s) \d s,
\end{equation*}
and for each fixed pair $(y, g)$, $f_{y, g, r, v} (t, x)$ satisfies the PDE:
\begin{equation}\label{eq:Parisi_beta}
\begin{split}
	& \partial_t f_{y, g, r, v} (t, x) + \frac{1}{2} \beta v(t) \left( \partial_x f_{y, g, r, v} (t, x) \right)^2 + \frac{1}{2} \partial_x^2 f_{y, g, r, v} (t, x) = \, 0, \\
	& f_{y, g, r, v} (1, x) = \, h \left( y, r^\sT g + \sqrt{1 - \norm{r}_2^2} \ x \right).
\end{split}
\end{equation}
The lemma below gives the zero-temperature limit ($\beta \to \infty$) of 
the variational functional $\sA(v,\beta)$, along specific 
sequences of $v_{\beta}$.
\begin{lem}\label{lem:Parisi_final}
	Let $c > 0$, and $\mu(t): [0, 1) \to \R_{\ge 0}$ be a non-decreasing function with $\int_{0}^1\mu(t)\, \d t<\infty$. Further, assume that $v(t) = v_{\beta} (t)$ has the following form:
	\begin{equation*}
		v_{\beta}(t) = \frac{\mu(t)}{\beta} \bone_{t < 1 - \frac{c}{\beta}} + \bone_{t \ge 1 - \frac{c}{\beta}}.
	\end{equation*}
	Then, we have
	\begin{align*}
		& \lim_{\beta \to \infty} 
        \sA(v_{\beta},\beta) = \sF_1(\mu,c,r)\, ,\\
        & \sF_1(\mu,c,r) := \E_{Y, G} \left[ f_{Y,G,r} (0, 0) \right] + \frac{1}{2 \alpha} \int_{0}^{1} \frac{\d t}{c + \int_{t}^{1} \mu(s) \d s},
	\end{align*}
	where for each fixed pair $(y, g)$, $f_{y, g, r}$ solves the terminal-value problem:
	\begin{equation}\label{eq:Parisi_limit}
	\begin{split}
		&\partial_t f_{y, g, r} (t, x) + \frac{1}{2} \mu(t) \left( \partial_x f_{y, g, r} (t, x) \right)^2 + \frac{1}{2} \partial_x^2 f_{y, g, r} (t, x) = \,  0, \\
		&f_{y, g, r} (1, x) = \,  \sup_{u \in \R} \left\{ h \left( y, r^\sT g + \sqrt{1 - \norm{r}_2^2} \ (x+u) \right) - \frac{u^2}{2 c} \right\}.
	\end{split}
	\end{equation}
\end{lem}

\begin{proof}
	In this proof, we will suppress the dependence of $f_{y, g, r, v}$ on $v$ to simplify the notation.
	Defining $t_{\beta} = 1 - c / \beta$, Eq.~\eqref{eq:Parisi_beta} then reduces to
	\begin{align*}
	& \partial_t f_{y, g, r} (t, x) + \frac{1}{2} \mu(t) \left( \partial_x f_{y, g, r} (t, x) \right)^2 + \frac{1}{2} \partial_x^2 f_{y, g, r} (t, x) = \, 0, \ t \in [0, t_{\beta}), \\
	& \partial_t f_{y, g, r} (t, x) + \frac{1}{2} \beta \left( \partial_x f_{y, g, r} (t, x) \right)^2 + \frac{1}{2} \partial_x^2 f_{y, g, r} (t, x) = \, 0, \ t \in [t_{\beta}, 1), \\
	& f_{y, g, r} (1, x) = \, h \left( y, r^\sT g + \sqrt{1 - \norm{r}_2^2} \ x \right).
\end{align*}
Using Cole-Hopf transform, we know that
\begin{align*}
	f_{y, g, r} (t_{\beta}, x) = \, & \frac{1}{\beta} \log \E_{G \sim \normal(0, 1)} \left[ \exp \left( \beta \cdot h \left( y, r^\sT g + \sqrt{1 - \norm{r}_2^2} \ \left( x + \sqrt{\frac{c}{\beta}} G \right) \right) \right) \right] \\
	=\, & \frac{1}{\beta} \log \left( \frac{1}{\sqrt{2 \pi}} \int_{\R} \exp \left( \beta \cdot h \left( y, r^\sT g + \sqrt{1 - \norm{r}_2^2} \ \left( x + \sqrt{\frac{c}{\beta}} z \right) \right) - \frac{z^2}{2} \right) \d z \right) \\
	=\, & \frac{1}{\beta} \log \left( \frac{\sqrt{\beta}}{\sqrt{2 \pi c}} \int_{\R} \exp \left( \beta \cdot h \left( y, r^\sT g + \sqrt{1 - \norm{r}_2^2} \ ( x + u ) \right) - \frac{\beta u^2}{2 c} \right) \d u \right),
\end{align*}
which converges to 
\begin{equation*}
	\sup_{u \in \R} \left\{ h \left( y, r^\sT g + \sqrt{1 - \norm{r}_2^2} \ ( x + u ) \right) - \frac{u^2}{2 c} \right\}
\end{equation*}
as $\beta \to \infty$, uniformly over
compact sets. Moreover, since $t_{\beta} \to 1$, we deduce that the limit of $f_{y, g, r}$ as $\beta \to \infty$ indeed solves Eq.~\eqref{eq:Parisi_limit}. This establishes that the first term in $\sA (v_{\beta}, \beta)$ converges to the first term in $\sF_1 (\mu, c, r)$. To compute the limit of the second term, we note that $D_v (t) = 1 - t$ if $t \ge t_{\beta}$. Therefore,
\begin{align*}
	\frac{1}{2 \alpha \beta} \int_{0}^{1} \left( \frac{1}{D_v (t)} - \frac{1}{1 - t} \right) \d t = \, & \frac{1}{2 \alpha \beta} \int_{0}^{t_{\beta}} \left( \frac{1}{D_v (t)} - \frac{1}{1 - t} \right) \d t \\
	= \, & \frac{1}{2 \alpha} \int_{0}^{t_{\beta}} \frac{\d t}{c + \int_{t}^{t_{\beta}} \mu(s) \d s} + \frac{1}{2 \alpha \beta} \log \left( 1 - t_{\beta} \right) \\
	= \, & \frac{1}{2 \alpha} \int_{0}^{t_{\beta}} \frac{\d t}{c + \int_{t}^{t_{\beta}} \mu(s) \d s} + \frac{1}{2 \alpha \beta} \log \left( \frac{c}{\beta} \right) \\
	\to \, & \frac{1}{2 \alpha} \int_{0}^{1} \frac{\d t}{c + \int_{t}^{1} \mu(s) \d s} \ \text{as} \ \beta \to \infty.
\end{align*}
This completes the proof.
\end{proof}
Defining
\begin{equation*}
	\mathscr{U} = \left\{ \mu: [0, 1) \to \R_{\ge 0}: \ \mu \ \text{non-decreasing}, \ \int_{0}^{1} \mu(t) \d t < \infty \right\},
\end{equation*}
our replica calculation leads to
\begin{equation*}
	\lim_{n \to \infty} \max_{\ww \in \S^{d-1}} H_{n, d} (\ww) = \, \sup_{r \in B_k (1)} \inf_{(\mu, c) \in \mathscr{U} \times \R_{>0}} \mathsf{F}_1 (\mu, c, r),
\end{equation*}
with $\mathsf{F}_1$ defined in \cref{lem:Parisi_final}. To establish the replica prediction in \cref{conj:Parisi_formula_1dim}, it suffices to show that the $\mathsf{F}_1$ defined as per \cref{eq:ParisiFunc_1dim_reformulated} matches the one defined in \cref{lem:Parisi_final}. For fixed $y \in \R$ and $(\mu, c) \in \mathscr{U}$, let $f_{y, \mu}^{(r)}: [\norm{r}_2^2, 1] \times \R \to \R$ be the solution to the PDE:
\begin{equation*}
		\begin{split}
			&\partial_t f_{y, \mu}^{(r)} (t, x) + \frac{1}{2} \mu \left( \frac{t - \norm{r}_2^2}{1 - \norm{r}_2^2} \right) (\partial_x f_{y, \mu}^{(r)} (t, x))^2 +  \frac{1}{2} \partial_x^2 f_{y, \mu}^{(r)} (t, x) = \, 0, \\
			&f_{y, \mu}^{(r)} (1, x) = \,  \sup_{u \in \R} \left\{ h \left( y, x + u \right) - \frac{u^2}{2 (1 - \norm{r}_2^2) c}  \right\}.
		\end{split}
\end{equation*}
Then, straightforward calculation reveals that
\begin{equation*}
\begin{split}
	& f_{y, g, r} (t, x) = f_{y, \mu}^{(r)} \big( \norm{r}_2^2 + (1 - \norm{r}_2^2) t , (1 - \norm{r}_2^2)^{1/2} x + r^\sT g \big) \\
	\implies \, & f_{y, g, r} (0, 0) = f_{y, \mu}^{(r)} ( \norm{r}_2^2, r^\sT g),
\end{split}
\end{equation*}
which leads to
\begin{align*}
	\sF_1(\mu,c,r) := \, & \E_{Y, G} \left[ f_{Y, \mu}^{(r)} ( \norm{r}_2^2, r^\sT G) \right] + \frac{1}{2 \alpha} \int_{0}^{1} \frac{\d t}{c + \int_{t}^{1} \mu(s) \d s} \\
	= \, & \E_{Y, G} \left[ f_{Y, \mu}^{(r)} ( \norm{r}_2^2, r^\sT G) \right] + \frac{1}{2 \alpha} \int_{0}^{1} \frac{ (1 - \norm{r}_2^2) \d t}{ (1 - \norm{r}_2^2) c + \int_{t}^{1} (1 - \norm{r}_2^2) \mu(s) \d s} \\
	= \, & \E_{Y, G} \left[ f_{Y, \mu}^{(r)} ( \norm{r}_2^2, r^\sT G) \right] + \frac{1}{2 \alpha} \int_{\norm{r}_2^2}^{1} \frac{ \d t}{ (1 - \norm{r}_2^2) c + \int_{t}^{1} \mu \left( \frac{s - \norm{r}_2^2}{1 - \norm{r}_2^2} \right) \d s},
\end{align*}
where the last equality follows from the change of time variable $t \mapsto \norm{r}_2^2 + (1 - \norm{r}_2^2) t$. Now, recasting $(1 - \norm{r}_2^2) c$ as $c$ and $\mu \left( \frac{t - \norm{r}_2^2}{1 - \norm{r}_2^2} \right)$ as $\mu (t)$ recovers \cref{eq:ParisiPDE_1dim_reformulated,eq:ParisiFunc_1dim_reformulated} in \cref{conj:Parisi_formula_1dim}. This establishes the replica prediction.

\section{Appendix for \cref{sec:main_results}}

\subsection{Reduction to a random signal}\label{app:Facts}

\begin{lem}\label{lem:Reduction_Signal}
Under the multi-index model specified in \cref{ass:linear_signal}, the distribution of the training and test errors of any machine learning method (of the form $\hat{\WW}_n = \hat{\WW}_n (\XX, \yy)$) is the same as the one for another multi-index model, where $\varphi$ is replaced by a new function $\tilde{\varphi}$,
and $\WW_*$ is replaced by a new signal matrix $\tilde{\WW}_*$ that is uniformly random over $O(d, k)$ (Haar distributed).

In particular, the empirical distribution of the rows of $\sqrt{n} \tilde{\WW}_*$ converges in $W_2$ distance
(and hence weakly) to $\normal(0, \alpha I_k)$ almost surely, as $n,d\to\infty$ with $n / d \to \alpha$.
\end{lem}

\begin{proof}
First of all, let $\tilde{\WW}_* = \WW_* (\WW_*^\sT \WW_*)^{-1/2}$ and $\tilde{\varphi} (\zz, \veps) = \varphi ((\WW_*^\sT \WW_*)^{1/2} \zz, \veps)$. Then, we know that $\tilde{\WW}_* \in O(d, k)$, and
\begin{equation*}
    \tilde{\varphi} \big( \tilde{\WW}_*^\sT \xx_i, \veps_i \big) = \varphi \big( \WW_*^\sT \xx_i, \veps_i \big).
\end{equation*}
Hence, one can assume without loss of generality that $\WW_* \in O(d, k)$. 
Now for any fixed $\WW_* \in O(d, k)$, let $\RR$ be a uniformly random orthogonal matrix of size $d \times d$, namely $\RR$ is sampled from the Haar measure on $O(d, d)$. Then, we know that $\XX \WW_* = \XX \RR^\sT \RR \WW_*$, which leads to
\begin{equation*}
	\yy = \varphi (\XX \WW_*, \vveps) = \varphi (\XX \RR^\sT \RR \WW_*, \vveps).
\end{equation*}
Further, $\tilde{\WW}_* = \RR \WW_*$ is Haar distributed over $O(d, k)$.
Using rotational invariance of standard Gaussian measure and the Haar measure on $O(d, k)$, we know that the empirical distribution of the rows of $\sqrt{n} \tilde{\WW}_*$ weakly converges to $\normal(0, \alpha I_k)$ almost surely, and that $\tilde{\XX} = \XX \RR^\sT \stackrel{d}{=} \XX$ is independent of $\tilde{\WW}_*$. As a consequence, using the estimator $\hat{\WW}_n (\tilde{\XX} \RR, \yy)$ would yield the same training and test errors as in the original model. This completes the proof.
\end{proof}

\subsection{AMP achievability results for the case $m=k=1$}
\label{app:Special_m=1}

In this appendix, we specialize our general
AMP achievability result (Theorem \ref{thm:AMP_Feasibility_general}) 
to the case $m = k = 1$. 
The following definitions are simplified versions of Definitions \ref{defn:Bayes_AMP_region_general} and \ref{defn:mu_Q_contraction}, respectively.

\begin{defn}[Bayes AMP feasible region]\label{defn:Bayes_AMP_region}
	Let the random vector $(Y, G, Z) \in \R^3$ be such that
	\begin{equation}\label{eq:Triple_distribution}
		\begin{array}{l}
			(Y, G) \perp\!\!\!\perp Z, \, Z \sim \normal(0, 1), \, G \sim \normal(0, 1), \\
			Y = \varphi(G, \veps), \, \veps \sim P_{\veps}, \, \veps \perp\!\!\!\perp G.	
		\end{array}
	\end{equation}
	Assume that $\E \left[ G \vert Y \right] \neq 0$, and
    define $c_{\sBayes}\in\R_{\ge 0}$  via
	\begin{align}
		c_{\sBayes}&:= \min
		\Big\{ c\ge 0:\; c^2=  \alpha
		 \E \Big[ \E \big[ G - c Z \big\vert cG + Z, Y \big]^2  \Big] \Big\}. \label{eq:CBayesDef_m=1}
		\end{align}
	We then define the \emph{Bayes AMP feasible region} as
	\begin{equation*}
		A_{\sBayes} = \left\{ (r, q) \in \R^2: q \in [r^2, 1], \, r^2 \le c_{\sBayes}^2 (q - r^2) \right\}.
	\end{equation*}
\end{defn}

\begin{rem}
As in the case of general $m$ and $k$, $c_{\sBayes}$ can be equivalently defined as the limit
of a certain one-dimensional recursion. Namely, we define the sequence $\{ c_t \}_{t=0}^{\infty}$ via
	\begin{equation}\label{eq:Bayes_recursion}
		c_0 = 0, \quad c_{t+1} = \alpha^{1/2} \E 
		\left[ \E \left[ G - c_t Z \vert c_t G + Z, Y \right]^2 \right]^{1/2}, \quad t \ge 0.
	\end{equation}
	Then, we have similarly $c_{\sBayes} = \lim_{t \to \infty} c_t$ 
	(cf. \cref{lem:property_recursion}).  
\end{rem}

\begin{defn}\label{defn:mu_q_contraction}
	Let $(Y, G, Z)$ be as described in \cref{eq:Triple_distribution}, and define $Z_{r, q} = r G + \sqrt{q - r^2} Z$. We say that $F: \R^2 \to \R$ is an $(r, q)$-contraction, if
	\begin{align*}
		r = \, \E \left[ \frac{\partial F}{\partial G} \left( Z_{r, q}, \varphi(G, \veps) \right) \right], \quad q = \, r^2 + \frac{1}{\alpha} \E \left[ F \left( Z_{r, q}, \varphi(G, \veps) \right)^2 \right],
	\end{align*}
	and
	\begin{equation*}
		\frac{1}{\alpha} \E \left[ \frac{\partial F}{\partial Z_{r, q}} \left( Z_{r, q}, \varphi(G, \veps) \right)^2 \right] \le 1.
	\end{equation*}
\end{defn}

\begin{thm}
\label{thm:AMP_Feasibility_Super}
For any $(r, q) \in A_{\sBayes}$, let $(Y, Z_{r, q})$ be as described in \cref{defn:mu_q_contraction}, and $F$ be an $(r, q)$-contraction. Let $(B_t)_{0 \le t \le 1}$ be a standard Brownian motion independent of $(Y, Z_{r, q})$. Define the filtration $\{ \cF_t \}_{0 \le t \le 1}$ by
\begin{equation*}
    \cF_t = \sigma \left( Z_{r, q}, Y, (B_s)_{0 \le s \le t} \right), \ 0 \le t \le 1.
\end{equation*}
Assume $q(t) \in L^2 [0, 1]$ satisfies $\int_{0}^{1} q(t)^2 \d t = 1 - q$, and $\{ \phi_t \}_{0 \le t \le 1}$ is a progressively measurable stochastic process with respect to the filtration $\{ \cF_t \}_{0 \le t \le 1}$, satisfying
\begin{equation*}
    \E \left[ \phi_t^2 \right] \le \frac{1}{\alpha}, \ \forall 0 \le t \le 1.
\end{equation*}
Then, there exists a two-stage AMP algorithm (described in \cref{sec:IAMP} and further detailed in \cref{sec:iamp}) that outputs $\hat{\ww}_n^{\sAMP} \in \S^{d-1}$, satisfying:
    \begin{itemize}
        \item [(a)] $\lim_{n \to \infty} \ww_*^\sT \hat{\ww}_n^{\sAMP} = r$ almost surely.
        \item [(b)] Define
            \begin{equation*}
                U = \, Z_{r, q} + \frac{1}{\alpha} F \left( Z_{r, q}, Y \right) + \int_{0}^{1} q(t) ( 1 + \phi_t ) \d B_t.
            \end{equation*}
            As $n / d \to \alpha$, the following holds almost surely for any $h \in C_b (\R^2)$:
            \begin{equation*}
            \begin{split}
                \frac{1}{n} \sum_{i=1}^{n} \delta_{( y_i, \, (\hat{\ww}_n^{\sAMP})^\sT \xx_i ) } \stackrel{w}{\to} \,  \operatorname{Law} ( Y, U ), \quad
                \frac{1}{n} \sum_{i=1}^{n} h ( y_i, \, ( \hat{\ww}_n^{\sAMP} )^\sT \xx_i ) \to \,  \E [ h ( Y, U ) ].
            \end{split}
            \end{equation*}
            Consequently, we have $\operatorname{Law} (Y, U) \in \cuF_{1, \alpha, \varphi}^{\salg}$.
    \end{itemize}
\end{thm}

\subsection{Auxiliary lemmas}
\begin{lem}\label{lem:property_recursion_general}
	Let $(Y, G, Z)$ be as defined in \cref{defn:Bayes_AMP_region_general}. For $C \in \R^{k \times m}$, define
	\begin{equation*}
		g_{\sBayes} (C) = \E \left[ \E \left[ G - C Z \big\vert C^\sT G + Z, Y \right] \E \left[ G - C Z \big\vert C^\sT G + Z, Y \right]^\sT \right].
	\end{equation*}
	Then, $g_{\sBayes} (C)$ only depends on $C C^\sT$, and is increasing in $C C^\sT$ in the following sense:
	\begin{equation*}
		C_1 C_1^\sT \preceq C_2 C_2^\sT \implies g_{\sBayes} (C_1) \preceq g_{\sBayes} (C_2).
	\end{equation*}
	As a consequence, $\Gamma_* = \lim_{t \to \infty} \Gamma_t$ exists for the sequence $\{ \Gamma_t \}_{t = 0}^{\infty}$ defined in \cref{eq:Bayes_recursion_general}.
\end{lem}
\begin{proof} 
	It suffices to show that for any $\beta \in \R^k$, the function
	\begin{align*}
		g_{\sBayes} (C, \beta) = \, & \beta^\sT \E \left[ \E \left[ G - C Z \big\vert C^\sT G + Z, Y \right] \E \left[ G - C Z \big\vert C^\sT G + Z, Y \right]^\sT \right] \beta \\
		= \, & \E \left[ \E \left[ \beta^\sT (G - C Z) \big\vert C^\sT G + Z, Y \right]^2 \right]
	\end{align*}
	only depends on $C C^\sT$, and is non-decreasing in $C C^\sT$. To this end, we first establish an equivalent definition for this mapping. According to Cauchy-Schwarz inequality, we have
	\begin{equation*}
		g_{\sBayes} (C, \beta) = \, \max_{\E [f^2] \le 1} \E \left[ \beta^\sT (G - C Z) \cdot f ( C^\sT G + Z, Y ) \right]^2,
	\end{equation*}
	where $f$ can depend on some other randomness as well. Denote $Z_C = C^\sT G + Z$ and recall $Y = \varphi(G, \veps)$. We know that $(Z_C^\sT, G^\sT)^\sT$ is a $(k+m)$-dimensional Gaussian random vector with mean $0$ and covariance matrix equal to
	\begin{equation*}
		\Sigma_C = \begin{bmatrix}
			I_m + C^\sT C & C^\sT \\
			C & I_k
		\end{bmatrix}.
	\end{equation*}
	Stein's lemma then implies that
	\begin{align*}
		& \E \left[ (Z_C^\sT, G^\sT)^\sT f(C^\sT G + Z, Y) \right] = \, \E \left[ (Z_C^\sT, G^\sT)^\sT f(Z_C, \varphi(G, \veps)) \right] \\
		= \, & \Sigma_C \E \left[ \left( \nabla_{Z_C} f(Z_C, \varphi(G, \veps))^\sT, \nabla_{G} f(Z_C, \varphi(G, \veps))^\sT \right)^\sT \right].
	\end{align*}
	Therefore,
	\begin{align*}
		& \E \left[ (G - C Z) \cdot f ( C^\sT G + Z, Y ) \right] = \, \begin{bmatrix}
				-C & I_k + C C^\sT
			\end{bmatrix}
			\E \left[ (Z_C^\sT, G^\sT)^\sT f(C^\sT G + Z, Y) \right] \\
		= \, & \begin{bmatrix}
				-C & I_k + C C^\sT
			\end{bmatrix}
			\Sigma_C \E \left[ \left( \nabla_{Z_C} f(Z_C, \varphi(G, \veps))^\sT, \nabla_{G} f(Z_C, \varphi(G, \veps))^\sT \right)^\sT \right] \\
		= \, & \begin{bmatrix}
				0 & I_k
			\end{bmatrix}
			\E \left[ \left( \nabla_{Z_C} f(Z_C, \varphi(G, \veps))^\sT, \nabla_{G} f(Z_C, \varphi(G, \veps))^\sT \right)^\sT \right] = \E \left[ \nabla_{G} f(Z_C, \varphi(G, \veps)) \right].
	\end{align*}
	As a consequence,
	\begin{equation*}
		g_{\sBayes} (C, \beta) = \, \max_{\E [f^2] \le 1} \E \left[ \beta^\sT \nabla_{G} f(Z_C, \varphi(G, \veps)) \right]^2.
	\end{equation*}
	Based on this definition, we can easily show that $g_{\sBayes} (C, \beta)$ only depends on $C C^\sT$. In fact, let $C_1$ and $C_2$ be two $k \times m$ matrices such that $C_1 C_1^\sT = C_2 C_2^\sT$, then there exists an orthogonal matrix $O \in \R^{m \times m}$ such that $C_2 = C_1 O^\sT$. Hence,
	\begin{align*}
		g_{\sBayes} (C_2, \beta) = \, & \max_{\E [f^2] \le 1} \E \left[ \beta^\sT \nabla_{G} f(O C_1^\sT G + Z, \varphi(G, \veps)) \right]^2 \\
		\stackrel{(i)}{=} \, & \max_{\E [f_O^2] \le 1} \E \left[ \beta^\sT \nabla_{G} f_O ( C_1^\sT G + O^\sT Z, \varphi(G, \veps)) \right]^2 \\
		\stackrel{(ii)}{=} \, & \max_{\E [f^2] \le 1} \E \left[ \beta^\sT \nabla_{G} f ( C_1^\sT G + Z, \varphi(G, \veps)) \right]^2 = g_{\sBayes} (C_1, \beta),
	\end{align*}
	where in $(i)$ we define $f_O (x, y) = f(Ox, y )$, and $(ii)$ follows from the fact that $(C_1^\sT G + O^\sT Z, G) \stackrel{d}{=} (C_1^\sT G + Z, G)$. To establish the monotonicity of $g_{\sBayes} (\cdot, \beta)$, note that it can be rewritten as
	\begin{align*}
		g_{\sBayes} (C, \beta) = \, & \max_{\E [f^2] \le 1} \E \left[ \beta^\sT \nabla_{G} f(C^\sT G + Z, \varphi(G, \veps)) \right]^2 \\
		= \, & \max_{\E [f^2] \le 1} \E \left[ \beta^\sT \nabla_{G} f(G + (C C^\sT)^{-1} C Z, \varphi(G, \veps)) \right]^2.
	\end{align*}
	Consider $C_1, C_2 \in \R^{k \times m}$ satisfying $C_1 C_1^\sT \preceq C_2 C_2^\sT$. Denote $U_i = (C_i C_i^\sT)^{-1} C_i$ for $i = 1, 2$, then we have $U_2 U_2^\sT \preceq U_1 U_1^\sT$. Therefore, we can write
	\begin{equation*}
		U_1 Z = U_2 Z + \veps_U, \quad \veps_U \indep (G, Z, \veps).
	\end{equation*}
	This decomposition further implies that
	\begin{align*}
		g_{\sBayes} (C_1, \beta) = \, & \max_{\E [f^2] \le 1} \E \left[ \beta^\sT \nabla_{G} f(G + U_1 Z, \varphi(G, \veps)) \right]^2 \\
		= \, & \max_{\E [f^2] \le 1} \E \left[ \beta^\sT \nabla_{G} f(G + U_2 Z + \veps_U, \varphi(G, \veps)) \right]^2 \\
		\le \, & \max_{\E [f^2] \le 1} \E \left[ \beta^\sT \nabla_{G} f(G + U_2 Z, \varphi(G, \veps), \veps_U) \right]^2 \\
		= \, & \max_{\E [f^2] \le 1} \E \left[ \beta^\sT \nabla_{G} f(G + U_2 Z, \varphi(G, \veps)) \right]^2 = g_{\sBayes} (C_2, \beta),
	\end{align*}
	which proves the monotonicity of $g_{\sBayes} (\cdot, \beta)$. Finally, let the sequence $\{ \C_t \}_{t=0}^{\infty}$ be defined as in \cref{eq:Bayes_recursion_general}; then
	\begin{equation*}
		\C_1 = \alpha g_{\sBayes} (0), \quad \C_{t+1} = \alpha g_{\sBayes} \big( \C_t^{1/2} \big).
	\end{equation*}
	Using induction, one can easily show that $\{ \C_t \}_{t = 0}^{\infty}$ is a non-decreasing sequence. Below, we consider two cases:
	\begin{itemize}
		\item [(a)] $\cuC_{\sBayes} = \emptyset$, i.e., there does not exist $\C \in \Sym_+^{k}$ such that $\alpha g_{\sBayes} (\C^{1/2}) = \C$. In this case, $\C_t$ will diverge to infinity. Otherwise, if it converges to some finite $\C_*$, $\C_*$ must satisfy $\alpha g_{\sBayes} (\C_*^{1/2}) = \C_*$, a contradiction.
	  	\item [(b)] $\cuC_{\sBayes} \neq \emptyset$, i.e., there exists some $\C \in \Sym_+^{k}$ such that $\alpha g_{\sBayes} (\C^{1/2}) = \C$. In this case, $\C_t$ will converge to one such $\C$. Denote this limit by $\C_*$. Then $\C_*$ is the smallest (with respect to the Loewner order) of all $\C \in \Sym_+^{k}$ such that $\alpha g_{\sBayes} (\C^{1/2}) = \C$, since it is the limit of a fixed-point iteration starting from $0$.
	\end{itemize}
	Combining these two cases completes the proof of \cref{lem:property_recursion_general}.
\end{proof}

\begin{lem}\label{lem:property_recursion}
	For the sequence $\{ c_t \}_{t=0}^{\infty}$ defined in \cref{eq:Bayes_recursion}, we have $c_{\sBayes} = \lim_{t \to \infty} c_t$.
\end{lem}
\begin{proof}
	Similar to the case of general $k$ and $m$, let us define for $x \ge 0$:
	\begin{equation*}
		g_{\sBayes} (x) = \E 
		\left[ \E \left[ G - x Z \vert x G + Z, Y \right]^2 \right].
	\end{equation*}
	Then, the recursion \eqref{eq:Bayes_recursion} becomes $c_{t+1} = \alpha^{1/2} g_{\sBayes} (c_t)^{1/2}$. Specializing the conclusions of \cref{lem:property_recursion_general} to the setting $m = k = 1$, we know that $g_{\sBayes}$ is monotone increasing on $[0, \infty)$. Furthermore, since $c_1 = \alpha^{1/2}g_{\sBayes} (0)^{1/2}$, it follows that the sequence $\{ c_t \}_{t=1}^{\infty}$ is increasing. Similarly, we consider the following two situations:
	\begin{itemize}
		\item [(a)] The mapping $c \mapsto \alpha^{1/2} g_{\sBayes} (c)^{1/2}$ does not have any fixed point. In this case, we have $c_{\sBayes} = \infty$, and $c_t \to \infty$ as $t \to \infty$.
		\item [(b)] The mapping $c \mapsto \alpha^{1/2} g_{\sBayes} (c)^{1/2}$ has a fixed point. In this case, $c_{\sBayes} = \lim_{t \to \infty} c_t$ exists and is the first fixed point of $g_{\sBayes}$.
	\end{itemize} 
	This completes the proof of \cref{lem:property_recursion}.
\end{proof}

\section{Appendix for \cref{sec:heuristic}}\label{append:heuristic}

%
%

\subsection{Proof of \cref{prop:ERM_asymptotics} and \ref{prop:rs_amp_achievability}} 
\begin{proof}[Proof of \cref{prop:ERM_asymptotics}]
    We first prove $(i)$. To this end, we invoke \cref{prop:RS_condition} and the condition~\eqref{eq:condition_RS} for replica symmetry. It suffices to verify \eqref{eq:condition_RS} for some small $c_*$ when $\alpha$ is large. Note that for small enough $c_*$, the condition $h_{c_*} = h$ is automatically satisfied by definition of $h_{c_*}$. It then suffices to show that there exists $c_* = c_* (\alpha) \to 0$ as $\alpha \to \infty$, such that
    \begin{equation*}
        \E \left[ \left( \partial_x f_{Y, c_*} (1, X_1) \right)^2 \right] = \frac{1}{\alpha c_*^2} (1 - r^2), \quad \E \left[ \left( \partial_x^2 f_{Y, c_*} (1, X_1) \right)^2 \right] \le \frac{1}{\alpha c_*^2}.
    \end{equation*}
    By definition and Lipschitzness of $h$, we know that
    \begin{equation*}
        f_{y, c} (1, x) = \, \sup_{u \in \R} \left\{ h \left( y, x + u \right) - \frac{u^2}{2c}  \right\} = \sup_{u \in \R} \left\{ h \left( y, x + \sqrt{c} u \right) - \frac{u^2}{2}  \right\}
    \end{equation*}
    converges to $h(y, x)$ as $c \to 0$. Further, we can show that $\partial_x f_{y, c}(1, x) \to \partial_x h(y, x)$ and $\partial_x^2 f_{y, c}(1, x) \to \partial_x^2 h(y, x)$. Therefore, as long as $c_* (\alpha) \to 0$ and $1 - r^2 \le C / \alpha$, we have
    \begin{align*}
        & \lim_{\alpha \to \infty} \E \left[ \left( \partial_x f_{Y, c_*} (1, X_1) \right)^2 \right] = \, \E \left[ (\partial_x h(Y, G))^2 \right], \\
        & \lim_{\alpha \to \infty} \E \left[ \left( \partial_x^2 f_{Y, c_*} (1, X_1) \right)^2 \right] = \, \E \left[ (\partial_x^2 h(Y, G))^2 \right].
    \end{align*}
    This means that we can choose 
    \begin{equation*}
        c_* (\alpha) \propto \sqrt{\frac{1 - r^2}{\alpha \E \left[ (\partial_x h(Y, G))^2 \right]}}
    \end{equation*}
    to satisfy \cref{eq:condition_RS}. This completes the proof of $(i)$.

    We next prove $(ii)$. Note that for any $C > 0$:
    \begin{align*}
        & \sup_{r^2 < 1 - C / \alpha} \inf_{(\mu, c) \in \mathscr{U} \times \R_{> 0}} \mathsf{F} (\mu, c, r) \le \, \sup_{r^2 < 1 - C / \alpha} \inf_{c > 0} \mathsf{F} (0, c, r) \\
        = \, & \sup_{r^2 < 1 - C / \alpha} \inf_{c > 0} \left\{ \E_{Y, G} \left[ f_{Y, c} \left( 1, r G + \sqrt{1 - r^2} Z \right) \right] + \frac{1}{2 \alpha c} (1 - r^2) \right\} \\
        = \, & \sup_{r^2 < 1 - C / \alpha} \sup_{\E [U^2] \le 1 / \alpha} \E \left[ h \left( Y, r G + \sqrt{1 - r^2} (Z + U) \right) \right],
    \end{align*}
    where the last equality follows from the definition of $f_{y, c}$ and Lagrange duality. Define
    \begin{equation*}
        v_{\alpha} (r) = \, \sup_{\E [U^2] \le 1 / \alpha} \E \left[ h \left( Y, r G + \sqrt{1 - r^2} (Z + U) \right) \right].
    \end{equation*}
    Since $h$ is Lipschitz, we obtain that
    \begin{equation*}
        \left\vert v_{\alpha} (r) - v (r) \right\vert \le L \sqrt{\frac{1 - r^2}{\alpha}}.
    \end{equation*}
    Without loss of generality, we assume that $r_1 = 1$. By our assumptions on $v$, it follows that
    \begin{align*}
        v_{\alpha} (r) - v_{\alpha} (1) \le \, & v(r) - v(1) + L \sqrt{\frac{1 - r^2}{\alpha}} \le - \min \left( \veps_0, c_0 (1 - r) \right) + L \sqrt{\frac{1 - r^2}{\alpha}} \\
        \le \, & - c_0 (1 - r) + L \sqrt{\frac{2(1 - r)}{\alpha}} = - \sqrt{1 - r} \left( c_0 \sqrt{1 - r} - L \sqrt{\frac{2}{\alpha}} \right).
    \end{align*}
    Note that $1 - r^2 > C / \alpha$ implies $1 - r > C/ 2 \alpha$. Therefore, for all such $r$, we have
    \begin{align*}
        v_{\alpha} (r) - v_{\alpha} (1) \le \, - \sqrt{1 - r} \left( c_0 \sqrt{\frac{C}{2 \alpha}} - L \sqrt{\frac{2}{\alpha}} \right) < 0
    \end{align*}
    for $C > 4 L^2 / c_0^2$. This implies that
    \begin{align*}
        \sup_{r^2 < 1 - C / \alpha} \inf_{(\mu, c) \in \mathscr{U} \times \R_{> 0}} \mathsf{F} (\mu, c, r) < \, \E \left[ h (Y, G) \right] \stackrel{(*)}{=} \inf_{(\mu, c) \in \mathscr{U} \times \R_{> 0}} \mathsf{F} (\mu, c, 1),
    \end{align*}
    where $(*)$ follows from our conclusion in part $(i)$. Hence, $r_*$ must satisfy $1 - r_*^2 \le C / \alpha$, completing the proof of part $(ii)$. \cref{eq:r_*_RS} then follows naturally.
\end{proof}

\begin{proof}[Proof of \cref{prop:rs_amp_achievability}]
    We first show that $(r_*, 1) \in A_{\sBayes}$ for sufficiently large $\alpha$. 
    By \cref{lem:err_bayes_erm}, we know that $r_* \le r_{\sBayes}$ for sufficiently large $\alpha$, which implies $(r_*, 1) \in A_{\sBayes}$. Therefore, the correlation $r_*$ with the true signal $\ww_*$ is achievable by our two-stage AMP algorithm, and so is the limiting test error~\eqref{eq:RS-asymptotics-test}. Invoking \cref{thm:AMP_Feasibility_Simple} with $r = r_*$, $q = 1$, and $h = - \ell$, we know that the minimum limiting training error achieved by our two-stage AMP algorithm is given by
    \begin{equation*}
        \inf \; \E \left[ \ell \left( Y, \, Z_{r_*, 1} + \frac{1}{\alpha} F \left( Z_{r_*, 1}, Y \right) \right) \right], \quad \mbox{subject to} \,\, \text{$F$ is an $(r_*, 1)$-contraction},
    \end{equation*}
    where $Z_{r_*, 1} = r_* G + \sqrt{1 - r_*^2} Z$. Further, the limiting training error of ERM given in \cref{eq:RS-asymptotics-train} can be rewritten as
    \begin{align*}
        & \inf_{\E [U^2] \le 1 / \alpha} \E \left[ \ell \left( Y, r_* G + \sqrt{1 - r_*^2} (Z + U) \right) \right] \\
        = \, & \inf_{\E [F(Z_{r_*, 1}, Y)^2] \le \alpha (1 - r_*^2)} \E \left[ \ell \left( Y, Z_{r_*, 1} + \frac{1}{\alpha} F \left( Z_{r_*, 1}, Y \right) \right) \right].
    \end{align*}
    Recall $c_*$ from the proof of \cref{prop:ERM_asymptotics}. According to that proof, we know that $(r_*, c_*)$ satisfies \cref{eq:condition_RS} with $r = r_*$. Further, by definition of $f_{y, c} (t, x)$ and Lagrange duality (cf. Proof of \cref{prop:ERM_asymptotics}), we know that the above infimum is achieved at
    \begin{equation*}
        F \left( Z_{r_*, 1}, Y \right) = \, \alpha c_* \partial_x f_{Y, c_*} \big( 1, Z_{r_*, 1} \big).
    \end{equation*}
    It then suffices to verify that $F$ is indeed an $(r_*, 1)$-contraction. To this end, we note that the inequality constraint
    \begin{equation*}
        \frac{1}{\alpha} \E \left[ \frac{\partial F}{\partial Z_{r_*, 1}} \left( Z_{r_*, 1}, \varphi (G, \veps) \right)^2 \right] \le 1
    \end{equation*}
    is already verified in the proof of \cref{prop:ERM_asymptotics}, as part of condition~\eqref{eq:condition_RS} for replica symmetry. Further, the equality constraint
    \begin{equation*}
        r_* = \E \left[ \frac{\partial F}{\partial G} \left( Z_{r_*, 1}, \varphi (G, \veps) \right) \right]
    \end{equation*}
    can be verified by using the envelope theorem to compute the first-order condition for \cref{eq:r_*_RS} with respect to $r_*$, in a way similar to the proof of \cref{thm:two_stage_strong_duality} (c). This completes the proof of \cref{prop:rs_amp_achievability}.
\end{proof}


\subsection{Estimation errors of ERM and Bayes AMP}
We need the following lemma regarding the asymptotic estimation errors of ERM and Bayes optimal AMP as $\alpha \to \infty$:
\begin{lem}\label{lem:err_bayes_erm}
    Let $\ell = - h$. Recall $r_*$ from \cref{eq:r_*_RS} and let $r_{\sBayes}$ be the largest $r \in [-1, 1]$ such that $(r, 1) \in A_{\sBayes}$. Then, as $\alpha \to \infty$, we have
    \begin{align}
        \mathsf{err}_{\rm ERM} := \, & \sqrt{2 (1 - r_*)} \sim \sqrt{\frac{\E [ ( \partial_x \ell (Y, G) )^2 ]}{\alpha \E [ \partial_{y x} \ell (Y, G) \cdot \frac{\partial \varphi (G, \veps)}{\partial G} ]^2}}, \\
        \mathsf{err}_{\sBayes} := \, & \sqrt{2 (1 - r_{\sBayes})} \sim \min_{f: \R^2 \to \R} \sqrt{\frac{\E [ f (Y, G) ^2 ]}{\alpha \E [ \partial_{y} f (Y, G) \cdot \frac{\partial \varphi (G, \veps)}{\partial G} ]^2}}.
    \end{align}
\end{lem}
\begin{proof}[Proof of \cref{lem:err_bayes_erm}]
We first compute the asymptotics of $\mathsf{err}_{\rm ERM}$ as $\alpha \to \infty$. By \cref{prop:ERM_asymptotics} and \cref{eq:r_*_RS}, we have
\begin{align*}
    r_* = \, & \arg \max_{1 - r^2 \le C / \alpha} \inf_{c > 0} \left\{ \E_{Y, G} \left[ f_{Y, c} \left( 1, r G + \sqrt{1 - r^2} Z \right) \right] + \frac{1}{2 \alpha c} (1 - r^2) \right\} \\
    = \, & \arg \max_{1 - r^2 \le C / \alpha} \sup_{\E [U^2] \le 1 / \alpha} \E \left[ h \left( Y, r G + \sqrt{1 - r^2} (Z + U) \right) \right].
\end{align*}
Without loss of generality, we assume $r_1 = 1$, and perform local Taylor expansions around $1$:
\begin{align*}
    & \E \left[ h \left( Y, r G + \sqrt{1 - r^2} (Z + U) \right) \right] = \, \E \left[ h \left( Y, G + (r - 1) G + \sqrt{1 - r^2} (Z + U) \right) \right] \\
    = \, & \E \left[ h ( Y, G ) + \partial_x h(Y, G) \left( (r - 1) G + \sqrt{1 - r^2} (Z + U) \right) + \frac{1}{2} \partial_x^2 h(Y, G) (1 - r^2) (Z + U)^2 \right] + o \left( \frac{1}{\alpha} \right) \\
    \stackrel{(i)}{=} \, & \E \left[ h(Y, G) \right] - \frac{t}{\alpha} \E \left[ G \partial_x h(Y, G) \right] + \sqrt{\frac{2 t}{\alpha}} \E \left[ U \partial_x h(Y, G) \right] + \frac{t}{\alpha} \E \left[ \partial_x^2 h(Y, G) \right] + o \left( \frac{1}{\alpha} \right),
\end{align*}
where in $(i)$ we set $r = 1 - t / \alpha$, and use the fact that $Z$ is independent of $(Y, G)$. Maximizing the above expression over $U$ subject to $\E [U^2] \le 1 / \alpha$, we obtain that
\begin{align*}
    & \sup_{\E [U^2] \le 1 / \alpha} \E \left[ h \left( Y, r G + \sqrt{1 - r^2} (Z + U) \right) \right] \\
    = \, & \E \left[ h(Y, G) \right] - \frac{t}{\alpha} \E \left[ G \partial_x h(Y, G) \right] + \frac{\sqrt{2 t}}{\alpha} \E \left[ \left( \partial_x h(Y, G) \right)^2 \right]^{1/2} + \frac{t}{\alpha} \E \left[ \partial_x^2 h(Y, G) \right] + o \left( \frac{1}{\alpha} \right) \\
    = \, & \E \left[ h(Y, G) \right] + \frac{1}{\alpha} \left( - t \E \left[ \partial_{y x} h (Y, G) \cdot \frac{\partial \varphi (G, \veps)}{\partial G} \right] + \sqrt{2 t} \E \left[ \left( \partial_x h(Y, G) \right)^2 \right]^{1/2} \right) + o \left( \frac{1}{\alpha} \right).
\end{align*}
To ensure that the maximum over $t$ exists, we need to assume
\begin{equation*}
    \E \left[ \partial_{y x} h (Y, G) \cdot \frac{\partial \varphi (G, \veps)}{\partial G} \right] > \, 0.
\end{equation*}
Under this assumption, the above expression is maximized at
\begin{equation*}
    t_{\rm ERM}^* \approx \, \frac{\E [ ( \partial_x h(Y, G) )^2 ]}{2 \E [ \partial_{y x} h (Y, G) \cdot \frac{\partial \varphi (G, \veps)}{\partial G} ]^2} = \frac{\E [ ( \partial_x \ell(Y, G) )^2 ]}{2 \E [ \partial_{y x} \ell (Y, G) \cdot \frac{\partial \varphi (G, \veps)}{\partial G} ]^2}
\end{equation*}
since $h = - \ell$. This finally gives the asymptotic estimation error for ERM:
\begin{equation*}
    \mathsf{err}_{\rm ERM} \sim \sqrt{\frac{2 t_{\rm ERM}^*}{\alpha}} = \sqrt{\frac{\E [ ( \partial_x \ell (Y, G) )^2 ]}{\alpha \E [ \partial_{y x} \ell (Y, G) \cdot \frac{\partial \varphi (G, \veps)}{\partial G} ]^2}}.
\end{equation*}

For Bayes AMP, the calculation is rather straightforward. Using the AMP achievability results for $m = k = 1$ developed in \cref{app:Special_m=1} and setting $q = 1$, we can show that
\begin{equation*}
    \lim_{\alpha \to \infty} 2 \alpha \left( 1 - r_{\sBayes} \right) = \, \lim_{\alpha \to \infty} \alpha \left( 1 - r_{\sBayes}^2 \right) = \frac{1}{g_{\sBayes} (+\infty)},
\end{equation*}
where
\begin{equation*}
    g_{\sBayes} (x) = \, \E 
		\left[ \E \left[ G - x Z \vert x G + Z, Y \right]^2 \right].
\end{equation*}
According to the proof of \cref{lem:property_recursion_general}, we have
\begin{align*}
    g_{\sBayes} (x) = \, & \max_{\E [f^2] \le 1} \E \left[ (G - x Z) \cdot f ( Y, x G + Z ) \right]^2 \\
    = \, & \max_{\E [f^2] \le 1} \E \left[ \partial_y f(Y, x G + Z) \cdot \frac{\partial \varphi (G, \veps)}{\partial G} \right]^2.
\end{align*}
Therefore,
\begin{equation*}
    \lim_{x \to + \infty} g_{\sBayes} (x) = \, \max_{\E [f^2] \le 1} \E \left[ \partial_y f( Y, G ) \cdot \frac{\partial \varphi (G, \veps)}{\partial G} \right]^2 = \max_{f: \R^2 \to \R} \frac{\E \left[ \partial_y f( Y, G) \cdot \frac{\partial \varphi (G, \veps)}{\partial G} \right]^2}{\E [f(Y, G)^2]},
    \end{equation*}
which leads to the claimed asymptotics for $\mathsf{err}_{\sBayes}$. This completes the proof.
\end{proof}

\subsection{Condition for replica symmetry}
We present a necessary and sufficient condition for determining whether a Parisi variational problem has a replica symmetric solution. To this end, we consider the Parisi functional $\mathsf{F} (\mu, c, r)$ defined in \cref{conj:Parisi_formula_1dim}, specialized to the case $m = k = 1$.
\begin{lem}\label{prop:RS_condition}
    For any fixed $r \in [-1, 1]$ and $c > 0$, let $X_t = r G + B_t - B_{r^2}$ for $t \in [r^2, 1]$, and $f_{y, c}$ solves the heat equation on $[r^2, 1]$:
    \begin{equation*}
		\partial_t f_{y, c} (t, x) + \frac{1}{2} \partial_x^2 f_{y, c} (t, x) = \, 0, \quad f_{y, c} (1, x) = \,  \sup_{u \in \R} \left\{ h \left( y, x + u \right) - \frac{u^2}{2c}  \right\}.
    \end{equation*} 
    Recall from \cref{thm:variation_compute} that
    \begin{equation*}
        h_c (y, x) = \, \operatorname{conc}  \left( h (y, x) - \frac{x^2}{2 c} \right) + \frac{x^2}{2 c}, \quad g_c (y, x) = \frac{\partial h_c (y, x)}{\partial c}.
    \end{equation*}
    Then, the following are equivalent for any $c_* > 0$:
    \begin{itemize}
        \item [(a)] $\inf_{(\mu, c) \in \mathscr{U} \times \R_{>0}} \mathsf{F} (\mu, c, r)$ is achieved at $(\mu_* = 0, c_*)$.
        \item [(b)] $c_*$ satisfies
    \begin{align*}
        & \int_{s}^{1} \left( \E \left[ \left( \partial_x f_{Y, c_*} (t, X_t) \right)^2 \right] - \frac{1}{\alpha c_*^2} (t - r^2) \right) \d t \ge 0, \,\, \forall s \in [r^2, 1], \\
        & \E \left[ g_{c_*} \left( Y, M_1 \right) \right] + \frac{1}{2} \left( \E \left[ \left( \partial_x f_{Y, c_*} (1, X_1) \right)^2 \right] - \frac{1}{\alpha c_*^2} (1 - r^2) \right) = 0,
    \end{align*}
    where $M_1 = c_* \partial_x f_{Y, c_*} (1, X_1) + X_1$.
    \end{itemize}
\end{lem}
\begin{proof}[Proof of \cref{prop:RS_condition}]
    We first prove $(b)$ implies $(a)$. It suffices to show that, for any $\mu \in \mathscr{U}$ and $c > 0$, one always has $\mathsf{F} (\mu, c, r) \ge \mathsf{F} (0, c_*, r)$. Using \cref{prop:veri_arg_supervised}, we know that $\mathsf{F} (\mu, c, r)$ is a convex functional of $\gamma$, where we recall that
    \begin{equation*}
        \gamma (t) = \frac{1}{c + \int_{t}^{1} \mu(s) \d s}.
    \end{equation*}
    For notational convenience, in this proof we recast $\mathsf{F} (\mu, c, r)$ as $V (\gamma)$ whenever the above equation holds. We note that the condition $(\mu, c) \in \mathscr{U} \times \R_{> 0}$ is equivalent to $\gamma$ being strictly positive, non-decreasing, and $1 / \gamma$ being concave. It is easy to show that these properties are preserved under convex combinations of $\gamma$. Therefore, to show that $V (\gamma) \ge V(1/c_*)$ (note that $\gamma \equiv 1/c_*$ if $\mu = 0$ and $c = c_*$), it suffices to prove that
    \begin{equation}\label{eq:RS_FOC}
        \frac{\d}{\d u} \mathsf{F} \left( u \mu, c_*, r \right) \bigg\vert_{u = 0^+} \ge \, 0, \quad \frac{\d}{\d c} \mathsf{F} (0, c, r) \bigg\vert_{c = c_*} = 0.
    \end{equation}
    By \cref{thm:variation_compute} (iii), we have
    \begin{align*}
        \frac{\d}{\d u} \mathsf{F} \left( u \mu, c_*, r \right) \bigg\vert_{u = 0^+} = \, & \frac{1}{2} \int_{r^2}^{1} \mu(t) \left( \E \left[ \left( \partial_x f_{Y, c_*} (t, X_t) \right)^2 \right] - \frac{1}{\alpha c_*^2} (t - r^2) \right) \d t, \\
        \frac{\d}{\d c} \mathsf{F} (0, c, r) \bigg\vert_{c = c_*} = \, & \E \left[ g_{c_*} \left( Y, M_1 \right) \right] + \frac{1}{2} \left( \E \left[ \left( \partial_x f_{Y, c_*} (1, X_1) \right)^2 \right] - \frac{1}{\alpha c_*^2} (1 - r^2) \right).
    \end{align*}
    Using our assumptions and integration by parts, we can show that \cref{eq:RS_FOC} holds for any $\mu \in \mathscr{U}$. The proof of ``$(a)$ implies $(b)$'' follows similarly by computing the first-order variations and using integration by parts.
    This completes the proof.
\end{proof}
By straightforward calculations, condition $(b)$ in \cref{prop:RS_condition} is equivalent to
\begin{align*}
    & \E \left[ g_{c_*} \left( Y, M_1 \right) \right] = \, 0, \quad \E \left[ \left( \partial_x f_{Y, c_*} (1, X_1) \right)^2 \right] - \frac{1}{\alpha c_*^2} (1 - r^2) = 0, \\
    & \int_{s}^{1} \left( \E \left[ \left( \partial_x f_{Y, c_*} (t, X_t) \right)^2 \right] - \frac{1}{\alpha c_*^2} (t - r^2) \right) \d t \ge 0, \,\, \forall s \in [r^2, 1].
\end{align*}
Using Proposition C.4 in \cite{montanari2024exceptional}, we know that
\begin{align*}
    \E \left[ g_{c_*} \left( Y, M_1 \right) \right] = \, 0 \implies \E \left[ h_{c_*} \left( Y, M_1 \right) \right] = \, \E \left[ h \left( Y, M_1 \right) \right].
\end{align*}
Since $M_1 = c_* \partial_x f_{Y, c_*} (1, X_1) + X_1$ and $X_1 = r G + B_1 - B_{r^2}$, the above condition implies that
\begin{equation*}
    h_{c_*} \left( y, c_* \partial_x f_{y, c_*} (1, x) + x \right) = \, h \left( y, c_* \partial_x f_{y, c_*} (1, x) + x \right)
\end{equation*}
for all $(y, x) \in \supp (Y) \times \R$, which further implies that $h_{c_*} (y, x) = h (y, x)$ for all $(y, x) \in \supp(Y) \times \R$ as $\partial_x f_{y, c_*} (1, \cdot)$ is bounded and continuous. 

According to It\^{o}'s formula, we deduce that the mapping
\begin{equation*}
    t \mapsto \E \left[ \left( \partial_x f_{Y, c_*} (t, X_t) \right)^2 \right] - \frac{1}{\alpha c_*^2} (t - r^2)
\end{equation*}
is convex. Therefore, given
\begin{equation*}
    \E \left[ \left( \partial_x f_{Y, c_*} (1, X_1) \right)^2 \right] - \frac{1}{\alpha c_*^2} (1 - r^2) = 0,
\end{equation*}
the condition
\begin{equation*}
    \int_{s}^{1} \left( \E \left[ \left( \partial_x f_{Y, c_*} (t, X_t) \right)^2 \right] - \frac{1}{\alpha c_*^2} (t - r^2) \right) \d t \ge 0, \,\, \forall s \in [r^2, 1]
\end{equation*}
is equivalent to
\begin{equation*}
    \E \left[ \left( \partial_x^2 f_{Y, c_*} (1, X_1) \right)^2 \right] \le \, \frac{1}{\alpha c_*^2}.
\end{equation*}
We are now able to obtain a simpler equivalent form of condition $(b)$:
\begin{equation}\label{eq:condition_RS}
\begin{split}
    & h_{c_*} (y, x) = h (y, x) \,\, \mbox{for all} \,\, (y, x) \in \supp(Y) \times \R, \\
    & \E \left[ \left( \partial_x f_{Y, c_*} (1, X_1) \right)^2 \right] = \frac{1}{\alpha c_*^2} (1 - r^2), \quad \E \left[ \left( \partial_x^2 f_{Y, c_*} (1, X_1) \right)^2 \right] \le \frac{1}{\alpha c_*^2},
\end{split}
\end{equation}
where $Y = \varphi (G, \veps)$, $X_1 = r G + B_1 - B_{r^2}$.



\section{Appendix for \cref{sec:Applications}}\label{append:single_index}

\subsection{Derivation of \cref{conj:replica_single_index}}
\label{app:Derivation_replica_single_index}
In this appendix, we derive the Parisi variational formula presented in \cref{conj:replica_single_index}. Recall from \cref{conj:Parisi_formula_1dim} that for fixed $\alpha, \lambda > 0$, the replica method predicts that
\begin{equation*}
\begin{split}
	& \lim_{n/d \to \alpha} \max_{\ww \in \S^{d-1}} \hCorr_n (\ww) = \, \sup_{r \in [-1, 1]} \inf_{(\mu, c) \in \mathscr{U} \times \R_{> 0} } \frac{1}{\sqrt{\lambda}} \mathsf{F} (\mu, c, r), \\
	& \quad \mathsf{F} (\mu, c, r) = \, \E_{Y, G} [ f_{Y, \mu} ( r^2, r G ) ] + \frac{1}{2 \alpha} \int_{r^2}^{1} \frac{\d t}{c + \int_{t}^{1} \mu(u) \d u }\, ,
\end{split}
\end{equation*}
where $Y = \sqrt{\lambda} \varphi (G) + \veps$, and $f_{y, \mu}$ solves the PDE:
\begin{equation*}
\begin{split}
	& \partial_t f_{y, \mu} (t, x)+\frac{1}{2} \mu (t) (\partial_x f_{y, \mu} (t, x))^2 +  \frac{1}{2} \partial_x^2 f_{y, \mu} (t, x) = \, 0, \\
	& f_{y, \mu} (1, x) = \,  \sup_{u \in \R} \left\{ y \sigma (x+u) - \frac{u^2}{2c}  \right\}.
\end{split}
\end{equation*}
We next compute the limiting Parisi formula as $\alpha \to \infty$, $\lambda \to 0$, $\alpha \lambda \to \oalpha$. In order for the limit of $(1 / \sqrt{\lambda}) \mathsf{F} (\mu, c, r)$ to be meaningful, the function order parameters $(\mu, c)$ should scale like $\sqrt{\lambda}$. Hence, we replace $(\mu, c)$ with $(\sqrt{\lambda} \mu, \sqrt{\lambda} c)$, and obtain that
\begin{equation*}
    \frac{1}{\sqrt{\lambda}} \mathsf{F} \left( \sqrt{\lambda} \mu, \sqrt{\lambda} c, r \right) = \, \frac{1}{\sqrt{\lambda}} \E_{Y, G} \left[ f_{Y, \sqrt{\lambda} \mu} (r^2 , r G) \right] + \frac{1}{2 \alpha \sqrt{\lambda}} \int_{r^2}^{1} \frac{\d t}{\sqrt{\lambda} c + \sqrt{\lambda} \int_{t}^{1} \mu(s) \d s },
\end{equation*}
where $f_{y, \sqrt{\lambda} \mu}$ is the solution to the PDE:
\begin{equation*}
\begin{split}
    &\partial_t f_{y, \sqrt{\lambda} \mu} (t, x)+\frac{1}{2} \sqrt{\lambda} \mu (t) (\partial_x f_{y, \sqrt{\lambda} \mu} (t, x))^2 +  \frac{1}{2} \partial_x^2 f_{y, \sqrt{\lambda} \mu} (t, x) = \, 0, \\
	&f_{y, \sqrt{\lambda} \mu} (1, x) = \,  \sup_{u \in \R} \left\{ y \sigma \left( x + u \right) - \frac{u^2}{2 \sqrt{\lambda} c}  \right\}.
\end{split}    
\end{equation*}
For any fixed $(\mu, c) \in \mathscr{U} \times \R_{> 0}$, it is straightforward to see that $f_{y, \sqrt{\lambda} \mu}$ reduces to $f_{\veps, 0}$ as $\lambda \to 0$, with $f_{\veps, 0}$ being the solution to the heat equation:
\begin{equation*}
\begin{split}
    \partial_t f_{\veps, 0} (t, x) + \frac{1}{2} \partial_x^2 f_{\veps, 0} (t, x) = \, 0, \quad f_{\veps, 0} (1, x) = \, \veps \sigma (x).
\end{split}
\end{equation*}
Of course, the above heat equation can be solved explicitly:
\begin{equation*}
	f_{\veps, 0} (t, x) = \, \veps \E_{H \sim \normal (0, 1)} \big[ \sigma \big( x + \sqrt{1 - t} H \big) \big].
\end{equation*}
Since $Y = \sqrt{\lambda} \varphi (G) + \veps$, we can then perform Taylor expansions around $\lambda = 0$ (using \cref{prop:variation_compute_1} and Feynman-Kac formula) to deduce that
\begin{align*}
    & \E_{Y, G} \left[ f_{Y, \sqrt{\lambda} \mu} (r^2 , r G) \right] \\ 
    = \, & \E_{\veps, G} \left[ f_{\veps, 0} (r^2, r G) \right] + \sqrt{\lambda} \E \left[ \varphi (G) \sigma \left( rG + \int_{r^2}^{1} \d B_t \right) \right] \\
    & + \frac{\sqrt{\lambda}}{2} \left( \int_{r^2}^{1} \mu(t) \E \left[ \partial_x f_{\veps, 0} (t, B_t)^2 \right] \d t + c \E \left[ \veps^2 \sigma' (B_1)^2 \right] \right) + o(\sqrt{\lambda}) \\
    = \, & \E \left[ \veps \sigma \left( r G + \int_{r^2}^{1} \d B_t \right) \right] + \sqrt{\lambda} \xifs (r) + \frac{\sqrt{\lambda}}{2} \left( \int_{r^2}^{1} \mu(t) 
    \xiss' (t) \d t + c \xiss' (1) \right) + o (\sqrt{\lambda}) \\
    = \, & \sqrt{\lambda} \left( \xifs (r) + \frac{1}{2} \int_{r^2}^{1} \mu(t) \xiss' (t) \d t + \frac{1}{2} c \xiss' (1) \right) + o (\sqrt{\lambda}),
\end{align*}
which leads to
\begin{equation*}
\begin{split}
	& \lim_{\alpha \lambda \to \oalpha} \frac{1}{\sqrt{\lambda}} \mathsf{F} \left( \sqrt{\lambda} \mu, \sqrt{\lambda} c, r \right) \\
	= \, & \xifs (r) + \frac{1}{2} \int_{r^2}^{1} \mu(t) \xiss' (t) \d t + \frac{1}{2} c \xiss' (1) + \frac{1}{2 \oalpha} \int_{r^2}^{1} \frac{\d t}{c +  \int_{t}^{1} \mu(s) \d s } = \osF_{\oalpha} (\mu, c, r).
\end{split}	
\end{equation*}
This completes the derivation of the replica prediction~\eqref{eq:replica_single_index}.

\subsection{Proof of \cref{thm:single_index_duality}, \cref{prop:T0_equiv_def} and \cref{prop:optimality_large_oalpha}}
\begin{proof}[Proof of \cref{prop:T0_equiv_def}]
    We begin with defining the value achievable by our two-stage AMP algorithm at a fixed pair $(\alpha, \lambda)$, based on the general results presented in \cref{app:Special_m=1}. Let 
    $$
    A_{\sBayes}^{\alpha, \lambda} = \{ (r, q) \in \R^2: q \in [r^2, 1], \, r^2 \le (c_{\sBayes}^{\alpha, \lambda})^2 (q - r^2) \},
    $$
    and $c_{\sBayes}^{\alpha, \lambda}$ be the smallest positive solution to the fixed point equation (again, $c_{\sBayes}^{\alpha, \lambda} = + \infty$ if this equation has no solution)
\begin{equation*}
    c^2 = \alpha \E 
		\left[ \E \left[ G - c Z \vert c G + Z, Y \right]^2 \right], \quad Y = \sqrt{\lambda} \varphi(G) + \veps, \,\, Z \sim \normal(0, 1), \,\, Z \perp \!\!\! \perp (Y, G).
\end{equation*} 
Further, define for $(r, q) \in [-1, 1] \times [0, 1]$, $q \ge r^2$:
\begin{equation}\label{eq:defn_H_infty_alpha_lambda}
\begin{split}
    & H_{\infty}^{\alpha, \lambda} (r, q) = \, \sup \, \E \left[ Y \sigma \left( Z_{r, q} + \frac{1}{\alpha} F \left( Z_{r, q}, Y \right) + \int_{q}^{1} \left( 1 + \frac{1}{\sqrt{\alpha}} \phi_t \right) \d B_t \right) \right], \\
    & \mbox{subject to} \,\, \text{$F$ is an $(r, q)$-contraction}, \, \text{and} \,\, \sup_{t \in [q, 1]} \E \left[ \phi_t^2 \right] \le 1,
\end{split}
\end{equation}
where $Y = \sqrt{\lambda} \varphi (G) + \veps$, $Z_{r, q} = r G + \sqrt{q - r^2} Z$. 

We first prove that $\liminf_{\alpha \lambda \to \oalpha} c_{\sBayes}^{\alpha, \lambda} \ge c_{\sBayes}$. To this end, note that as $\alpha \to \infty$, $\lambda \to 0$ and $\alpha \lambda \to \oalpha$, we have:
\begin{align*}
& \E \left[ \E\big[G-cZ\big|cG+Z,Y\big]^2 \right]^{1/2} \\
= \, & \E \left[ \E\big[G-cZ\big|cG+Z, \veps \big]^2 \right]^{1/2} + \sqrt{\lambda} \, \E \left[ \E\big[(G-cZ)\varphi(G)\big|cG+Z\big]^2 \right]^{1/2} + O(\lambda) \\
= \, & \sqrt{\lambda} \, \E \left[ \E\big[(G-cZ)\varphi(G)\big|cG+Z\big]^2 \right]^{1/2} + O(\lambda),
\end{align*}
where the last equality follows from the fact that $G - c Z$ is independent of $(c G + Z, \veps)$. Define $W_1 = (cG+Z)/\sqrt{1+c^2}$, $W_2 = (G  - cZ)/\sqrt{1+c^2}$, we then have $W_1, W_2 \stackrel{\iid}{\sim} \normal(0,1)$, and
\begin{align*}
    & \E \left[ \E\big[(G-cZ)\varphi(G)\big|cG+Z\big]^2 \right]^{1/2} \\
    = \, & \E \left[ \E\big[ \sqrt{1 + c^2} W_2 \varphi((c W_1 + W_2) / \sqrt{1 + c^2})\big| W_1 \big]^2 \right]^{1/2} \\
    \stackrel{(i)}{=} \, & \E_{W_1} \left[ \E_{W_2, W_2'} \left[ (1 + c^2) W_2 W_2' \, \varphi \left( \frac{c W_1 + W_2}{\sqrt{1 + c^2}} \right) \varphi \left( \frac{c W_1 + W_2'}{\sqrt{1 + c^2}} \right) \right] \right]^{1/2} \\
    \stackrel{(ii)}{=} \, & \E_{W_1} \left[ \E_{W_2, W_2'} \left[ \varphi' \left( \frac{c W_1 + W_2}{\sqrt{1 + c^2}} \right) \varphi' \left( \frac{c W_1 + W_2'}{\sqrt{1 + c^2}} \right) \right] \right]^{1/2},
\end{align*}
where $W_1, W_2, W_2' \stackrel{\iid}{\sim} \normal (0, 1)$ in $(i)$, and in $(ii)$ we use Stein's identity. Recalling the definition of $\xiff$, it follows that
\begin{equation*}
    \E \left[ \E\big[(G-cZ)\varphi(G)\big|cG+Z\big]^2 \right]^{1/2} = \, \xiff'\Big(\frac{c^2}{1+c^2}\Big)^{1/2}.
\end{equation*}
Therefore, the fixed point equation defining $c_{\sBayes}^{\alpha, \lambda}$ is just
\begin{equation*}
    c^2 = \oalpha \xiff'\Big(\frac{c^2}{1+c^2}\Big) + O \big( \sqrt{\lambda} \big).
\end{equation*}
By definition of $c_{\sBayes}$, we immediately obtain that $c_{\sBayes}^{\alpha, \lambda} \ge c_{\sBayes} - O(\sqrt{\lambda})$ as $\lambda \to 0$, which establishes the desired claim. 

Now we return to the proof of \cref{prop:T0_equiv_def}. By continuity, it suffices to consider $(r, q) \in \operatorname{int} A_{\sBayes}$. Since $\liminf_{\alpha \lambda \to \oalpha} c_{\sBayes}^{\alpha, \lambda} \ge c_{\sBayes}$, for sufficiently small $\lambda$ we have $(r, q) \in A_{\sBayes}^{\alpha, \lambda}$. It then suffices to show that
\begin{equation*}
	\lim_{\alpha \lambda \to \oalpha} \frac{1}{\sqrt{\lambda}} H_{\infty}^{\alpha, \lambda} (r, q) \ge \, \VH_{\oalpha}^{\sAMP} (r, q),
\end{equation*}
with $H_{\infty}^{\alpha, \lambda} (r, q)$ defined in \cref{eq:defn_H_infty_alpha_lambda}. To this end, we define an auxiliary quantity:
\begin{align}\label{eq:H_infty_r_q}
	H_{\infty} (r, q) : = \, \sup \,\, & \bigg\{ \E \left[ \varphi(G) \, \sigma\Big(Z_{r,q}+\int_{q}^{1} \de B_t\Big) \right] +\frac{1}{\oalpha} \E \left[ \eps\sigma'\Big(Z_{r,q}+\int_{q}^{1} \de B_t\Big) F_0(Z_{r,q},\eps) \right] \\
    &+\frac{1}{\sqrt{\oalpha}} \E \left[ \eps\sigma'\Big(Z_{r,q}+\int_{q}^{1} \de B_s\Big) \int_q^1 \phi_t \de B_t \right] \bigg\}, \nonumber\\
\mbox{subject to} \phantom{AA} &\sup_{t \in [q, 1]} \E [ \phi_t^2 ] \le 1\, , \nonumber \\
&     \label{eq:single_index_contraction}
    r = \E\left[\frac{\partial F_0(Z_{r,q},\eps)}{\partial\eps} \,\varphi'(G)\right] \, , \,\,
    q= r^2+\frac{1}{\oalpha} \E\big[ F_0(Z_{r,q},\eps)^2\big] \, , \\
	 &\frac{1}{\oalpha} \E\left[ \frac{\partial F_0}{\partial Z_{r,q}}(Z_{r,q},\eps)^2\right] \le 1\, .\nonumber
\end{align}
We will prove two claims: (a) $\lim_{\alpha \lambda \to \oalpha} H_{\infty}^{\alpha, \lambda} (r, q) / \sqrt{\lambda} \ge H_{\infty} (r, q)$; (b) $H_{\infty} (r, q) = \VH_{\oalpha}^{\sAMP} (r, q)$.

\paragraph{Proof of claim (a).}
 For any $F_0$ and $\phi$ satisfying the constraints in the optimization problem defining $H_{\infty} (r, q)$, let
\begin{equation}\label{eq:from_F_0_to_F}
    F (Z_{r, q}, Y) = \, \frac{1}{\sqrt{\lambda}} \left( F_0 \left(  Z_{r, q}, Y \right) + \sum_{i=1}^{3} a_i (\lambda) g_i (Z_{r, q}, Y) \right),
\end{equation}
where $a(\lambda) := \{ a_i (\lambda) \}_{i=1}^{3}$ are $\lambda$-dependent constants, and $\{ g_i \}_{i=1}^{3}$ are differentiable functions, both to be determined.
Then, we have (note that $Y = \sqrt{\lambda} \varphi (G) + \veps$)
\begin{align*}
    & \E \left[ \frac{\partial F}{\partial G} \left( Z_{r, q}, Y \right) \right] = \, \E \left[ \left( \partial_y F_0 \left( Z_{r, q}, Y \right) + \sum_{i=1}^{3} a_i (\lambda) \partial_y g_i \left( Z_{r, q}, Y \right) \right) \varphi'(G) \right] := A_1 (\lambda, a(\lambda)), \\
    & \frac{1}{\alpha} \E \left[ F \left( Z_{r, q}, Y \right)^2 \right] = \, \frac{1}{\oalpha} \E \left[ \left( F_0 \left(  Z_{r, q}, Y \right) + \sum_{i=1}^{3} a_i (\lambda) g_i (Z_{r, q}, Y) \right)^2 \right] := A_2 (\lambda, a(\lambda)), \\
    & \frac{1}{\alpha} \E \left[ \frac{\partial F}{\partial Z_{r, q}} \left( Z_{r, q}, Y \right)^2 \right] = \, \frac{1}{\oalpha} \E \left[ \left( \partial_z F_0 \left(  Z_{r, q}, Y \right) + \sum_{i=1}^{3} a_i (\lambda) \partial_z g_i (Z_{r, q}, Y) \right)^2 \right] := A_3 (\lambda, a(\lambda)).
\end{align*}
Denote $A (\lambda, a(\lambda)) = \{ A_i (\lambda, a(\lambda)) \}_{i=1}^{3}$. Our goal is to show that for all sufficiently small $\lambda$, there exists $a(\lambda)$ such that $A (\lambda, a(\lambda)) = (r, q-r^2, b(\lambda))$ with $b(\lambda) \le 1$, thus implying that $F$ is an $(r, q)$-contraction.

To this end, we invoke the implicit function theorem. Since $F_0$ satisfies \cref{eq:single_index_contraction}, we know that $A(0, 0) = (r, q-r^2, b(0))$ with $b(0) \le 1$. It suffices to show that there exists $\{ g_i \}_{i=1}^{3}$ such that the Jacobian $(\partial A / \partial a) (0, 0)$ is non-singular, i.e., the vectors $\{ v_i \}_{i=1}^{3} \subset \R^3$ are linearly independent, where
\begin{equation*}
    v_i = \, \left( \E \left[ \partial_y g_i(Z_{r, q}, \veps) \varphi' (G) \right], \E \left[ g_i(Z_{r, q}, \veps) F_0 (Z_{r, q}, \veps) \right], \E \left[ \partial_z g_i(Z_{r, q}, \veps) \partial_z F_0 (Z_{r, q}, \veps) \right] \right).
\end{equation*}
This is equivalent to showing that the linear mapping
\begin{equation*}
    g \mapsto v(g) := \left( \E \left[ \partial_y g(Z_{r, q}, \veps) \varphi' (G) \right], \E \left[ g(Z_{r, q}, \veps) F_0 (Z_{r, q}, \veps) \right], \E \left[ \partial_z g(Z_{r, q}, \veps) \partial_z F_0 (Z_{r, q}, \veps) \right] \right)
\end{equation*}
is surjective, which can be easily proved using Gaussian integration by parts and noting that $Z_{r, q} \neq G$ since $(r, q) \in \operatorname{int} A_{\sBayes}$.

We have verified that $F$ defined as per \cref{eq:from_F_0_to_F} is an $(r, q)$-contraction. 
Next we compute
\begin{equation*}
    \lim_{\alpha \lambda \to \oalpha} \frac{1}{\sqrt{\lambda}} \E \left[ Y \sigma \left( Z_{r, q} + \frac{1}{\alpha} F \left( Z_{r, q}, Y \right) + \int_{q}^{1} \left( 1 + \frac{1}{\sqrt{\alpha}} \phi_t \right) \d B_t \right) \right].
\end{equation*}
Note that since
\begin{align*}
    F (Z_{r, q}, Y) = \, & \frac{1}{\sqrt{\lambda}} \left( F_0 \left(  Z_{r, q}, Y \right) + \sum_{i=1}^{3} a_i (\lambda) g_i (Z_{r, q}, Y) \right) \\
    = \, & \frac{1}{\sqrt{\lambda}} \left( F_0 (Z_{r, q}, Y) + o(1) \right) = \frac{1}{\sqrt{\lambda}} \left( F_0 (Z_{r, q}, \veps) + o(1) \right),
\end{align*}
we have
\begin{small}
\begin{align*}
    & \frac{1}{\sqrt{\lambda}} \E \left[ Y \sigma \left( Z_{r, q} + \frac{1}{\alpha} F \left( Z_{r, q}, Y \right) + \int_{q}^{1} \left( 1 + \frac{1}{\sqrt{\alpha}} \phi_t \right) \d B_t \right) \right] \\
    = \, & \E \left[ \left( \varphi(G) + \frac{\veps}{\sqrt{\lambda}} \right) \sigma \left( Z_{r, q} + \frac{\sqrt{\lambda}}{\oalpha} F_0 \left( Z_{r, q}, \veps \right) + \int_{q}^{1} \left( 1 + \frac{1}{\sqrt{\alpha}} \phi_t \right) \d B_t + o (\sqrt{\lambda}) \right) \right] \\
    = \, & \E \left[ \left( \varphi(G) + \frac{\veps}{\sqrt{\lambda}} \right) \left( \sigma \left( Z_{r, q} +  \int_{q}^{1} \d B_t \right) + \sigma' \left( Z_{r, q} +  \int_{q}^{1} \d B_t \right) \left( \frac{\sqrt{\lambda}}{\oalpha} F_0 \left( Z_{r, q}, \veps \right) + \frac{1}{\sqrt{\alpha}} \int_{q}^{1} \phi_t \d B_t \right) + o(\sqrt{\lambda} )\right) \right] \\
    = \, & \E \left[ \varphi(G) \, \sigma\Big(Z_{r,q}+\int_{q}^{1} \de B_t\Big) \right] +\frac{1}{\oalpha} \E \left[ \eps\sigma'\Big(Z_{r,q}+\int_{q}^{1} \de B_t\Big) F_0(Z_{r,q},\eps) \right] +\frac{1}{\sqrt{\oalpha}} \E \left[ \eps\sigma'\Big(Z_{r,q}+\int_{q}^{1} \de B_s\Big) \int_q^1 \phi_t \de B_t \right] + o(1),
\end{align*}    
\end{small}
which leads to
\begin{align*}
    & \lim_{\alpha \lambda \to \oalpha} \frac{1}{\sqrt{\lambda}} \E \left[ Y \sigma \left( Z_{r, q} + \frac{1}{\alpha} F \left( Z_{r, q}, Y \right) + \int_{q}^{1} \left( 1 + \frac{1}{\sqrt{\alpha}} \phi_t \right) \d B_t \right) \right] \\
    = \, & \E \left[ \varphi(G) \, \sigma\Big(Z_{r,q}+\int_{q}^{1} \de B_t\Big) \right] +\frac{1}{\oalpha} \E \left[ \eps\sigma'\Big(Z_{r,q}+\int_{q}^{1} \de B_t\Big) F_0(Z_{r,q},\eps) \right] \\
    & + \frac{1}{\sqrt{\oalpha}} \E \left[ \eps\sigma'\Big(Z_{r,q}+\int_{q}^{1} \de B_s\Big) \int_q^1 \phi_t \de B_t \right].
\end{align*}
The above value is achievable for any $F_0$ and $\phi$ satisfying $\sup_{t \in [q, 1]} \E [ \phi_t^2 ] \le 1$ and \eqref{eq:single_index_contraction}, thus proving our claim (a).

\paragraph{Proof of claim (b).} We consider the three terms in the definition of $H_{\infty} (r, q)$, respectively. For the first term, we note that by definition of $\xifs$ and $Z_{r, q}$, 
\begin{align*}
\E\left[\varphi(G)\sigma\left(Z_{r,q} + \int_{q}^{1} \d B_t \right)\right] = \xifs(r)\, .
\end{align*}

As for the third term in \cref{eq:H_infty_r_q}, we have
\begin{align*}
    & \E \left[ \eps\sigma'\Big(Z_{r,q}+\int_{q}^{1} \de B_s \Big) \int_q^1 \phi_t \de B_t \right] = \, \E \left[ \int_q^1 \eps\sigma'\Big(Z_{r,q}+\int_{q}^{1} \de B_s \Big) \phi_t \de B_t \right] \\
    \stackrel{(i)}{=} \, & \E \left[ \int_q^1 \eps \phi_t \left( \sigma'\Big(Z_{r,q}+\int_{q}^{1} \de B_s\Big) - \E \left[ \sigma'\Big(Z_{r,q}+\int_{q}^{1} \de B_s \Big) \Big\vert \cF_t \right] \right) \de B_t \right] \\
    \stackrel{(ii)}{=} \, & \E \left[ \int_q^1 \eps \phi_t \left( \int_{t}^{1} \E \left[ \sigma'' \Big(Z_{r,q}+\int_{q}^{1} \de B_s \Big) \Big\vert \cF_u \right] \d B_u \right) \de B_t \right] \\
    = \, & \int_{q}^{1} \E \left[ \veps \phi_t \E \left[ \sigma'' \Big(Z_{r,q}+\int_{q}^{1} \de B_s \Big) \Big\vert \cF_t \right] \right] \d t,
\end{align*}
where $(i)$ is because $\veps$ and $\phi_t$ are adapted to $\cF_t$, and in $(ii)$ we use the Clark-Ocone formula (cf. \cite{nualart2006malliavin}). By Cauchy-Schwarz inequality, for any $t \in [q, 1]$:
\begin{align*}
    & \E \left[ \veps \phi_t \E \left[ \sigma'' \Big(Z_{r,q}+\int_{q}^{1} \de B_s \Big) \Big\vert \cF_t \right] \right] \\
    \le \, & \E [\phi_t^2]^{1/2} \E \left[ \veps^2 \E \left[ \sigma'' \Big(Z_{r,q}+\int_{q}^{1} \de B_s \Big) \Big\vert \cF_t \right]^2 \right]^{1/2} \le \, \E \left[ \veps^2 \E \left[ \sigma'' \Big(Z_{r,q}+\int_{q}^{1} \de B_s \Big) \Big\vert \cF_t \right]^2 \right]^{1/2} \\
    \stackrel{(i)}{=} \, & \E_{\veps, G_1} \left[ \veps^2 \E_{G_2, G_2'} \left[ \sigma'' (\sqrt{t} G_1 + \sqrt{1-t} G_2) \sigma'' (\sqrt{t} G_1 + \sqrt{1-t} G_2') \right] \right]^{1/2} \\
    = \, & \E_{G_1, G_2, G_2'} \left[ \sigma'' (\sqrt{t} G_1 + \sqrt{1-t} G_2) \sigma'' (\sqrt{t} G_1 + \sqrt{1-t} G_2') \right]^{1/2} = \xiss'' (t)^{1/2},
\end{align*}
where $G_1, G_2, G_2' \stackrel{\iid}{\sim} \normal (0, 1)$ in $(i)$. Further, the equalities are achieved at
\begin{equation*}
    \phi_t = \frac{\E \left[ \veps \sigma'' \Big(Z_{r,q}+\int_{q}^{1} \de B_s \Big) \Big\vert \cF_t \right]}{\E \left[ \E \left[ \veps \sigma'' \Big(Z_{r,q}+\int_{q}^{1} \de B_s \Big) \Big\vert \cF_t \right]^2 \right]^{1/2}}.
\end{equation*}
Therefore, for any fixed $(r, q)$, the maximum of the third term in \cref{eq:H_infty_r_q} is
\begin{equation*}
     \frac{1}{\sqrt{\oalpha}} \int_{q}^{1} \xiss'' (t)^{1/2} \d t.
\end{equation*}

Finally, we consider optimizing the second term in \cref{eq:H_infty_r_q}. Our goal is to maximize
\begin{equation*}
    \E \left[ \eps\sigma'\Big(Z_{r,q}+\int_{q}^{1} \de B_t\Big) F_0(Z_{r,q},\eps) \right],
\end{equation*}
subject to the constraints in~\cref{eq:single_index_contraction}. Assume that $F_0$ has the following Hermite polynomial decomposition:
\begin{equation*}
    F_0 (\sqrt{q} x, y) = \sum_{k, l = 0}^{\infty} f_{k,l} \mathrm{He}_k (x) \mathrm{He}_l (y),
\end{equation*}
where $\mathrm{He}_k$ is the $k$-th normalized Hermite polynomial. The $(r, q)$-contraction conditions~\eqref{eq:single_index_contraction} are equivalent to
\begin{align*}
    r = \sum_{k = 0}^{\infty} \sqrt{k+1} \varphi_{k+1} f_{k,1} \left( \frac{r}{\sqrt{q}} \right)^k, \quad q = r^2 + \frac{1}{\oalpha} \sum_{k,l=0}^{\infty} f_{k, l}^2, \quad \frac{1}{\oalpha q} \sum_{k, l = 0}^{\infty} k f_{k, l}^2 \le 1.
\end{align*}
Further,
\begin{equation*}
    \E \left[ \eps\sigma'\Big(Z_{r,q}+\int_{q}^{1} \de B_t\Big) F_0(Z_{r,q},\eps) \right] = \, \sum_{k=0}^{\infty} \sqrt{k+1} \sigma_{k+1} f_{k, 1} (\sqrt{q})^k.
\end{equation*}
It is straightforward to see that this optimization problem is equivalent to the following:
\begin{equation}\label{eq:optimize_F_0}
\begin{split}
    & \mbox{maximize} \,\, \sum_{k=0}^{\infty} \sqrt{k+1} \sigma_{k+1} f_{k, 1} (\sqrt{q})^k, \\
    & \mbox{s.t.} \,\, r = \sum_{k = 0}^{\infty} \sqrt{k+1} \varphi_{k+1} f_{k,1} \left( \frac{r}{\sqrt{q}} \right)^k, \,\, q \ge r^2 + \frac{1}{\oalpha} \sum_{k=0}^{\infty} f_{k, 1}^2, \,\, \frac{1}{\oalpha q} \sum_{k=0}^{\infty} k f_{k, 1}^2 \le 1.
\end{split}
\end{equation}
Introducing the change of variables $x_{k+1} = f_{k, 1} / \sqrt{\oalpha(q-r^2)}$, $k\ge 0$, and recalling the definitions of $c$, $v_k$'s and $s_k$'s from \cref{eq:T0_change_of_variable}, it is straightforward to see that the value of the optimization problem~\eqref{eq:optimize_F_0} is equal to $\sqrt{\oalpha(q-r^2)} T_0 (r, q)$. This establishes our claim (b) and completes the proof of \cref{prop:T0_equiv_def}.
\end{proof}

\begin{proof}[Proof of \cref{thm:single_index_duality}]
\noindent \textbf{Proof of $(a)$.} Comparing the expressions for $\VH_{\oalpha}^{\sAMP} (r, q)$ and $\VH_{\oalpha}^*(r,q)$, it suffices to show that $T_0 (r, q) \le \sqrt{\xiss' (q)}$. By definition of $T_0 (r, q)$ and Cauchy-Schwarz inequality, we have
\begin{equation*}
    T_0 (r, q) \le \, \norm{s}_2 = \sqrt{\sum_{k \ge 1} k \sigma_k^2 q^{k-1}} = \sqrt{\xiss' (q)}.
\end{equation*}
This proves part $(a)$.

\vspace{0.5em}
\noindent \textbf{Proof of $(b)$.} We will prove a stronger result: For all $(q, r)$ with $r^2 \le q \le 1$, $\VH_{\oalpha}^{\sAMP} (r, q) \le \, \VH_{\oalpha}^* (r, q_*(r))$. To this end, denote $T(r, q) = \sqrt{(q-r^2)} T_0 (r, q) / \sqrt{\oalpha}$, the second term in the expression for $\VH_{\oalpha}^{\sAMP} (r, q)$. Then, the proof of \cref{prop:T0_equiv_def} implies that,
\begin{align*}
    & T (r, q) = \, \max \, \frac{1}{\oalpha} \E \left[ \sigma'\Big(Z_{r,q}+\int_{q}^{1} \de B_t\Big) F_0(Z_{r,q}) \right], \\
    & \mbox{s.t.} \,\, r = \E\left[ F_0(Z_{r,q}) \varphi'(G)\right] , \,\, q \ge r^2+\frac{1}{\oalpha} \E\big[ F_0(Z_{r,q})^2\big] , \,\, \frac{1}{\oalpha} \E\left[ F_0' (Z_{r,q})^2\right] \le 1.
\end{align*}
We can view $r G$ as $B_{r^2}$ and $Z_{r, q} = r G + \sqrt{q - r^2} Z$ as $B_q$, so that the objective function can be rewritten as (using the Clark-Ocone formula from \cite{nualart2006malliavin})
\begin{align*}
    & \E \left[ \sigma'\Big(Z_{r,q}+\int_{q}^{1} \de B_t\Big) F_0(Z_{r,q}) \right] = \, \E \left[ \sigma' (B_1) F_0(B_q) \right] \\
    = \, & \E \left[ \left( \E [\sigma' (B_1) \vert \cF_{r^2}] + \int_{r^2}^{1} \E [\sigma'' (B_1) \vert \cF_t] \d B_t \right) \left( \E [F_0 (B_q) \vert \cF_{r^2}] + \int_{r^2}^{q} \E [F_0' (B_q) \vert \cF_t] \d B_t \right) \right] \\
    = \, & \E \left[ \E [\sigma' (B_1) \vert \cF_{r^2}] \E [F_0 (B_q) \vert \cF_{r^2}] \right] + \int_{r^2}^{q} \E \left[ \E [\sigma'' (B_1) \vert \cF_t] \E [F_0' (B_q) \vert \cF_t] \right] \d t \\
    = \, & \E \left[ V U_{r^2} \right] + \int_{r^2}^{q} \E \left[ \phi_t \psi_t \right] \d t,
\end{align*}
where we denote $U_{r^2} = \E [\sigma' (B_1) \vert \cF_{r^2}]$, $V = \E [F_0 (B_q) \vert \cF_{r^2}]$, $\psi_t = \E [\sigma'' (B_1) \vert \cF_t]$, and $\phi_t = \E [F_0' (B_q) \vert \cF_t]$ for $t \in [r^2, q]$. Our second and third constraints are equivalent to
\begin{equation*}
    \sup_{t \in [r^2, q]} \E \left[ \phi_t^2 \right] \le \oalpha, \quad \mbox{and} \,\, \E \left[ V^2 \right] + \int_{r^2}^{q} \E \left[ \phi_t^2 \right] \d t \le \oalpha (q - r^2).
\end{equation*}
Ignoring the first constraint, we can use Lagrange duality to deduce that
\begin{align*}
    \oalpha T(r, q) \le \, \sup_{V, \phi} \inf_{\gamma, \lambda \ge 0} \bigg\{ \, & \E \left[ V U_{r^2} \right] + \int_{r^2}^{q} \E \left[ \phi_t \psi_t \right] \d t - \frac{1}{2} \int_{r^2}^{q} \gamma(t) \left( \E \left[ \phi_t^2 \right] - \oalpha \right) \d t \\
    & - \frac{\lambda}{2} \left( \E \left[ V^2 \right] + \int_{r^2}^{q} \left( \E \left[ \phi_t^2 \right] - \oalpha \right) \d t \right) \bigg\} \\
    \le \, \inf_{\gamma, \lambda \ge 0} \sup_{V, \phi} \bigg\{ \, & \E \left[ V U_{r^2} - \frac{\lambda}{2} V^2 \right] + \int_{r^2}^{q} \E \left[ \phi_t \psi_t - \frac{\gamma(t) + \lambda}{2} \phi_t^2 \right] \d t \\
    & + \frac{\oalpha}{2} \int_{r^2}^q \left( \gamma(t) + \lambda \right) \d t \bigg\} \\
    = \, \inf_{\gamma, \lambda \ge 0} \bigg\{ \frac{1}{2 \lambda} \, & \E \left[ U_{r^2}^2 \right] + \frac{1}{2} \int_{r^2}^q \frac{1}{\gamma(t) + \lambda} \E \left[ \psi_t^2 \right] \d t + \frac{\oalpha}{2} \int_{r^2}^q \left( \gamma(t) + \lambda \right) \d t \bigg\} \\
    = \, \inf_{\gamma, \lambda \ge 0} \bigg\{ \frac{1}{2 \lambda} \, & \xiss' (r^2) + \frac{1}{2} \int_{r^2}^q \frac{1}{\gamma(t) + \lambda} \xiss'' (t) \d t + \frac{\oalpha}{2} \int_{r^2}^q \left( \gamma(t) + \lambda \right) \d t \bigg\} \, ,
\end{align*}
where by definition, we have 
\begin{equation*}
    \E \left[ \psi_t^2 \right] = \xiss'' (t), \quad \E \left[ U_{r^2}^2 \right] = \xiss' (r^2).
\end{equation*}
This yields that for all $q \in [r^2, 1]$:
\begin{align*}
    \VH_{\oalpha}^{\sAMP} (r, q) = \, & \xifs(r) + T(r, q) + \frac{1}{\sqrt{\oalpha}}\int_q^1\sqrt{\xiss''(t)}\, \de t \\
    \le \, & \xifs(r) + \inf_{\gamma, \lambda \ge 0} \bigg\{ \frac{1}{2 \oalpha \lambda} \xiss' (r^2) + \frac{1}{2 \oalpha} \int_{r^2}^{1} \frac{1}{\gamma(t) + \lambda \bone_{t \le q}} \xiss'' (t) \d t + \frac{1}{2} \int_{r^2}^{1} \left( \gamma(t) + \lambda \bone_{t \le q} \right) \d t \bigg\} \, \\
    = \, & \xifs(r) + \inf_{\substack{\gamma, \lambda \ge 0 \\ \gamma \vert_{[r^2, q]} \ge \lambda}} \bigg\{ \frac{1}{2 \oalpha \lambda} \xiss' (r^2) + \frac{1}{2 \oalpha} \int_{r^2}^{1} \frac{1}{\gamma(t)} \xiss'' (t) \d t + \frac{1}{2} \int_{r^2}^{1}  \gamma(t) \d t \bigg\} \,.
\end{align*}
Therefore,
\begin{align*}
    \sup_{q \in [r^2, 1]} \VH_{\oalpha}^{\sAMP} (r, q) \le \, \xifs(r) + \inf_{ \gamma \vert_{[r^2, 1]} \ge \lambda \ge 0} \bigg\{ \frac{1}{2 \oalpha \lambda} \xiss' (r^2) + \frac{1}{2 \oalpha} \int_{r^2}^{1} \frac{1}{\gamma(t)} \xiss'' (t) \d t + \frac{1}{2} \int_{r^2}^{1}  \gamma(t) \d t \bigg\} \,.  
\end{align*}
To compute this infimum, let us consider two cases:
\begin{enumerate}
    \item $\lambda > \sqrt{\xiss'' (1) / \oalpha}$. Then the optimal $\gamma$ is just $\lambda$, and we obtain that
    \begin{align*}
        \sup_{q \in [r^2, 1]} \VH_{\oalpha}^{\sAMP} (r, q) \le \, \xifs(r) + \inf_{\lambda > \sqrt{\xiss'' (1) / \oalpha}} \bigg\{ \frac{1}{2 \lambda \oalpha} \xiss'(1) + \frac{\lambda}{2} (1 - r^2) \bigg\} \, .
    \end{align*}
    If $\xiss'(1) > \xiss'' (1) (1 - r^2)$, the above infimum is achieved at $\lambda = \sqrt{\xiss'(1) / \oalpha (1 - r^2)}$, and we obtain the upper bound
    \begin{equation*}
        \sup_{q \in [r^2, 1]} \VH_{\oalpha}^{\sAMP} (r, q) \le \, \xifs(r) + \sqrt{\frac{(1-r^2) \xiss' (1)}{\oalpha}} \, = \VH_{\oalpha}^* (r, 1) = \VH_{\oalpha}^* (r, q_*(r)) ,
    \end{equation*}
    where $q_* (r) = 1$ follows from \cref{lem:q_star_r}.
    Otherwise, the infimum is achieved at $\lambda = \sqrt{\xiss'' (1) / \oalpha}$, which will be considered in the next case.
    
    \item $\lambda \le \sqrt{\xiss'' (1) / \oalpha}$. In this case, there exists some $s \in [r^2, 1]$ such that $\lambda = \sqrt{\xiss''(s) / \oalpha}$, and the optimal $\gamma$ satisfies
    \begin{equation*}
        \gamma(t) = \bone_{t \ge s} \cdot \sqrt{\frac{\xiss''(t)}{\oalpha}} + \bone_{t < s} \cdot \lambda.
    \end{equation*}
    We similarly obtain that
    \begin{align*}
        \sup_{q \in [r^2, 1]} \VH_{\oalpha}^{\sAMP} (r, q) \le \, \xifs (r) + \frac{\xiss'(s)}{2 \sqrt{\oalpha \xiss''(s)}} + \frac{\sqrt{\xiss''(s)}}{2 \sqrt{\oalpha}} \left( s - r^2 \right) +  \frac{1}{\sqrt{\oalpha}} \int_{s}^{1} \sqrt{\xiss'' (t)} \d t.
    \end{align*}
    Differentiating the above right-hand side with respect to $s$ and using \cref{lem:q_star_r}, we deduce that it is minimized at $s = q_* (r)$. Hence,
    \begin{equation*}
        \sup_{q \in [r^2, 1]} \VH_{\oalpha}^{\sAMP} (r, q) \le \, \xifs (r) + \sqrt{\frac{(s-r^2) \xiss'(s)}{\oalpha}} +  \frac{1}{\sqrt{\oalpha}} \int_{s}^{1} \sqrt{\xiss'' (t)} \d t = \VH_{\oalpha}^* (r, q_*(r)).
    \end{equation*}
\end{enumerate}
Combining the two cases and \cref{lem:q_star_r}, we conclude that $\sup_{q \in [r^2, 1]} \VH_{\oalpha}^{\sAMP} (r, q) \le \VH_{\oalpha}^* (r, q_*(r))$. This completes the proof of part $(b)$.

\vspace{0.5em}
\noindent \textbf{Proof of $(c)$.}
We first show that the set $\{ r \in [-1, 1]: (r, q_*(r)) \in A_{\sBayes} \}$ is an interval. By definition,
\begin{equation*}
    (r, q_*(r)) \in A_{\sBayes} \Longleftrightarrow q_* (r) \ge \left( 1 + c_{\sBayes}^{-2} \right) r^2 := q_{\sBayes} (r).
\end{equation*}
If $q_* (r) = 1$, this is equivalent to $(r, 1) \in A_{\sBayes}$, which is already enforced by the constraint $(r, q_*(r)) \in A_{\sBayes}$. Now we consider the case $q_* (r) < 1$. According to \cref{lem:q_star_r}, $q_* (r)$ uniquely solves the equation
\begin{equation*}
    r^2 = \, q - \frac{\xiss' (q)}{\xiss'' (q)}.
\end{equation*}
By direct calculation, we can verify that the mapping $q \mapsto q - \xiss' (q) / \xiss'' (q)$ is monotone increasing. Therefore, $q_* (r) \ge q_{\sBayes} (r)$ is equivalent to
\begin{equation*}
    q_{\sBayes} (r) - \frac{\xiss' (q_{\sBayes} (r))}{\xiss'' (q_{\sBayes} (r))} \le \, r^2  \Longleftrightarrow \frac{q_{\sBayes} (r) \xiss'' (q_{\sBayes} (r))}{\xiss' (q_{\sBayes} (r))} \le \, 1 + c_{\sBayes}^2.
\end{equation*}
It suffices to show that the set of $q_{\sBayes} (r)$'s satisfying the above inequality is an interval. Note that
\begin{align*}
    & \frac{q \xiss'' (q)}{\xiss' (q)} \le \, 1 + c_{\sBayes}^2 \Longleftrightarrow \, \sum_{k \ge 2} k (k-1) \sigma_k^2 q^{k-1} \le \, \left( 1 + c_{\sBayes}^2 \right) \sum_{k \ge 1} k \sigma_k^2 q^{k-1} \\
    \Longleftrightarrow \, & \sum_{k \ge 2 + c_{\sBayes}^2} k \left( k - 2 - c_{\sBayes}^2 \right) \sigma_k^2 q^{k-1} \le \, \left( 1 + c_{\sBayes}^2 \right) \sigma_1^2 + \sum_{2 \le k < 2 + c_{\sBayes}^2} k \left( 2 + c_{\sBayes}^2 - k \right) \sigma_k^2 q^{k-1}.
\end{align*}
Of course, if $q$ satisfies the above inequality, then any $q' \le q$ also satisfies it. This proves our claim.

We can now denote $\{ r \in [-1, 1]: (r, q_*(r)) \in A_{\sBayes} \} = [ - r_A, r_A ]$. By continuity, we must have
$$
\frac{r_A}{\sqrt{q_* (r_A) - r_A^2}} = c_{\sBayes}.
$$ 
We first derive the first-order condition satisfied by $r_*$. On the one hand, in the region $\{ r : (r, 1) \in A_{\sBayes}, q_*(r) = 1 \}$ we have
\begin{align*}
    \frac{\de \VH_{\oalpha}^*}{\de r} (r, q_*(r)) = \, \xifs' (r) - \sqrt{\frac{\xiss' (1)}{\oalpha}} \frac{r}{\sqrt{1 - r^2}}.
\end{align*}
On the other hand, in the region $\{ r : (r, 1) \in A_{\sBayes}, q_*(r) < 1 \}$, by the envelope theorem we have
\begin{align*}
    \frac{\de \VH_{\oalpha}^*}{\de r} (r, q_*(r)) = \, \xifs' (r) - \sqrt{\frac{\xiss' (q_*(r))}{\oalpha}} \frac{r}{\sqrt{q_*(r) - r^2}}.
\end{align*}
To summarize, we always have
\begin{equation*}
    \frac{\de \VH_{\oalpha}^*}{\de r} (r, q_*(r)) = \, \xifs' (r) - \sqrt{\frac{\xiss' (q_*(r))}{\oalpha}} \frac{r}{\sqrt{q_*(r) - r^2}}.
\end{equation*}
By definition of $c_{\sBayes}$ and Cauchy-Schwarz inequality (cf. \cref{lem:xi_cauchy_schwarz}):
\begin{align*}
    & \frac{\de \VH_{\oalpha}^*}{\de r} (r_A, q_*(r_A)) = \, \xifs' (r_A) - \sqrt{\frac{\xiss' (q_*(r_A))}{\oalpha}} \frac{r_A}{\sqrt{q_*(r_A) - r_A^2}} \\
    \le \, & \sqrt{\xiss' (q_*(r_A)) \xiff' \left( \frac{r_A^2}{q_*(r_A)} \right)} - \sqrt{\frac{\xiss' (q_*(r_A))}{\oalpha}} \frac{r_A}{\sqrt{q_*(r_A) - r_A^2}} \\
    = \, & \sqrt{\xiss' (q_*(r_A))} \left( \sqrt{\xiff' \left( \frac{c_{\sBayes}^2}{1 + c_{\sBayes}^2} \right)} - \frac{c_{\sBayes}}{\sqrt{\oalpha}} \right) = 0.
\end{align*}
Similarly, we can show that
\begin{equation*}
    \frac{\de \VH_{\oalpha}^*}{\de r} ( - r_A, q_*(- r_A)) \ge \, 0.
\end{equation*}
Since $r_* = \arg \max_{[-r_A, r_A]} \VH_{\oalpha}^* (r,q_*(r))$, we must have
\begin{equation*}
    0 = \, \frac{\de \VH_{\oalpha}^*}{\de r} (r_*, q_*(r_*)) = \xifs' (r_*) - \sqrt{\frac{\xiss' (q_*)}{\oalpha}} \frac{r_*}{\sqrt{q_* - r_*^2}}.
\end{equation*}
Further, $q_* = q_* (r)$ implies that $(q_*-r_*^2)\xiss''(q_*) \le \xiss'(q_*)$.

It now remains to show that
\begin{equation*}
    T(r_*, q_*) = \, \frac{1}{\sqrt{\oalpha}} \sqrt{(q_* - r_*^2) \xiss'(q_*)} \, .
\end{equation*}
In \cref{eq:optimize_F_0}, the choice of $\{ f_{k, 1} \}_{k=0}^{\infty}$ achieving $\sqrt{(q_* - r_*^2) \xiss'(q_*) / \oalpha}$ is given by
\begin{equation*}
    f_{k, 1} = \, \sqrt{\frac{\oalpha (q_* - r_*^2)}{\xiss' (q_*)}} \sqrt{k+1} \sigma_{k+1} q_*^{k/2}.
\end{equation*}
It suffices to verify the first and third constraints. For the first one, we have
\begin{align*}
    \sum_{k = 0}^{\infty} \sqrt{k+1} \varphi_{k+1} f_{k,1} \left( \frac{r_*}{\sqrt{q_*}} \right)^k = \, \sqrt{\frac{\oalpha (q_* - r_*^2)}{\xiss' (q_*)}} \sum_{k = 0}^{\infty} (k+1) \varphi_{k+1} \sigma_{k+1} r_*^k = \, \sqrt{\frac{\oalpha (q_* - r_*^2)}{\xiss' (q_*)}} \xifs' (r_*) = r_*,
\end{align*}
where the last equality is due to the first-order condition for $r_*$. The third constraint follows from the definition of $q_*$:
\begin{align*}
    \frac{1}{\oalpha q_*} \sum_{k=0}^{\infty} k f_{k, 1}^2 \le \, & \frac{1}{\oalpha q_*} \frac{\oalpha (q_* - r_*^2)}{\xiss' (q_*)} \sum_{k=0}^{\infty} k (k+1) \sigma_{k+1}^2 q_*^k \\
    = \, & \frac{1}{\oalpha q_*} \frac{\oalpha (q_* - r_*^2)}{\xiss' (q_*)} q_* \xiss'' (q_*) = \frac{(q_* - r_*^2) \xiss'' (q_*)}{\xiss' (q_*)} \le 1.
\end{align*}
This establishes $\VH_{\oalpha}^{\sAMP}(r_*, q_*) =  \VH_{\oalpha}^* (r_*, q_*)$. Finally, the ``as a consequence'' part follows automatically, completing the proof of part $(c)$.
\end{proof}

\begin{proof}[Proof of \cref{prop:optimality_large_oalpha}]
	From the proof of \cref{thm:single_index_duality}, we know that
	$$
	\{ r \in [-1, 1]: (r, q_*(r)) \in A_{\sBayes} \} = [ - r_A, r_A ].
	$$
	Further, we have
	$$ 
	\{ r \in [-1, 1]: (r, 1) \in A_{\sBayes} \} = [ - r_B, r_B ], \quad r_B = \frac{c_{\sBayes}}{\sqrt{1 + c_{\sBayes}^2}}.
	$$
	Therefore, $r_A = r_B$ if and only if $q_* (r_B) = 1$, which (by \cref{lem:q_star_r}) is equivalent to
	\begin{equation*}
		\xiss' (1) \ge \, \xiss'' (1) (1 - r_B^2) = \frac{1}{1 + c_{\sBayes}^2} \xiss'' (1).
	\end{equation*}
	By definition, $c_{\sBayes}$ is an increasing function of $\oalpha$. Therefore, the above equation holds for large enough $\oalpha$, with the threshold depending only on $\varphi$ and $\sigma$. This completes the proof.
\end{proof}

\subsection{Appendix for \cref{sec:Gaussian_model_setup}}\label{append:generalized_tensor_pca}

\subsubsection{Parisi formula}
From a statistical physics perspective, the objective function $\hCorr^g_d (\ww)$ for generalized tensor PCA can be viewed as the Hamiltonian of a spherical mixed $p$-spin model with a planted signal $\ww_*$. In this section, we derive the Parisi formula for predicting the asymptotic global maximum of $\hCorr^g_d (\ww)$ over the unit sphere $\S^{d-1}$, building on foundational results from spin glass theory \cite{chen2017parisi,jagannath2017low}.

 Let $\ww = r \ww_* + \sqrt{1 - r^2} \ww_{\perp}$, with $\ww_{\perp} \in \S^{d-1}$, $\ww_{\perp} \perp \ww_*$, we obtain that
\begin{align*}
    \hCorr^g_d(\ww) = \, \sum_{k \ge 1} \sigma_k \varphi_k \langle \ww_*, \ww \rangle^k + \frac{1}{\sqrt{d \oalpha}} \sum_{k \ge 1} \sigma_k \langle \GG^{(k)},\ww^{\otimes k} \rangle := \xifs (r) + \frac{1}{\sqrt{d}} H_{\sigma, \oalpha} (\ww),
\end{align*}
where $H_{\sigma, \oalpha}$ is a mean-zero Gaussian process on the unit sphere, with covariance
\begin{equation*}
    \E \left[ H_{\sigma, \oalpha} (\ww_1) H_{\sigma, \oalpha} (\ww_2) \right] = \, \frac{1}{\oalpha} \xiss (\langle \ww_1, \ww_2 \rangle).
\end{equation*}
It thus follows that
\begin{align*}
    \max_{\ww \in \S^{d-1}} \hCorr^g_d(\ww) = \max_{\substack{r \in [-1, 1] \\ \ww_\perp \in \S^{d-1}, \, \ww_{\perp} \perp \ww_* } } \left\{ \xifs (r) + \frac{1}{\sqrt{d}} H_{\sigma, \oalpha} \left( r \ww_* + \sqrt{1 - r^2} \ww_{\perp} \right) \right\}.
\end{align*}
Fix $r \in [-1, 1]$, for any $\ww_{\perp}, \ww_{\perp}' \in \S^{d-1}$ orthogonal to $\ww_*$, we have
\begin{equation*}
    \E \left[ H_{\sigma, \oalpha} \left( r \ww_* + \sqrt{1 - r^2} \ww_{\perp} \right) H_{\sigma, \oalpha} \left( r \ww_* + \sqrt{1 - r^2} \ww_{\perp}' \right) \right] = \frac{1}{\oalpha} \xiss (r^2 + (1 - r^2) \langle \ww_{\perp}, \ww_{\perp}' \rangle).
\end{equation*}
Applying \cite[Theorem 1]{chen2017parisi} then implies that almost surely,
\begin{align*}
    \lim_{d \to \infty} \max_{\ww_{\perp} \in \S^{d-1}, \, \ww_{\perp} \perp \ww_* } \frac{1}{\sqrt{d}} H_{\sigma, \oalpha} \left( r \ww_* + \sqrt{1 - r^2} \ww_{\perp} \right) = \, \min_{\Gamma\in \cuG} \mathsf{U}_{\oalpha} (\Gamma, r),
\end{align*}
where
\begin{align*}
    & \mathsf{U}_{\oalpha} (\Gamma, r) := \, \frac{\Gamma(0)}{2 \oalpha} (1 - r^2) \xiss' (r^2) + \frac{(1-r^2)^2}{2 \oalpha} \int_{0}^{1} \Gamma(t) \xiss'' (r^2 + (1-r^2) t) \d t + \frac{1}{2} \int_{0}^{1} \frac{\d t}{\Gamma(t)}, \\
    & \cuG := \, \Big\{ \Gamma:[0,1) \to \R_{> 0}: \Gamma(t) = L - \int_0^t \alpha(s)\, \d s, \, \alpha \mbox{ non-decreasing, right-continuous, integrable} \Big\}.
\end{align*}
Define for $t \in [r^2, 1]$:
\begin{equation*}
    \gamma (t) = \, \frac{\oalpha}{1 - r^2} \Gamma \left( \frac{t - r^2}{1 - r^2} \right)^{-1},
\end{equation*}
we can rewrite $\mathsf{U}_{\oalpha}$ in terms of $\gamma$ and $r$:
\begin{align*}
    \mathsf{U}_{\oalpha} (\gamma, r) := \, & \frac{\xiss' (r^2)}{2 \gamma (r^2)} + \frac{1-r^2}{2} \int_{0}^{1} \frac{\xiss'' (r^2 + (1-r^2) t)}{\gamma (r^2 + (1-r^2) t)} \d t + \frac{1 - r^2}{2 \oalpha} \int_{0}^{1} \gamma (r^2 + (1-r^2) t) \d t \\
    = \, & \frac{\xiss' (r^2)}{2 \gamma (r^2)} + \frac{1}{2} \int_{r^2}^{1} \frac{\xiss'' (t)}{\gamma ( t)} \d t + \frac{1}{2 \oalpha} \int_{r^2}^{1} \gamma ( t) \d t.
\end{align*}
Further, by definition of $\cuG$, we know that $\gamma$ has the form $\gamma (t) = 1 / (c + \int_{t}^{1} \mu (s) \d s)$, where $c > 0$, $\mu$ is non-decreasing, right continuous, and integrable. We finally obtain the prediction
\begin{align*}
	\lim_{d \to \infty} \max_{\ww \in \S^{d-1}} \hCorr^g_d(\ww) = \, & \sup_{r \in [-1, 1]} \inf_{(\mu, c) \in \mathscr{U} \times \R_{> 0}} \left\{ \xifs (r) + \frac{\xiss' (r^2)}{2 \gamma (r^2)} + \frac{1}{2} \int_{r^2}^{1} \frac{\xiss'' (t)}{\gamma ( t)} \d t + \frac{1}{2 \oalpha} \int_{r^2}^{1} \gamma ( t) \d t \right\} \\
	\stackrel{(i)}{=} \, & \sup_{r \in [-1, 1]} \inf_{(\mu, c) \in \mathscr{U} \times \R_{> 0}} \left\{ \xifs (r) + \frac{\xiss' (1)}{2 \gamma (1)} + \frac{1}{2} \int_{r^2}^{1} \mu (t) \xiss' (t) \d t + \frac{1}{2 \oalpha} \int_{r^2}^{1} \gamma ( t) \d t \right\} \\
	\stackrel{(ii)}{=} \, & \sup_{r \in [-1, 1]} \inf_{(\mu, c) \in \mathscr{U} \times \R_{> 0}} \osF_{\oalpha} (\mu, c, r),
\end{align*}
where in $(i)$ we use integration by parts, and $(ii)$ follows from the definition of $\osF_{\oalpha}$ in \cref{eq:parisi_func_single_index}. This matches the Parisi formula for the single-index model derived in \cref{conj:replica_single_index}.

\subsubsection{Value of $\hCorr_d^g$ achieved by a two-stage algorithm}\label{sec:tensor_two_stage}

In this section, we describe and analyze a two-stage algorithm to optimize 
$\hCorr^g_d (\ww)$ over the unit sphere. Similar to our AMP 
algorithm for the single-index model, the first stage consists of 
several tensor AMP iterations. For the second stage, we use the 
Hessian ascent algorithm of \cite{subag2021following}.

Before presenting the actual tensor AMP algorithm, we need to introduce its state evolution,
 a set of recursive equations characterizing its asymptotic behavior 
  as $d \to \infty$. Let $\{ a_k (\ell) \}_{k \ge 1, \, \ell \ge 0}$ 
  be a sequence of real numbers. For $\ell, \ell_1, \ell_2 \ge 0$, define
\begin{equation*}
    \xiaf^{\ell} (t) = \, \sum_{k \ge 1} a_k (\ell) \varphi_k t^k, \quad \xiaa^{\ell_1, \ell_2} (t) = \, \sum_{k \ge 1} a_k (\ell_1) a_k (\ell_1) t^k.
\end{equation*}
Let $p_0$ be a probability distribution on $\R$ and $M^0 \sim p_0$. For each $\ell \ge 1$, let $M^{\le \ell} = (M^1, \cdots, M^{\ell}) \in \R^{\ell}$ be jointly Gaussian with zero mean, independent of $M^0$. Further, let $G \sim \normal (0, 1)$ be independent of $\{ M^{\ell} \}_{\ell \ge 0}$. The covariance matrix of $M^{\le \ell}$ is specified as follows:
\begin{enumerate}
    \item Define the sequence $\{ r_{\ell} \}_{\ell = 0}^{\infty}$ iteratively by
        \begin{equation*}
            r_0 = 0, \quad r_{\ell + 1} = ( \xiaf^{\ell} )' (r_{\ell}).
        \end{equation*}
        For each $\ell \ge 0$, define $Z^{\ell} = r_{\ell} G + M^{\ell}$.
    \item The covariances $Q_{\ell_1, \ell_2} = \E [M^{\ell_1} M^{\ell_2}]$ are defined by the recursion:
        \begin{equation*}
            Q_{\ell_1 + 1, \ell_2 + 1} = \, \frac{1}{\oalpha} (\xiaa^{\ell_1, \ell_2})' ( \E [ Z^{\ell_1} Z^{\ell_2} ] ) = \frac{1}{\oalpha} (\xiaa^{\ell_1, \ell_2})' ( r_{\ell_1} r_{\ell_2} + Q_{\ell_1, \ell_2} ), \quad \ell_1, \ell_2 \ge 0,
        \end{equation*}
        where $Q_{\ell_1, \ell_2} = 0$ if $\ell_1 = 0$ or $\ell_2 = 0$.
\end{enumerate}
Our tensor AMP algorithm takes the general form (cf. \cite{montanari2014statistical,el2021optimization}):
\begin{equation}\label{eq:tensor_AMP}
    \zz^{\ell+1} = \sum_{k\ge 1} ka_k(\ell)  \YY^{(k,s)}\{\zz^{\ell}\} - B_{\ell}\zz^{\ell-1}\, , \quad \ell \ge 0,
\end{equation}
where $\YY^{(k,s)}$ is the symmetrized version of $\YY^{(k)}$, the initialization $\zz^0 \in \R^d$ has i.i.d. coordinates drawn from $p_0 / \sqrt{d}$, and the Onsager terms $\{ B_{\ell} \}_{\ell \ge 0}$ are defined as follows: $B_0 = 0$, and for $\ell \ge 1$:
\begin{equation}\label{eq:def_B_ell}
    B_{\ell} = \, \frac{1}{\oalpha} (\xiaa^{\ell, \ell - 1})'' (r_{\ell} r_{\ell - 1} + Q_{\ell, \ell - 1}).
\end{equation}
We are now ready to characterize the limiting empirical joint distribution of the coordinates of tensor AMP iterations as $d \to \infty$.
\begin{prop}\label{prop:se_tensor_AMP}
    Assume that $p_0$ has finite second moment. For any $\ell \ge 0$ and pseudo-Lipschitz function $\psi: \R^{\ell + 1} \to \R$, we have
    \begin{equation*}
        \begin{split}
        & \plim_{d \to \infty} \frac{1}{d} \sum_{i=1}^{d} \psi \big( \sqrt{d} \zz_i^0, \sqrt{d} \zz_i^1, \cdots, \sqrt{d} \zz_i^{\ell} \big) = \, \E \big[ \psi \big( Z^0, Z^1, \cdots, Z^{\ell} \big) \big] \, .
        \end{split}
    \end{equation*}
\end{prop}
The proof of \cref{prop:se_tensor_AMP} is deferred to \cref{sec:proof_se_tensor_AMP}. We next characterize some important limiting quantities achieved by the first stage of our algorithm.
\begin{prop}\label{prop:AchievableTensorFirst}
Recall the definition of $A_{\sBayes}$ and $c_{\sBayes}$ from \cref{eq:A_Bayes_single_index,eq:C_Bayes_Perceptron}. Assume that $(r,q)\in A_{\sBayes}$, and $\{ a_k \}_{k \ge 1}$ satisfying
\begin{equation}\label{eq:Gaussian_fixed_pt}
    r = \sum_{k \ge 1} k a_k \varphi_k r^{k-1}, \quad q \ge r^2 + \frac{1}{\oalpha} \sum_{k \ge 1} k a_k^2 q^{k-1}, \quad \frac{1}{\oalpha} \sum_{k \ge 2} k(k-1) a_k^2 q^{k-2} \le 1,
\end{equation}
Then, there exists a tensor AMP algorithm (as described above) returning $\hww_1$, such that 
\begin{align*}
	&\plim_{d\to\infty} \big\< \hww_1, \ww_* \big\> =r\, ,\;\;\;\;\;\;\plim_{d\to\infty} \big\| \hww_1 \big\|^2_2=q\, ,\\
	&\plim_{d \to \infty} \hCorr^g_d \big( \hww_1 \big) = \, \xifs (r) + \frac{1}{\oalpha} \sum_{k \ge 1} k \sigma_k a_k q^{k-1}.
\end{align*}
\end{prop}
In the second stage, we run the Hessian ascent algorithm \cite{subag2021following} on the Hamiltonian
\begin{equation*}
    \ww_{\perp} \mapsto \hCorr^g_d \big( \hww_1 + \sqrt{1 - \| \hww_1 \|^2} \, \ww_{\perp} \big)
\end{equation*}
over $\ww_{\perp} \in \S^{d-1}$, $\ww_{\perp} \perp \hww_1$. Since the limiting empirical spectral distribution of $\nabla^2 \hCorr^g_d$, the spherical Hessian of $\hCorr^g_d$, only depends on its pure noise part, the same argument as in the proof of \cite[Theorem 4]{subag2021following} establishes that Hessian ascent achieves the value
\begin{equation*}
\begin{split}
    \hCorr^g_d \big( \hww_1 + \sqrt{1 - \| \hww_1 \|^2} \, \ww_{\perp} \big) \approx \, & \hCorr^g_d ( \hww_1 ) + \frac{1}{\sqrt{\oalpha}} \int_{0}^{1} \left( \frac{\d^2}{\d t^2} \xiss \big( \| \hww_1 \|^2 + (1 - \| \hww_1 \|^2) t \big) \right)^{1/2} \d t \\
    = \, & \hCorr^g_d (\hww_1) + \frac{1}{\sqrt{\oalpha}} \int_{\| \hww_1 \|^2}^{1} \xiss'' (t)^{1/2} \d t.
\end{split}
\end{equation*}
As $\| \hww_1 \|^2 \pto q$ as $d \to \infty$, we finally obtain the asymptotic value of $\hCorr^g_d$ achieved by our two-stage (tensor AMP + Hessian ascent) algorithm:
\begin{equation*}
    \plim_{d \to \infty} \hCorr^g_d \big( \hww_1 + \sqrt{1 - \| \hww_1 \|^2} \, \ww_{\perp} \big) = \, \xifs (r) + \frac{1}{\oalpha} \sum_{k \ge 1} k \sigma_k a_k q^{k-1} + \frac{1}{\sqrt{\oalpha}} \int_{q}^{1} \xiss'' (t)^{1/2} \d t.
\end{equation*}
To further analyze this value, define
\begin{equation*}
    T^g (r, q) = \, \max \, \frac{1}{\oalpha} \sum_{k \ge 1} k \sigma_k a_k q^{k-1}, \quad \mbox{subject to \cref{eq:Gaussian_fixed_pt}}. 
\end{equation*}
Then, straightforward calculation yields that
\begin{equation*}
	T^g (r, q) = \, \sqrt{\frac{q - r^2}{\oalpha}} \, T_0 (r, q),
\end{equation*}
with $T_0 (r, q)$ defined in \cref{eq:T0_def}. Consequently, for any fixed $(r, q) \in A_{\sBayes}$, the asymptotic maximum of $\hCorr_d^g$ attained by our algorithm is exactly $\VH_{\oalpha}^{\sAMP} (r, q)$, which matches the result established for the single-index model.

\subsection{Proof of \cref{lem:tensor_pca_equivalence}, \cref{prop:se_tensor_AMP} and \ref{prop:AchievableTensorFirst}}\label{sec:proof_se_tensor_AMP}

\begin{proof}[Proof of \cref{lem:tensor_pca_equivalence}]
Obviously, $\hCorr^g_d$ is a Gaussian process. We now verify \cref{eq:GaussianModel}. By definition,
\begin{align*}
    \E \left[ \hCorr^g_d(\ww) \right] = \, & \sum_{k \ge 1} \sigma_k \varphi_k \langle \ww_*^{\otimes k}, \ww^{\otimes k} \rangle = \xifs (\langle \ww_*, \ww \rangle), \\
    \Cov\big(\hCorr^g_d(\ww_1), \hCorr^g_d(\ww_2)\big) = \, & \frac{1}{d \oalpha} \Cov \left( \sum_{k \ge 1} \sigma_k \langle \GG^{\otimes k}, \ww_1^{\otimes k} \rangle, \sum_{k \ge 1} \sigma_k \langle \GG^{\otimes k}, \ww_2^{\otimes k} \rangle \right) \\
    = \, & \frac{1}{d \oalpha} \sum_{k \ge 1} \sigma_k^2 \langle \ww_1^{\otimes k}, \ww_2^{\otimes k} \rangle = \frac{1}{d \oalpha} \xiss (\langle \ww_1, \ww_2 \rangle),
\end{align*}
where in the above calculation we use \cref{eq:xi_expansion} and the assumption $\sigma_0 = \E [\sigma (G)] = 0$. This completes the proof.
\end{proof}

\begin{proof}[Proof of \cref{prop:se_tensor_AMP}]
This proof is based on the state evolution of general tensor AMP algorithms without a signal, which we establish below. As in \cref{sec:tensor_two_stage}, we begin with introducing the Gaussian process that governs the limiting behavior of tensor AMP iterates as $d \to \infty$.

Let $\{ \sigma_k \}_{k \ge 1}$ be a sequence of real numbers such that $\sum_{k \ge 1} \sigma_k^2 t^k < + \infty$ for some $t > 1$. Let $\{ f_{\ell}^{k} \}_{\ell \ge 0, \, k \ge 1}$ be a sequence of functions such that for each $\ell \ge 0$ and $k \ge 1$, $f_{\ell}^k: \R^{\ell + 1} \to \R$ is Lipschitz. Let $U^0 \sim p_0$, where $p_0$ is a probability distribution on $\R$ with finite second moment. For each $\ell \ge 1$, $U^{\le \ell} = (U^1, \cdots, U^{\ell}) \in \R^{\ell}$ is a mean-zero multivariate Gaussian vector, independent of $U^0$. Further, the covariance structure of $U^{\le \ell}$ is recursively defined via:
\begin{equation*}
\begin{split}
    \E \left[ U^{\ell_1 + 1} U^{\ell_2 + 1} \right] = \, & \sum_{k \ge 1} k \sigma_k^2 \, \E \left[ f_{\ell_1}^k \left( U^0, \cdots, U^{\ell_1} \right) f_{\ell_2}^k \left( U^0, \cdots, U^{\ell_2} \right) \right]^{k-1}, \quad \forall \ell_1, \ell_2 \ge 0, \\
    \E \left[ U^{\ell_1} U^{\ell_2} \right] = \, & 0, \quad \mbox{ if } \ell_1 = 0 \mbox{ or } \ell_2 = 0.
\end{split}
\end{equation*}
For each $k \ge 1$, let $\GG^{(k)} = (G^{(k)}_{i_1, \dots, i_k})_{1 \le i_1, \dots, i_k \le d} \in (\R^d)^{\otimes k}$ be a tensor with i.i.d. $\normal (0, 1)$ entries, and let $\GG^{(k, s)}$ be the symmetrization of $\GG^{(k)}$. Consider the following general tensor AMP iterations:
\begin{equation}\label{eq:noise_tensor_AMP_def}
    \uu^{\ell + 1} = \, \frac{1}{\sqrt{d}} \sum_{k \ge 1} k \sigma_k \GG^{(k, s)} \{ f_{\ell}^k (\uu^0, \cdots, \uu^{\ell}) \} - \sum_{k \ge 1} \sum_{j \le \ell} B_{\ell, j, k} f_{j-1}^{k} (\uu^0, \dots, \uu^{j-1}), \quad \ell \ge 0,
\end{equation}
with $\uu^0 \in \R^d$ having coordinates i.i.d. sampled from $p_0$, and the Onsager terms defined as
\begin{equation*}
    B_{\ell, j, k} = \, k (k-1) \sigma_k^2 \, \E [ f_{\ell}^k (U^0, \cdots, U^{\ell}) f_{j-1}^k (U^0, \cdots, U^{j-1}) ]^{k-2} \cdot \E \left[ \frac{\partial f_{\ell}^k}{\partial U^j} (U^0, \cdots, U^{\ell}) \right].
\end{equation*}
\begin{prop}\label{prop:se_noise_tensor_AMP}
    For any $\ell \ge 0$ and pseudo-Lipschitz function $\psi: \R^{\ell + 1} \to \R$, we have
    \begin{equation*}
        \plim_{d \to \infty} \frac{1}{d} \sum_{i=1}^{d} \psi \left( \sqrt{d} \uu_i^0, \cdots, \sqrt{d} \uu_i^{\ell} \right) = \, \E \left[ \psi \left( U^0, \cdots, U^{\ell} \right) \right].
    \end{equation*}
\end{prop}
The proof of \cref{prop:se_noise_tensor_AMP} closely follows that of \cite[Theorem 6]{el2021optimization}, and will be presented in \cref{sec:proof_se_noise_tensor_AMP}. 

We are now in position to complete the proof of \cref{prop:se_tensor_AMP}. In a similar spirit to \cref{lem:Reduction_Signal}, we can assume without loss of generality that the empirical distribution of the coordinates of $\sqrt{d} \ww_*$ converges in $W_2$ distance to $\normal (0, 1)$ as $d \to \infty$. Define for $\ell \ge 0$ and $k \ge 1$:
\begin{equation*}
\begin{split}
    f_{\ell}^k (\uu^0, \cdots, \uu^{\ell}) = \, & \left( \frac{a_k (\ell)}{ \sqrt{\oalpha} \sigma_k} \right)^{1/(k-1)} \left( \uu^{\ell} + r_{\ell} \cdot \ww_* \right) \\
    = \, & \left( \frac{a_k (\ell)}{ \sqrt{\oalpha} \sigma_k} \right)^{1/(k-1)} \left( \uu^{\ell} + (\xiaf^{\ell-1})' ( r_{\ell-1} ) \cdot \ww_* \right),
\end{split}
\end{equation*}
and consider the AMP algorithm~\eqref{eq:noise_tensor_AMP_def} (the base coefficients $\{ \sigma_k \}_{k \ge 1}$ may depend on $\ell$ and can be chosen such that $a_k (\ell) / \sigma_k > 0$). The state evolution equations then reduce to
\begin{align*}
    \E \left[ U^{\ell_1 + 1} U^{\ell_2 + 1} \right] = \, & \frac{1}{\oalpha} \sum_{k \ge 1} k a_k (\ell_1) a_k (\ell_2) \, \left( r_{\ell_1} r_{\ell_2} + \E [U^{\ell_1} U^{\ell_2}] \right)^{k-1} \\
    = \, & \frac{1}{\oalpha} (\xiaa^{\ell_1, \ell_2})' ( r_{\ell_1} r_{\ell_2} + \E [U^{\ell_1} U^{\ell_2}] ),
\end{align*}
which implies that $\E [U^{\ell_1} U^{\ell_2}] = Q_{\ell_1, \ell_2}$. The AMP iterations can also be simplified to
\begin{equation*}
    \uu^{\ell + 1} = \, \frac{1}{\sqrt{d \oalpha}} \sum_{k \ge 1} k a_k (\ell) \GG^{(k, s)} \{ \uu^{\ell} + r_{\ell}  \ww_* \} - B_{\ell} (\uu^{\ell-1} + r_{\ell-1} \ww_*),
\end{equation*}
where we recall $B_{\ell}$ from \cref{eq:def_B_ell}.

Define $\bm^{\ell} = \, \zz^{\ell} - r_{\ell} \ww_*$. To prove our claim, it suffices to show that $\{ \bm^{\ell} \}_{\ell \ge 0}$ and $\{ \uu^{\ell} \}_{\ell \ge 0}$ have the same limiting empirical distribution. To this end, we use induction to establish that
\begin{equation}\label{eq:equiv_u_and_m}
    \plim_{d \to \infty} \, \norm{\uu^{\ell} - \bm^{\ell}}_2^2 = \, \plim_{d \to \infty} \, \sum_{i=1}^d \left( \uu_i^{\ell} - \bm_{i}^{\ell} \right)^2 = 0, \quad \forall \ell \ge 0.  
\end{equation}
The base case $\ell = 0$ holds automatically since we can choose $\uu^0 = \zz^0 = \bm^0$. Now assume that \cref{eq:equiv_u_and_m} holds for all $\ell' \le \ell$. For $\ell + 1$, note that \cref{eq:tensor_AMP} implies
\begin{align*}
    \zz^{\ell + 1} = \, (\xiaf^{\ell})' (\langle \zz^{\ell}, \ww_* \rangle) \cdot \ww_* + \frac{1}{\sqrt{d \oalpha}} \sum_{k \ge 1} k a_k (\ell) \GG^{(k, s)} \{ \zz^{\ell} \} - B_{\ell} \zz^{\ell - 1}.
\end{align*}
Our induction assumption implies that
\begin{align*}
    \plim_{d \to \infty} \, \langle \zz^{\ell}, \ww_* \rangle = \, r_{\ell}, \quad \plim_{d \to \infty} \, \norm{\zz^{\ell'} - \uu^{\ell'} - r_{\ell'} \ww_*}_2^2 = \, 0, \,\, \forall \ell' \le \ell.
\end{align*}
Therefore, it remains to show that
\begin{equation*}
    \plim_{d \to \infty} \, \frac{1}{d} \norm{\GG^{(k, s)} \{ \zz^{\ell} \} - \GG^{(k, s)} \{ \uu^{\ell} \} }_2^2 = \, 0, \quad \forall k \ge 1,
\end{equation*}
which easily follows from concentration bounds on the operator norm of Gaussian random tensors, i.e., $\| \GG^{(k, s)} \|_{\op} = O_p (\sqrt{d})$ (cf. \cite[Lemma 2]{montanari2014statistical}). This completes our induction and the proof of \cref{prop:se_tensor_AMP}.
\end{proof}

\begin{proof}[Proof of Proposition \ref{prop:AchievableTensorFirst}]
For $\ell \ge 0$, define $q_{\ell} = r_{\ell}^2 + Q_{\ell, \ell}$. Then, we know that $(r_0, q_0) = 0$ and state evolution yields
\begin{equation}\label{eq:r_q_tensor_evo}
    r_{\ell + 1} = \, ( \xiaf^{\ell} )' (r_{\ell}), \quad q_{\ell + 1} = \, r_{\ell + 1}^2 + \frac{1}{\oalpha} (\xiaa^{\ell, \ell})' ( q_{\ell} ).
\end{equation}
Let $A_{\rm pca}$ denote the set of pairs $(r, q) \in [-1, 1] \times [0, 1]$ that can be achieved as limit points of the above iteration. Since $\xiaf^{\ell}$ is linear in the $a_k (\ell)$'s and $\xiaa^{\ell, \ell}$ is quadratic in the $a_k (\ell)$'s, we see that $A_{\rm pca}$ must be of the form $A_{\rm pca} = \{ (r,q): r^2/q \le c_{\rm pca}^2 / ( 1 + c_{\rm pca}^2 ) \}$ for some $c_{\rm pca} > 0$. To compute $c_{\rm pca}$, denote $c_{\ell} = r_{\ell} / \sqrt{q_{\ell} - r_{\ell}^2}$. For any fixed $(r_{\ell}, q_{\ell})$, \cref{eq:r_q_tensor_evo} implies that
\begin{align*}
    c_{\ell+1}^2 = \, \frac{r_{\ell+1}^2}{q_{\ell+1} - r_{\ell+1}^2} = \oalpha \frac{( \xiaf^{\ell} )' (r_{\ell})^2}{(\xiaa^{\ell, \ell})' ( q_{\ell} )} \stackrel{(i)}{\le} \oalpha \xiff' \left( \frac{r_{\ell}^2}{q_{\ell}} \right) = \oalpha \xiff' \left( \frac{c_{\ell}^2}{1 + c_{\ell}^2} \right),
\end{align*}
where in $(i)$ we use Cauchy-Schwarz inequality (cf. \cref{lem:xi_cauchy_schwarz}), and the equality is achieved if
\begin{equation*}
    a_k (\ell) \propto \varphi_k \left( \frac{r_{\ell}}{q_{\ell}} \right)^{k-1}.
\end{equation*}
This immediately implies that $c_{\rm pca}$ should be the minimum positive solution to the equation
\begin{equation*}
    c^2 = \, \oalpha \xiff' \left( \frac{c^2}{1 + c^2} \right),
\end{equation*}
leading to $c_{\rm pca} = c_{\sBayes}$ and $A_{\rm pca} = A_{\sBayes}$.

Now for any $(r, q) \in \operatorname{int} A_{\sBayes}$, there exists $\ell_1 \in \mathbb{N}$ such that after $\ell_1$ tensor AMP iterations, we achieve $(r_{\ell_1}, q_{\ell_1}) = (r, q)$. By state evolution of tensor AMP, for any sequence $\{ a_k \}_{k \ge 1}$ satisfying
\begin{equation*}
    r = \xiaf' (r), \quad q = r^2 + \frac{1}{\oalpha} \xiaa' (q)
\end{equation*}
(where $\xiaf(t):=\sum_{k\ge 1}a_k\varphi_k t^k$ and $\xiaa(t):=\sum_{k\ge 1}a_k^2 t^k$), using $a_k (\ell) = a_k$ yields $(r_{\ell}, q_{\ell}) = (r, q)$ for all iterations $\ell_1 \le \ell \le \ell_2 - 1$, where $\ell_2 \ge \ell_1 + 1$ is to be determined later. Further, the covariance between successive AMP iterations satisfies
\begin{equation*}
    Q_{\ell + 1, \ell} = \, \frac{1}{\oalpha} \xiaa' (r^2 + Q_{\ell, \ell - 1}).
\end{equation*}
Therefore, as long as $\xiaa'' (q) / \oalpha \le 1$, we have $Q_{\ell_2, \ell_2 - 1} \to q - r^2$ as $\ell_2 \to \infty$, which further implies that $\E [(Z^{\ell_2} - Z^{\ell_2 - 1})^2] \to 0$. In fact, by adding random noise to the $a_k$'s, it suffices to assume that
\begin{equation*}
	r = \xiaf' (r), \quad q \ge r^2 + \frac{1}{\oalpha} \xiaa' (q), \quad \frac{1}{\oalpha} \xiaa'' (q) \le 1,
\end{equation*}
which is equivalent to \cref{eq:Gaussian_fixed_pt}.

We are now in position to compute the asymptotic value of $\hCorr^g_d$ achieved in the first stage. Fix $(r, q) \in A_{\sBayes}$, and let $\{ a_k \}_{k \ge 1}$ satisfy \cref{eq:Gaussian_fixed_pt}.
We compute
\begin{align*}
    \hCorr^g_d (\zz^{\ell_2}) = \, \sum_{k \ge 1} \sigma_k \langle \YY^{(k)}, (\zz^{\ell_2})^{\otimes k} \rangle = \sum_{k \ge 1} \sigma_k \langle \YY^{(k, s)}, (\zz^{\ell_2})^{\otimes k} \rangle = \Big\langle \zz^{\ell_2}, \sum_{k \ge 1} \sigma_k \YY^{(k, s)} \{ \zz^{\ell_2} \} \Big\rangle.
\end{align*}
To this end, we ``do'' another step of tensor AMP:
\begin{equation*}
    \zz^{\ell_2+1} = \, \sum_{k \ge 1} \sigma_k \YY^{(k, s)} \{ \zz^{\ell_2} \} - \Big( \frac{1}{\oalpha} \sum_{k \ge 1} (k-1) \sigma_k a_k \E [Z^{\ell_2} Z^{\ell_2 - 1}]^{k-2} \Big) \cdot \zz^{\ell_2 - 1},
\end{equation*}
which is just \cref{eq:tensor_AMP} with $a_k (\ell_2) = \sigma_k / k$, and we note that $a_k (\ell_2 - 1) = a_k$ by our choice. It then follows that
\begin{align*}
    \hCorr^g_d (\zz^{\ell_2}) = \, & \Big\langle \zz^{\ell_2}, \sum_{k \ge 1} \sigma_k \YY^{(k, s)} \{ \zz^{\ell_2} \} \Big\rangle = \langle \zz^{\ell_2}, \zz^{\ell_2 + 1} \rangle + \langle \zz^{\ell_2 - 1}, \zz^{\ell_2} \rangle \cdot \frac{1}{\oalpha} \sum_{k \ge 1} (k-1) \sigma_k a_k \E [Z^{\ell_2} Z^{\ell_2 - 1}]^{k-2} \\
    \pto \, & \E [Z^{\ell_2 + 1} Z^{\ell_2}] + \frac{1}{\oalpha} \sum_{k \ge 1} (k-1) \sigma_k a_k \E [Z^{\ell_2} Z^{\ell_2 - 1}]^{k-1} \,\, \mbox{ as } d \to \infty.
\end{align*}
Since $\E [(Z^{\ell_2} - Z^{\ell_2 - 1})^2] \to 0$ as $\ell_2 \to \infty$, we know that $\E [Z^{\ell_2} Z^{\ell_2 - 1}] \to q$. Further,
\begin{align*}
    r_{\ell_2 + 1} = \, & ( \xiaf^{\ell_2} )' (r_{\ell_2}) = \sum_{k \ge 1} k a_k (\ell_2) \varphi_k r_{\ell_2}^{k-1} = \sum_{k \ge 1} \sigma_k \varphi_k r_{\ell_2}^{k-1} = \frac{1}{r_{\ell_2}} \xifs (r_{\ell_2}), \\
    Q_{\ell_2 + 1, \ell_2} = \, & \frac{1}{\oalpha} (\xiaa^{\ell_2, \ell_2 - 1})' ( r_{\ell_2} r_{\ell_2 - 1} + Q_{\ell_2, \ell_2 - 1} ) = \frac{1}{\oalpha} \sum_{k \ge 1} k a_k (\ell_2) a_k (\ell_2 - 1) \E [Z^{\ell_2} Z^{\ell_2 - 1}]^{k-1} \\
    = \, & \frac{1}{\oalpha} \sum_{k \ge 1} \sigma_k a_k \, \E [Z^{\ell_2} Z^{\ell_2 - 1}]^{k-1},
\end{align*}
leading to
\begin{equation*}
\begin{split}
    \E [Z^{\ell_2 + 1} Z^{\ell_2}] = \, & r_{\ell_2 + 1} r_{\ell_2} + Q_{\ell_2 + 1, \ell_2} = \xifs (r_{\ell_2}) + \frac{1}{\oalpha} \sum_{k \ge 1} \sigma_k a_k \, \E [Z^{\ell_2} Z^{\ell_2 - 1}]^{k-1} \\
    \to \, & \xifs (r) + \frac{1}{\oalpha} \sum_{k \ge 1} \sigma_k a_k q^{k-1} \,\, \mbox{ as } \ell_2 \to \infty.
\end{split}
\end{equation*}
We finally deduce that
\begin{equation*}
    \lim_{\ell_2 \to \infty} \plim_{d \to \infty} \hCorr^g_d (\zz^{\ell_2}) = \, \xifs (r) + \frac{1}{\oalpha} \sum_{k \ge 1} k \sigma_k a_k q^{k-1},
\end{equation*}
where $\{ a_k \}_{k \ge 1}$ satisfy \cref{eq:Gaussian_fixed_pt}. This completes the proof.
\end{proof}

%
%

\subsection{Proof of \cref{prop:se_noise_tensor_AMP}}\label{sec:proof_se_noise_tensor_AMP}
\begin{proof}
The tensor AMP algorithm~\eqref{eq:noise_tensor_AMP_def} in \cref{sec:proof_se_tensor_AMP} generalizes the one introduced in \cite[Appendix A]{el2021optimization}, by allowing for different non-linearities $f_{\ell}^k$ for different tensor orders $k$ at each iteration $\ell$. Since the proof of state evolution in \cite{el2021optimization} separately tracks the AMP iterates for different $k$, their argument extends naturally to our setting. Therefore, instead of presenting a new proof here, we describe below how to adapt the proof of their state evolution result \cite[Theorem 6]{el2021optimization} to establish our \cref{prop:se_noise_tensor_AMP}.

The proof of \cite[Theorem 6]{el2021optimization} proceeds in three main steps: (i) applying a Gaussian conditioning lemma \cite[Lemma A.2]{el2021optimization} to characterize the conditional distribution of the Gaussian tensors given their linear measurements from prior AMP iterations; (ii) deriving the state evolution for ``Long AMP'' (LAMP), a variant of tensor AMP that utilizes Gaussian tensors projected onto the orthogonal complement of the subspace spanned by previous iterates; and (iii) establishing the asymptotic equivalence between the standard tensor AMP and LAMP via an approximation argument. In what follows, we detail how to adapt each of these steps to our setting. We omit the details regarding the reduction to the well-conditioned case and the extension to $D = \infty$ (Appendices A.7 and A.8 in \cite{el2021optimization}), as these arguments apply directly to our setting without modification.

\paragraph{Gaussian conditioning lemma.}
In this section, we adapt the statement and proof of \cite[Lemma A.2]{el2021optimization} to our setting. Introducing the notations for $\ell \ge 0$, $k \ge 1$:
\begin{align*}
    \ff_{\ell}^k = \, f_{\ell}^k (\uu^0, \cdots, \uu^{\ell}), \quad \vv_{\ell}^{k} = \, \frac{1}{\sqrt{d}} k \sigma_k \GG^{(k, s)} \{ \ff_{\ell}^k \} - \sum_{j \le \ell} B_{\ell, j, k} \ff_{j-1}^{k},
\end{align*}
the key conclusion of this lemma is a characterization of the conditional distribution of the Gaussian random tensors $\{ \GG^{(k, s)} \}_{k \ge 1}$, given the $\sigma$-algebra $\cF_t = \sigma ( \{ \uu^{\ell}, \ff_{\ell}^k, \vv_{\ell}^k \}_{\ell \le t, \, k \le D} )$ generated by AMP iterations up to time $t$. More accurately, we aim to derive an expression for $\E [ \GG^{(k, s)} \vert \cF_t ]$ analogous to Eq. (A.19) in \cite{el2021optimization}.

Following the proof of \cite[Lemma A.2]{el2021optimization}, we know that $\E [ \GG^{(k, s)} \vert \cF_t ]$ is just the conditional expectation of $\GG^{(k, s)}$ given linear measurements:
\begin{equation*}
    \GG^{(k, s)} \{ \ff_{\ell}^k \} = \yy_{\ell}^k, \quad \ell \le t.
\end{equation*}
Using the method of Lagrange multipliers, one can write
\begin{equation*}
    \E [ \GG^{(k, s)} \vert \cF_t ] = \, \hat{\GG}^{(k, s)} := \sum_{\ell=0}^{t} \sum_{j=1}^{k} (\ff_{\ell}^k)^{\otimes (j-1)} \otimes \bm_{\ell} \otimes (\ff_{\ell}^k)^{\otimes (k-j)},
\end{equation*}
where $\{ \bm_{\ell} \}_{\ell \le t}$ are vectors in $\R^d$. Further, $\hat{\GG}^{(k, s)}$ must also satisfy the same set of linear constraints:
\begin{equation*}
    \hat{\GG}^{(k, s)} \{ \ff_{\ell}^k \} = \yy_{\ell}^k, \quad \ell \le t.
\end{equation*}
The vectors $\{ \bm_{\ell} \}_{\ell \le t} \subset \R^d$ can then be solved from the above linear constraints, in a similar way as in the proof of \cite[Lemma A.2]{el2021optimization}, completing the proof of the Gaussian conditioning lemma.

\paragraph{State evolution of LAMP.}
We now work on adapting the proof of LAMP state evolution to our setting. Similar to the definitions in Appendix A.3 of \cite{el2021optimization}, in LAMP we will use $\mathcal{P}_{t, k}^{\perp} \GG^{(k, s)}$ instead of $\GG^{(k, s)}$ in the $t$-th iteration, where $\mathcal{P}_{t, k}^{\perp}$ denotes the orthogonal projection onto the linear subspace of all tensors $\boldsymbol{T}$ satisfying $\boldsymbol{T} \{ \ff_{s}^k \} = 0$ for all $s < t$. The Onsager correction terms are now defined via the Gram matrices of the $\ff_{s}^k$'s.

To establish state evolution for LAMP, we follow the proof strategy of \cite[Theorem 7]{el2021optimization}. Part $(a)$ can be proved in exactly the same way by exploiting the independence between $\mathcal{P}_{t, k}^{\perp} \GG^{(k, s)}$ and $\cF_t$. Further, parts $(b)$ and $(c)$ directly follow from the conditional law and expectation \cite[Eq.s (A.32) and (A.33)]{el2021optimization} in part $(a)$, by replacing the $\ff_s$'s with the $\ff_s^k$'s in the original proof \cite[Appendices A.5.2 and A.5.3]{el2021optimization}.

\paragraph{Asymptotic equivalence between AMP and LAMP.}
This asymptotic equivalence is established in Appendix A.6 of \cite{el2021optimization} for the  case $\ff_s^k = \ff_s$ for all $k$. In fact, a stronger result is proven there: not only are the AMP and LAMP iterates asymptotically equivalent, but this equivalence holds individually for the iterates corresponding to each tensor order $k$. Examining their proof, we observe that it generalizes straightforwardly to our setting by simply replacing the $\ff_s$'s with the $\ff_s^k$'s.
\end{proof}

\subsection{Auxiliary lemmas}
\begin{lem}\label{lem:q_star_r}
    Recall the definition of $q_* (r)$ from \cref{eq:def_q_star}. For any $r \in [-1, 1]$, we have
    \begin{itemize}
        \item [(a)] If $\xiss' (1) > \xiss'' (1) (1 - r^2)$, then $q_* (r) = 1$.
        \item [(b)] If $\xiss' (1) \le \xiss'' (1) (1 - r^2)$, then $q_*(r)$ is the unique solution to the equation
            \begin{equation*}
                \xiss' (q) = \, \xiss'' (q) (q - r^2), \quad q \in [r^2, 1].
            \end{equation*}
    \end{itemize}
\end{lem}
\begin{proof}
Note that
    \begin{equation*}
        \frac{\d }{\d q} \left( \frac{\xiss'(q)}{q-r^2} \right) = \, \frac{\xiss''(q)}{(q - r^2)^2} \left( q - \frac{\xiss'(q)}{\xiss''(q)} - r^2 \right).
    \end{equation*}
    It is easy to show that the mapping $q \mapsto q - \xiss'(q) / \xiss''(q)$ is monotone increasing. Therefore, in case $(a)$ we have $1 - \xiss'(1) / \xiss'' (1) < r^2$, hence $q - \xiss'(q) / \xiss'' (q) < r^2$ for all $q \in [r^2, 1]$, which implies that $q_* (r) = 1$. Similarly, in case $(b)$ we know that the equation $\xiss' (q) = \, \xiss'' (q) (q - r^2)$ has a unique solution on $[r^2, 1]$, and $q_* (r)$ must be this solution.	
\end{proof}
\begin{lem}\label{lem:xi_cauchy_schwarz}
For any $t,s\in [-1,1]$, we have
\begin{align*}
\xifs(ts)\le \sqrt{\xiss(s^2) \cdot\xiff(t^2)}\, ,
\end{align*}
with equality if and only if there exists $b\in \R$ such that $\sigma_ks^k=b\varphi_k t^k$ for all $k \in \mathbb{N}$.
\end{lem}
\begin{proof}
	Follows from directly applying the Cauchy-Schwarz inequality.
\end{proof}
\begin{lem}\label{lem:single_index_FOC}
	Assume one of the following:
	\begin{itemize}
		\item [(i)] For any $\oalpha > 0$, the equation  $x = \oalpha (1 - x) \xiff' (x)$ has a unique solution over $[0, 1]$.
		
		\item [(ii)] $\sigma$ satisfies $\xiss'' (1) \le \xiss' (1)$.
	\end{itemize}
	Then, the optimality condition~\eqref{eq:optimality_single_index} holds for all $\oalpha > 0$.
\end{lem}
\begin{proof}
	Recall $r_A$ and $r_B$ defined in the proof of \cref{thm:single_index_duality} and \cref{prop:optimality_large_oalpha}. We already showed that $\xiss' (1) \ge \, \xiss'' (1) (1 - r_B^2)$ implies $r_A = r_B$. Under our condition (ii), this is true for all $\oalpha > 0$, which implies \cref{eq:optimality_single_index}.
	
	We next assume condition (i), and let
	\begin{equation*}
		r_*' = \, \arg \max_{r : (r, 1) \in A_{\sBayes}} \VH_{\oalpha}^*(r,q_*(r)).
	\end{equation*}
	Then, from the proof of \cref{thm:single_index_duality} we know that
	\begin{equation*}
		\frac{\de \VH_{\oalpha}^*}{\de r} (r, q_*(r)) = \, \xifs' (r) - \sqrt{\frac{\xiss' (q_*(r))}{\oalpha}} \frac{r}{\sqrt{q_*(r) - r^2}}.
	\end{equation*}
	By our condition (i), we know that the fixed point equation
	\begin{equation*}
		c^2 = \, \oalpha \xiff' \left( \frac{c^2}{1 + c^2} \right)
	\end{equation*}
	has a unique solution $c = c_{\sBayes}$. As a consequence, for any $r \in [r_A, r_B] \cup [-r_B, -r_A]$, we have
	\begin{equation*}
		c^2 := \frac{r^2}{q_* (r) - r^2} \ge c_{\sBayes}^2 \implies c^2 \ge \oalpha \xiff' \left( \frac{c^2}{1 + c^2} \right).  
	\end{equation*}
	Similar to the proof of \cref{thm:single_index_duality}, this implies that
	\begin{equation*}
		\frac{\de \VH_{\oalpha}^*}{\de r} (r, q_*(r)) \le 0, \,\, r \in [r_A, r_B], \quad \frac{\de \VH_{\oalpha}^*}{\de r} (r, q_*(r)) \ge 0, \,\, r \in [-r_B, -r_A],
	\end{equation*}
	leading to \cref{eq:optimality_single_index}. This completes the proof.
\end{proof}

\section{Appendix for Section~\ref{sec:iamp}}\label{sec:proof_iamp}

\subsection{Proof of \cref{prop:Bayes_feasible_region} and \cref{lem:cov_mu_Q}}

\begin{proof}[Proof of \cref{prop:Bayes_feasible_region}]
    Note that \cref{eq:covariance_stage_1} implies the following recursive relation between $(R_t, Q_t)$ and $(R_{t+1}, Q_{t+1})$:
	\begin{align*}
		Q_{t+1} = \, & R_{t+1}^\sT R_{t+1} + \frac{1}{\alpha} \E_{(R_t, Q_t)} \left[ F_{t} \left(\overline{Z}_{t} , Y\right)^\sT F_{t} \left(\overline{Z}_{t} , Y \right) \right], \\
		R_{t+1} = \, & \mathbb{E}_{(R_t, Q_t)} \left[ \frac{\partial F_t}{\partial \bar{z}_0} \left(\overline{Z}_{t} , Y \right) \right], \quad Y = \varphi \left( \bar{Z}_0, \veps \right),
	\end{align*}
	where $\E_{(R_t, Q_t)}$ represents the expectation taken under $(\bar{Z}_t, \bar{Z}_0)^\sT \sim \normal \left( 0, \begin{bmatrix}
		Q_t & R_t^\sT \\
		R_t & I_k 
		\end{bmatrix} \right)$, independent of $\veps \sim P_{\veps}$.
	Applying Stein's lemma, we obtain that
	\begin{align*}
		& R_{t+1} = \, \mathbb{E}_{(R_t, Q_t)} \left[ \frac{\partial F_t}{\partial \bar{z}_0} \left(\overline{Z}_{t} , \varphi \left( \bar{Z}_0, \veps \right) \right) \right] \\
		= \, & \E_{(R_t, Q_t)} \left[ \left( I_k - R_t Q_t^{-1} R_t^\sT \right)^{-1} \left( \bar{Z}_0 - \bar{Z}_t Q_t^{-1} R_t^\sT \right)^\sT F_t \left(\overline{Z}_{t} , \varphi \left( \bar{Z}_0, \veps \right) \right) \right].
	\end{align*}
	Denoting $H_t = \left( \bar{Z}_0 - \bar{Z}_t Q_t^{-1} R_t^\sT \right) \left( I_k - R_t Q_t^{-1} R_t^\sT \right)^{-1} \in \R^{1 \times k}$, we can write
	\begin{equation*}
		R_{t+1} = \, \E_{(R_t, Q_t)} \left[ H_t^\sT F_t \right] = \E_{(R_t, Q_t)} \left[ \E \left[ H_t \vert \bar{Z}_t, Y \right]^\sT F_t \right],
	\end{equation*}
	since $F_t$ is only a function of $\bar{Z}_t$ and $Y$. Applying \cref{lem:matrix_cauchy}, we know that for fixed $(R_t, Q_t)$, the range of all possible $(R_{t+1}, Q_{t+1})$ is the set of all $(R, Q) \in \R^{k \times m} \times \Sym_+^m$ satisfying
	\begin{equation*}
		R^\sT R \preceq Q, \quad R (Q - R^\sT R)^{-1} R^\sT \preceq \alpha \E_{(R_t, Q_t)} \left[ \E \left[ H_t \vert \bar{Z}_t, Y \right]^\sT \E \left[ H_t \vert \bar{Z}_t, Y \right] \right].
	\end{equation*}
	We next show that the right-hand side of the above inequality only depends on
	\begin{equation*}
		R_t \left( Q_t - R_t^\sT R_t \right)^{-1} R_t^\sT.
	\end{equation*}
	For future convenience, define for all $t \in \mathbb{N}$:
	\begin{equation*}
		M_t = R_t \left( Q_t - R_t^\sT R_t \right)^{-1/2} \in \R^{k \times m},
	\end{equation*}
	and
	\begin{equation*}
		G = \bar{Z}_0^\sT \in \R^k, \, Z = \left( Q_t - R_t^\sT R_t \right)^{-1/2} \left( \bar{Z}_t - \bar{Z}_0 R_t \right)^\sT \in \R^m.
	\end{equation*}
	Note that $G$ and $Z$ are column vectors, and that the joint distribution of $(Y, G, Z)$ is as described in \cref{eq:Triple_distribution_general} of \cref{defn:Bayes_AMP_region_general}. Further, one can write
	\begin{align*}
		\bar{Z}_t^\sT = \, R_t^\sT G + \left( Q_t - R_t^\sT R_t \right)^{1/2} Z = \left( Q_t - R_t^\sT R_t \right)^{1/2} \left( Z + M_t^\sT G \right),
	\end{align*}
	and
	\begin{align*}
		H_t^\sT = \, & \left( I_k - R_t Q_t^{-1} R_t^\sT \right)^{-1} \left( \bar{Z}_0^\sT - R_t Q_t^{-1} \bar{Z}_t^\sT \right) \\
		= \, & \left( I_k - R_t Q_t^{-1} R_t^\sT \right)^{-1} \left( G - R_t Q_t^{-1} \bar{Z}_t^\sT \right) \\
		= \, & G - \left( I_k - R_t Q_t^{-1} R_t^\sT \right)^{-1} R_t Q_t^{-1} \left( Q_t - R_t^\sT R_t \right)^{1/2} Z \\
		\stackrel{(i)}{=} \, & G - R_t \left( Q_t - R_t^\sT R_t \right)^{-1/2} Z = G - M_t Z,
	\end{align*}
	where $(i)$ follows from the matrix identity
	\begin{equation*}
		\left( I_k - R_t Q_t^{-1} R_t^\sT \right)^{-1} R_t Q_t^{-1} \left( Q_t - R_t^\sT R_t \right) = R_t.
	\end{equation*}
	We finally obtain that
	\begin{align*}
		& \alpha \E_{(R_t, Q_t)} \left[ \E \left[ H_t \vert \bar{Z}_t, Y \right]^\sT \E \left[ H_t \vert \bar{Z}_t, Y \right] \right] \\
		= \, & \alpha \E \left[ \E \left[ G - M_t Z \Big\vert M_t^\sT G + Z, Y \right] \E \left[ G - M_t Z \Big\vert M_t^\sT G + Z, Y \right]^\sT \right].
	\end{align*}
	Therefore,
	\begin{equation*}
		M_{t+1} M_{t+1}^\sT \preceq \alpha \E \left[ \E \left[ G - M_t Z \Big\vert M_t^\sT G + Z, Y \right] \E \left[ G - M_t Z \Big\vert M_t^\sT G + Z, Y \right]^\sT \right],
	\end{equation*}
	and given $M_t$, any $M_{t+1}$ satisfying the above inequality is achievable by some $F_t$.
	
	Now we are in position to prove $(a)$ and $(b)$. By direct calculation, we know that
	\begin{equation*}
		R_1 \left( Q_1 - R_1^\sT R_1 \right)^{-1} R_1^\sT = M_1 M_1^\sT \preceq \alpha \E \left[ \E \left[ G \vert Y \right] \E \left[ G \vert Y \right]^\sT \right] = C_1 C_1^\sT,
	\end{equation*}
	and every pair $(R_1, Q_1)$ satisfying the above inequality is achievable by some choice of $F_0$. Here, the sequence $\{ C_t \}_{t=1}^{\infty}$ is defined in \cref{defn:Bayes_AMP_region_general}. Using induction, one can show that for any $t \in \mathbb{N}$,
	\begin{equation*}
		R_t \left( Q_t - R_t^\sT R_t \right)^{-1} R_t^\sT = M_t M_t^\sT \preceq C_t C_t^\sT,
	\end{equation*}
	and that any pair $(R_t, Q_t)$ satisfying the above inequality is achievable by some choice of $\{ F_s \}_{s=0}^{t-1}$.
	
	\paragraph{Proof of $(a)$.} The conclusion directly follows from the fact that
	\begin{equation*}
		C_{\sBayes} C_{\sBayes}^\sT = \lim_{t \to \infty} C_t C_t^\sT \succeq M_t M_t^\sT.
	\end{equation*}
	
	\paragraph{Proof of $(b)$.} For any $(R, Q) \in \operatorname{int} A_{\sBayes}$, there exists $T_{11} \in \mathbb{N}$ such that
	\begin{equation*}
		R \left( Q - R^\sT R \right)^{-1} R^\sT \preceq C_{T_{11}} C_{T_{11}}^\sT.
	\end{equation*}
	Hence, $(R, Q)$ is achievable by some sequence of Lipschitz functions $\{ F_t \}_{t \le T_{11} - 1}$, completing the proof of \cref{prop:Bayes_feasible_region}.
\end{proof}

\begin{proof}[Proof of \cref{lem:cov_mu_Q}]
	Since $F$ is an $(R, Q)$-contraction, we know that
	\begin{align*}
		R^\sT = \, \E \left[ \frac{\partial F}{\partial G} \left( Z_{R, Q}, \varphi (G, \veps) \right) \right], \quad Q = R^\sT R + \frac{1}{\alpha} \E \left[ F \left( Z_{R, Q}, \varphi (G, \veps) \right) F \left( Z_{R, Q}, \varphi (G, \veps) \right)^\sT \right].
	\end{align*}
	By our assumption, $(R_{T_{11}}, Q_{T_{11}}) = (R, Q)$, which further implies that $(Y, G, Z_{R, Q})$ has the same distribution as $(Y, \bar{Z}_0, \bar{Z}_{T_{11}})^\sT$. Combining this fact with $F_{T_{11}} = F$, we know that $(R_{T_{11}+1}, Q_{T_{11}+1}) = (R, Q)$. Repeating the same argument yields that $(R_t, Q_t) = (R, Q)$ for all $T_{11} + 1 \le t \le T_{1}$.
	
    Next we show that $\lim_{T_{1} \to \infty} P_{T_{1} - 1} = Q$. According to AMP state evolution, we obtain the following recursive relationship:
    \begin{align*}
		P_{t} = \, & \E \left[ \bar{Z}_{t+1}^\sT \bar{Z}_t \right] = R_{t+1}^\sT R_t + \frac{1}{\alpha} \E \left[F_t \left(\overline{Z}_t , Y\right)^\sT F_{t-1} \left(\overline{Z}_{t-1} , Y \right)\right] \\
		= \, & R^\sT R + \frac{1}{\alpha} \E \left[F \left(\overline{Z}_t , Y\right)^\sT F \left(\overline{Z}_{t-1} , Y \right)\right],
    \end{align*}
    which is a function of $P_{t-1}$ for fixed $(R, Q)$. We denote this function by $\psi$, then the above recursion can be written as $P_t = \psi ( P_{t-1} )$. We further note that, the sequence of Lipschitz functions $\{ F_t \}_{t \le T_{11} - 1}$ can be chosen such that $P_{T_{11} + 1} \succeq R^\sT R$. Therefore, in order to show that $\lim_{T_{1} \to \infty} P_{T_{1} - 1} = Q$, it suffices to prove that $\lim_{t \to \infty} \psi^t (P) = Q$ for all $P$ satisfying $R^\sT R \preceq P \preceq Q$. To this end, we first establish the following two claims:
    \begin{itemize}
	\item [(a)] $\psi$ is increasing with respect to the matrix Loewner order, i.e., for $R^\sT R \preceq A \preceq B \preceq Q$, we have $\psi (A) \preceq \psi (B)$.
	\item [(b)] For any $R^\sT R \preceq P \preceq Q$, $\psi (P) = P$ if and only if $P = Q$, i.e., $Q$ is the only fixed point of $\psi$ in the region $\{ P \in \mathbb{S}_+^m: R^\sT R \preceq P \preceq Q \}$.
    \end{itemize}
    \paragraph{Proof of (a).} Denote $H = B - A \succeq 0$, and define for $t \in [0, 1]$:
	\begin{equation*}
		\psi_{H} (t) = \psi(A + tH).
	\end{equation*}
	Then, it suffices to show that $\psi_{H} (1) \succeq \psi_{H} (0)$. Using Gaussian interpolation (see, e.g., \cite[Lemma 1.3.1]{talagrand2010mean}), we deduce that
    \begin{align*}
	\psi_{H}' (t) = \frac{1}{\alpha} \E_{A + tH} \left[ J_F \left( \overline{Z}, Y \right)^\sT H J_F \left( \overline{Z}' , Y \right) \right],
    \end{align*}
    where $J_F$ represents the Jacobian of $F$ taken with respect to its first argument. Since $A + tH \succeq R^\sT R$, we know that $\psi_{H}' (t) \succeq 0$. This proves that $\psi_{H} (1) \succeq \psi_{H} (0)$ and completes the proof of claim (a).
	
    \paragraph{Proof of (b).} First note that, since $\cuF_{m, \alpha, \varphi}^{\salg}$ is closed under weak limits, we can assume without loss of generality that \cref{eq:R_Q_contraction_condition} in \cref{defn:mu_Q_contraction} holds with strict inequality. Now, assume by contradiction that $\psi (P) = P$ for some $R^\sT R \preceq P \preceq Q$, $P \neq Q$, and set $H = Q - P$. Define for $\beta \in \R^m \backslash \{0 \}$ and $t \in [0, 1]$:
    \begin{equation*}
	\psi_{\beta, H} (t) = \beta^\sT \psi(P + tH) \beta.
    \end{equation*}
    Similar to the proof of (a), we can use Gaussian interpolation to show that $\psi_{\beta, H}' (t) \ge 0$ and $\psi_{\beta, H}'' (t) \ge 0$. Therefore
    \begin{align*}
	\beta^\sT \left( Q - P \right) \beta = \, & \beta^\sT \left( \psi(Q) - \psi(P) \right) \beta = \psi_{\beta, H} (1) - \psi_{\beta, H} (0) = \int_{0}^{1} \psi_{\beta, H}' (t) \d t \\
	\le \, & \psi_{\beta, H}' (1) = \frac{1}{\alpha} \E_{Q} \left[ \nabla F_{\beta} \left( \overline{Z}, Y \right)^\sT H \nabla F_{\beta} \left( \overline{Z}' , Y \right) \right] \\
	= \, & \frac{1}{\alpha} \beta^\sT \E_{Q} \left[ \JF \left( \overline{Z}, Y \right)^\sT H \JF \left( \overline{Z}', Y \right) \right] \beta.
    \end{align*}
    The above calculation implies that for all $\beta \in \R^m \backslash \{ 0 \}$,
    \begin{align*}
	& \beta^\sT H \beta = \beta^\sT \left( Q -  P \right) \beta \le \frac{1}{\alpha} \beta^\sT \E_{Q} \left[ \JF \left( \overline{Z}, Y \right)^\sT H \JF \left( \overline{Z}', Y \right) \right] \beta \\
	\Longleftrightarrow \, & \left\langle H, \beta \beta^\sT \right\rangle \le \left\langle \frac{1}{\alpha} \E_{Q} \left[ \JF \left( \overline{Z}, Y \right)^\sT H \JF \left( \overline{Z}', Y \right) \right], \beta \beta^\sT \right\rangle,
    \end{align*}
    thus leading to $\forall S \in \Sym_+^m \backslash \{ 0 \}$,
    \begin{align*}
	\left\langle H, S \right\rangle \le \left\langle \frac{1}{\alpha} \E_{Q} \left[ \JF \left( \overline{Z}, Y \right)^\sT H \JF \left( \overline{Z}', Y \right) \right], S \right\rangle = \left\langle H, \frac{1}{\alpha} \E_{Q} \left[ \JF \left( \overline{Z}, Y \right) S \JF \left( \overline{Z}', Y \right)^\sT \right] \right\rangle.
    \end{align*}
    Recall the triple $(Y, G, Z)$ and $Z_{R, Q}$ in \cref{defn:mu_Q_contraction}. It is easy to check that $\overline{Z} = \overline{Z}'$ has the same conditional distribution as $Z_{R, Q}^\sT$ given $Y$. Therefore,
    \begin{equation*}
        \frac{1}{\alpha} \E_{Q} \left[ \JF \left( \overline{Z}, Y \right) S \JF \left( \overline{Z}', Y \right)^\sT \right] = \, \frac{1}{\alpha} \E \left[ \frac{\partial F}{\partial Z_{R, Q}} \left( Z_{R, Q}, \varphi (G, \veps) \right)^\sT S \,\frac{\partial F}{\partial Z_{R, Q}} \left( Z_{R, Q}, \varphi (G, \veps) \right) \right].
    \end{equation*}
    Since $F$ is an $(R, Q)$-contraction, the above argument yields that for $S \succ 0$ in \cref{defn:mu_Q_contraction}:
    \begin{align*}
        \left\langle H, S \right\rangle \le \, & \left\langle H, \frac{1}{\alpha} \E_{Q} \left[ \JF \left( \overline{Z}, Y \right) S \JF \left( \overline{Z}', Y \right)^\sT \right] \right\rangle \\
        = \, & \left\langle H, \frac{1}{\alpha} \E \left[ \frac{\partial F}{\partial Z_{R, Q}} \left( Z_{R, Q}, \varphi (G, \veps) \right)^\sT S \,\frac{\partial F}{\partial Z_{R, Q}} \left( Z_{R, Q}, \varphi (G, \veps) \right) \right] \right\rangle < \left\langle H, S \right\rangle,
    \end{align*}
    a contradiction. This proves claim (b).
	
    We are now in position to show that $\lim_{t \to \infty} \psi^t (P) = Q$. Since $\psi$ is increasing and bounded, we know that $\{ \psi^t (P) \}$ is a bounded monotone (in Loewner order) sequence, and thus admits a unique limit $P_*$. Further, continuity implies that $P_*$ is a fixed point of $\psi$. By part (b), we must have $P_* = Q$. This completes the proof of \cref{lem:cov_mu_Q}.
\end{proof}

\subsection{Proof of Proposition~\ref{prop:state_evolution_iamp}}
\begin{proof}
	According to Proposition~\ref{prop:state_evolution}, we already know that $(Z_t)_{t \ge T_1 + 1}$ and $(\bar{Z}_t)_{t \ge T_1 + 1}$ are centered multivariate Gaussians, hence it suffices to show that for any $k \ge 1$,
    \begin{equation}\label{eq:induction_claim}
    \begin{split}
    	& (Z_t)_{T_1 + 1 \le t \le T_1 + k} \sim_{\iid} \normal (0, I_m), \quad (Z_t)_{T_1 + 1 \le t \le T_1 + k} \perp\!\!\!\perp ((Z_t)_{1 \le t \le T_1}, V), \\ 
    	& (\bar{Z}_t)_{T_1 + 1 \le t \le T_1 + k} \sim_{\iid} \normal (0, I_m), \quad (\bar{Z}_t)_{T_1 + 1 \le t \le T_1 + k} \perp\!\!\!\perp ((\bar{Z}_t)_{1 \le t \le T_1}, Y),
    \end{split}
    \end{equation}
    and $R_{T_1 + k+1} = 0$. We prove the above claim via induction on $k$. For $k=1$, using Eq.~\eqref{eq:covariance} and Assumption~\ref{ass:F_t_and_G_t}, we obtain that
    \begin{equation*}
    	\E \left[ Z_{T_1 + 1}^\sT Z_{T_1 + 1} \right] = \E \left[ F_{T_1} (Y)^\sT F_{T_1} (Y) \right] = I_m.    
    \end{equation*}
    Further, for any $1 \le t \le T_1$, we have
    \begin{align*}
    	\E \left[ Z_{T_1 + 1}^\sT Z_{t} \right] = \, & \E \left[ F_{T_1} \left( \varphi \left( \bar{Z}_0, \veps \right) \right)^\sT F \left( \bar{Z}_{t-1}, \varphi \left( \bar{Z}_0, \veps \right) \right) \right] \\
    	= \, & \E_{(R, Q)} \left[ F_{T_1} \left( \varphi \left( \bar{Z}_0, \veps \right) \right)^\sT F \left(\overline{Z} , \varphi \left( \bar{Z}_0, \veps \right) \right) \right] = \, 0.
    \end{align*}
    Therefore, $Z_{T_1 + 1} \sim \normal (0, I_m)$, and is independent of $((Z_t)_{1 \le t \le T_1}, V)$. Similarly, by Eq.~\eqref{eq:covariance} and Assumption~\ref{ass:F_t_and_G_t} we know that
    \begin{align*}
    	\E \left[ \bar{Z}_{T_1 + 1}^\sT \bar{Z}_{T_1 + 1} \right] = \, & \frac{1}{\alpha} \E \left[ G_{T_1 + 1} \left( V R_{T_1 + 1} + W^{T_1 + 1} \right)^\sT G_{T_1 + 1} \left( V R_{T_1 + 1} + W^{T_1 + 1} \right) \right] = I_m, \\
    	\E \left[ \bar{Z}_{T_1 + 1}^\sT \bar{Z}_0 \right] = \, & \frac{1}{\alpha} \E \left[ G_{T_1 + 1} \left( V R_{T_1 + 1} + W^{T_1 + 1} \right)^\sT V \right] = 0, \\
    	\E \left[ \bar{Z}_{T_1 + 1}^\sT \bar{Z}_t \right] = \, & \frac{1}{\alpha} \E \left[ G_{T_1 + 1} \left( V R_{T_1 + 1} + W^{T_1 + 1} \right)^\sT \left( V R + W^t \right) \right] \\
    	= \, & \frac{1}{\alpha} \E \left[ G_{T_1 + 1} \left( V R_{T_1 + 1} + W^{T_1 + 1} \right)^\sT W^t \right] = 0,
    \end{align*}
    where the last line follows from the fact that $W^t \perp\!\!\!\perp (V, W^{T_1 + 1})$.
    As a consequence, we deduce that $\bar{Z}_{T_1 + 1} \sim \normal (0, I_m)$ and is independent of $((\bar{Z}_t)_{1 \le t \le T_1}, Y)$. Moreover,
    \begin{align*}
    	R_{T_1 + 2} =\, & \mathbb{E}\left[ \frac{\partial F_{T_1 + 1}}{\partial \bar{z}_0} \left(\overline{Z}_{\le T_1 + 1} , \varphi \left( \bar{Z}_0, \veps \right) \right)\right] = \mathbb{E} \left[ \bar{Z}_{T_1 + 1} \frac{\partial \Phi_{T_1}}{\partial \bar{z}_0} \left( \bar{Z}_{\le T_1}, \varphi \left( \bar{Z}_0, \veps \right) \right)\right] \\
    	=\, & \mathbb{E} \left[ \bar{Z}_{T_1 + 1} \right] \E \left[ \frac{\partial \Phi_{T_1}}{\partial \bar{z}_0} \left( \bar{Z}_{\le T_1}, \varphi \left( \bar{Z}_0, \veps \right) \right)\right] = 0.
    \end{align*}
    This completes the base case of our induction.
    Now assume that our claim~\eqref{eq:induction_claim} holds for $k \in \mathbb{N}$. For $k+1$, we have
    \begin{align*}
    	& \E \left[ Z_{T_1 + k+1}^\sT Z_{T_1 + k+1} \right] = \E \left[ F_{T_1 + k} (\bar{Z}_{\le T_1 + k}, Y)^\sT F_{T_1+k} (\bar{Z}_{\le T_1+k}, Y) \right] \\
    	=\, & \E \left[ \Phi_{T_1+k-1} (\bar{Z}_{\le T_1+k-1}, Y)^\sT \bar{Z}_{T_1+k}^\sT \bar{Z}_{T_1+k} \Phi_{T_1+k-1} (\bar{Z}_{\le T_1+k-1}, Y) \right] \\
    	=\, & \E \left[ \Phi_{T_1 + k-1} (\bar{Z}_{\le T_1+k-1}, Y)^\sT \E \left[ \bar{Z}_{T_1+k}^\sT \bar{Z}_{T_1+k} \right] \Phi_{T_1+k-1} (\bar{Z}_{\le T_1+k-1}, Y) \right] \\
    	=\, & \E \left[ \Phi_{T_1+k-1} (\bar{Z}_{\le T_1+k-1}, Y)^\sT \Phi_{T_1+k-1} (\bar{Z}_{\le T_1+k-1}, Y) \right] = I_m,
    \end{align*}
    and for all $t \le T_1 + k$,
    \begin{align*}
    	& \E \left[ Z_{T_1+k+1}^\sT Z_{t} \right] = \E \left[ F_{T_1+k} (\bar{Z}_{\le T_1+k}, Y)^\sT F_{t-1} (\bar{Z}_{\le t-1}, Y) \right] \\
    	=\, & \E \left[ \Phi_{T_1+k-1} (\bar{Z}_{\le T_1+k-1}, Y)^\sT \bar{Z}_{T_1+k}^\sT F_{t-1} (\bar{Z}_{\le t-1}, Y) \right] \\
    	=\, & \E \left[ \Phi_{T_1+k-1} (\bar{Z}_{\le T_1+k-1}, Y)^\sT \E \left[ \bar{Z}_{T_1+k} \right]^\sT F_{t-1} (\bar{Z}_{\le t-1}, Y) \right] = 0.
    \end{align*}
    This proves $(Z_t)_{T_1 + 1 \le t \le T_1 + k + 1} \sim_{\iid} \normal (0, I_m)$, and are independent of $((Z_t)_{1 \le t \le T_1}, V)$. Proceeding similarly, we get that
    \begin{footnotesize}
    	\begin{align*}
    	& \E \left[ \bar{Z}_{T_1+k+1}^\sT \bar{Z}_{T_1+k+1} \right] =\, \frac{1}{\alpha} \mathbb{E}\left[G_{T_1+k+1} \left( V R_{\leq T_1+k+1} +Z_{\leq T_1+k+1} \right)^\sT G_{T_1+k+1} \left( V R_{\leq T_1+k+1} + Z_{\leq T_1+k+1} \right)\right] \\
    	\stackrel{(i)}{=}\, & \frac{1}{\alpha} \mathbb{E}\left[G_{T_1+k+1} \left( V R_{\leq T_1+1} +Z_{\leq T_1+1}, (Z_t)_{T_1+2 \le t \le T_1+k+1} \right)^\sT G_{T_1+k+1} \left( V R_{\leq T_1+1} +Z_{\leq T_1+1}, (Z_t)_{T_1+2 \le t \le T_1+k+1} \right) \right] \\
    	=\, & \frac{1}{\alpha} \mathbb{E}\left[ \Psi_{T_1+k} \left( V R_{\leq T_1+1} +Z_{\leq T_1+1}, (Z_t)_{T_1+2 \le t \le T_1+k} \right)^\sT Z_{T_1+k+1}^\sT Z_{T_1+k+1} \Psi_{T_1+k} \left( V R_{\leq T_1+1} +Z_{\leq T_1+1}, (Z_t)_{T_1+2 \le t \le T_1+k} \right) \right] \\
    	=\, & \frac{1}{\alpha} \mathbb{E}\left[ \Psi_{T_1+k} \left( V R_{\leq T_1+1} +Z_{\leq T_1+1}, (Z_t)_{T_1+2 \le t \le T_1+k} \right)^\sT \E \left[ Z_{T_1+k+1}^\sT Z_{T_1+k+1} \right] \Psi_{T_1+k} \left( V R_{\leq T_1+1} +Z_{\leq T_1+1}, (Z_t)_{T_1+2 \le t \le T_1+k} \right) \right] \\
    	=\, & \frac{1}{\alpha} \mathbb{E}\left[ \Psi_{T_1+k} \left( V R_{\leq T_1+1} +Z_{\leq T_1+1}, (Z_t)_{T_1+2 \le t \le T_1+k} \right)^\sT \Psi_{T_1+k} \left( V R_{\leq T_1+1} +Z_{\leq T_1+1}, (Z_t)_{T_1+2 \le t \le T_1+k} \right) \right] = I_m,
    \end{align*}
    \end{footnotesize}
where $(i)$ is because of $R_t = 0$ for $T_1+2 \le t \le T_1+k+1$. For any $1 \le s \le T_1 + k$, we deduce that
    \begin{align*}
    	\E \left[ \bar{Z}_{T_1+k+1}^\sT \bar{Z}_{s} \right] =\, & \frac{1}{\alpha} \mathbb{E}\left[G_{T_1+k+1} \left( V R_{\leq T_1+k+1}+Z_{\leq T_1+k+1} \right)^\sT G_{s} \left( V R_{\leq s} + Z_{\leq s} \right)\right] \\
    	=\, & \frac{1}{\alpha} \mathbb{E}\left[G_{T_1+k+1} \left( V R_{\leq T_1+1} +Z_{\leq T_1+1}, (Z_t)_{T_1+2 \le t \le T_1+k+1} \right)^\sT G_{s} \left( V R_{\leq s} + Z_{\leq s} \right)\right] \\
    	=\, & \frac{1}{\alpha} \mathbb{E}\left[\Psi_{T_1+k} \left( V R_{\leq T_1+1} +Z_{\leq T_1+1}, (Z_t)_{T_1+2 \le t \le T_1+k} \right)^\sT Z_{T_1+k+1}^\sT G_{s} \left( V R_{\leq s} + Z_{\leq s} \right)\right] \\
    	=\, & \frac{1}{\alpha} \mathbb{E}\left[\Psi_{T_1+k} \left( V R_{\leq T_1+1} +Z_{\leq T_1+1}, (Z_t)_{T_1+2 \le t \le T_1+k} \right)^\sT \E \left[ Z_{T_1+k+1} \right]^\sT G_{s} \left( V R_{\leq s} + Z_{\leq s} \right)\right] = 0.
    \end{align*}
	Furthermore,
	\begin{align*}
		\E \left[ \bar{Z}_{T_1+k+1}^\sT \bar{Z}_0 \right] =\, & \frac{1}{\alpha} \mathbb{E}\left[G_{T_1+k+1} \left( V R_{\leq T_1+k+1}+Z_{\leq T_1+k+1} \right)^\sT V \right] \\
		=\, & \frac{1}{\alpha} \mathbb{E}\left[\Psi_{T_1+k} \left( V R_{\leq T_1+1} +Z_{\leq T_1+1}, (Z_t)_{T_1+2 \le t \le T_1+k} \right)^\sT Z_{T_1+k+1}^\sT V \right] \\
		=\, & \frac{1}{\alpha} \mathbb{E}\left[\Psi_{T_1+k} \left( V R_{\leq T_1+1}+Z_{\leq T_1+1}, (Z_t)_{T_1+2 \le t \le T_1+k} \right)^\sT \E \left[ Z_{T_1+k+1} \right]^\sT V \right] = 0.
	\end{align*}
	This proves that $(\bar{Z}_t)_{T_1 + 1 \le t \le T_1 + k + 1} \sim_{\iid} \normal (0, I_m)$, and are independent of $((\bar{Z}_t)_{1 \le t \le T_1}, Y)$. Finally, we need to show that $R_{T_1+k+2} = 0$. Using Eq.~\eqref{eq:covariance}, it follows that
	\begin{align*}
		R_{T_1+k+2} =\, & \mathbb{E}\left[ \frac{\partial F_{T_1+k+1}}{\partial \bar{z}_0} \left(\overline{Z}_{\leq T_1+k+1} , \varphi \left( \bar{Z}_0, \veps \right) \right)\right] = \mathbb{E}\left[ \bar{Z}_{T_1+k+1} \frac{\partial \Phi_{T_1+k}}{\partial \bar{z}_0} \left(\overline{Z}_{\leq T_1+k} , \varphi \left( \bar{Z}_0, \veps \right) \right)\right] \\
		=\, & \mathbb{E}\left[ \bar{Z}_{T_1+k+1} \right] \E \left[ \frac{\partial \Phi_{T_1+k}}{\partial \bar{z}_0} \left(\overline{Z}_{\leq T_1+k} , \varphi \left( \bar{Z}_0, \veps \right) \right)\right] = 0.
	\end{align*}
	This completes the induction step and the proof of the proposition.
\end{proof}

\subsection{Proof of Theorem~\ref{thm:iamp_feasibility}}
\begin{proof}
By our assumption, we know that for all $1 \le s, t \le T$, $G_t G_s$ is pseudo-Lipschitz of order $2$, hence we have almost surely,
\begin{align*}
    \WW_I^\sT \WW_I =\, & \frac{1}{n} \sum_{t=1}^{T_2} \sum_{s=1}^{T_2} Q_t^\sT G_{T_1 + t + 1} \left( \WW^{\le T_1 + t + 1} \right)^\sT G_{T_1+s+1} \left( \WW^{\le T_1+s+1} \right) Q_s \\
    =\, & \sum_{t=1}^{T_2} \sum_{s=1}^{T_2} Q_t^\sT \left( \frac{1}{n} \sum_{i=1}^{d} G_{T_1+t+1} (\ww_i^{\le T_1+t+1})^\sT G_{T_1+s+1} (\ww_i^{\le T_1+s+1}) \right) Q_s \\
    \to\, & \frac{1}{\alpha} \sum_{t=1}^{T_2} \sum_{s=1}^{T_2} Q_t^\sT \E \left[ G_{T_1+t+1} \left( V R_{\leq T_1+t+1}+Z_{\leq T_1+t+1} \right)^\sT G_{T_1+s+1} \left( V R_{\leq T_1+s+1} +Z_{\leq T_1+s+1} \right) \right] Q_s \\
    =\, & \sum_{t=1}^{T_2} \sum_{s=1}^{T_2} Q_t^\sT \E \left[ \bar{Z}_{T_1+t+1}^\sT \bar{Z}_{T_1+s+1} \right] Q_s = \sum_{t=1}^{T_2} Q_t^\sT Q_t = I_m - Q.
\end{align*}
Similarly, we can show that $\WW_F^\sT \WW_I \to 0$ almost surely. Therefore, $\WW_Q^\sT \WW_Q \to I_m$ almost surely as $n \to \infty$. Further, one can show that $\WW_*^\sT \hat{\WW}_n^{\sAMP} \to R$ almost surely. Using Slutsky's theorem, it now suffices to consider the empirical joint distribution of the rows of $(\yy, \XX \WW_Q )$. By direct calculation, we obtain that
\begin{align*}
	\XX \WW_{Q} = \, & \XX \WW_{F} + \XX \WW_{I} = \VV^{T_1} + \frac{d}{n} F(\VV^{T_1-1}, \yy) + \frac{1}{\sqrt{n}} \sum_{t=1}^{T_2} \XX G_{T_1 + t+1} \left( \WW^{\le T_1 + t + 1} \right) Q_t \\
	= \, & \VV^{T_1} + \frac{d}{n} F(\VV^{T_1-1}, \yy) + \sum_{t=1}^{T_2} \left( \VV^{T_1 + t+1} + \sum_{s=1}^{T_1+t+1} F_{s-1} (\VV^{\le s-1}, \yy) D_{T_1+t+1, s}^\sT \right) Q_t,
\end{align*}
where by state evolution,
\begin{align*}
    D_{T_1+t+1, s} = \, & \frac{1}{n} \sum_{i=1}^{d} \frac{\partial G_{T_1+t+1}}{\partial \ww_i^s} (\ww_i^1, \cdots, \ww_i^{T_1+t+1}) \\
    \stackrel{a.s.}{\to} \, & \frac{1}{\alpha} \E \left[ \frac{\partial G_{T_1+t+1}}{\partial w^s} \left( V R_{\leq T_1+t+1} +Z_{\leq T_1+t+1} \right) \right] \\
    = \, & \frac{\bone_{s = T_1+t+1}}{\alpha} \E \left[ \Psi_{T_1+t} \left( V R_{\leq T_1+1}+Z_{\leq T_1+1}, (Z_t)_{T_1+2 \le t \le T_1+t} \right) \right]^\sT = \bone_{s = T_1+t+1} A_t^\sT,
\end{align*}
where the last row follows from the definition of $G_{T_1+t+1}$.
Therefore, it suffices to consider the limiting empirical joint distribution of the rows of
\begin{equation*}
	\left( \yy, \ \VV^{T_1} + \frac{1}{\alpha} F(\VV^{T_1-1}, \yy) + \sum_{t=1}^{T_2} \left( \VV^{T_1 + t+1} + F_{T_1+t} \left( \VV^{\le T_1+t}, \yy \right) A_t \right) Q_t \right),
\end{equation*}
which almost surely weakly converges to (using continuous mapping theorem)
\begin{equation*}
	\operatorname{Law} \left( Y, \ \bar{Z}_{T_1} + \frac{1}{\alpha} F \left( \bar{Z}_{T_1-1}, Y \right) + \sum_{t=1}^{T_2} \left( \bar{Z}_{T_1 + t+1} + F_{T_1+t} \left( \bar{Z}_{\le T_1+t}, Y \right) A_t \right) Q_t \right).
\end{equation*}
This proves the first part of the theorem.

Finally, note that for the IAMP stage, the only requirement for the function $\Psi_{T_1+t}$ is that
\begin{equation*}
    \E \left[ \Psi_{T_1+t} \left( V R_{\leq T_1+1}+Z_{\leq T_1+1}, (Z_t)_{T_1+2 \le t \le T_1+t} \right)^\sT \Psi_{T_1+t} \left( V R_{\leq T_1+1}+Z_{\leq T_1+1}, (Z_t)_{T_1+2 \le t \le T_1+t} \right) \right] = \alpha I_m.
\end{equation*}
Hence, for any $t \ge 1$ and $A_{t} \in \R^{m \times m}$ such that $A_{t}^\sT A_{t} \preceq I_m / \alpha$, there exists a function $\Psi_{T_1+t}$ that satisfies the condition of this theorem and that
\begin{equation*}
	\E \left[ \Psi_{T_1+t} \left( V R_{\leq T_1+1}+Z_{\leq T_1+1}, (Z_t)_{T_1+2 \le t \le T_1+t} \right) \right] = \alpha A_t.
\end{equation*}
This proves that if $F$ and $F_{T_1 + t}$ are continuous, then
\begin{equation*}
    \operatorname{Law} \left( Y, \ \bar{Z}_{T_1} + \frac{1}{\alpha} F \left( \bar{Z}_{T_1-1}, Y \right) + \sum_{t=1}^{T_2} \left( \bar{Z}_{T_1 + t+1} + F_{T_1+t} \left( \bar{Z}_{\le T_1+t}, Y \right) A_t \right) Q_t \right)
\end{equation*}
is $(\alpha, m)$-feasible and can be achieved by our two-stage AMP algorithm, where the only constraint on $\{ A_t \}_{1 \le t \le T_2}$ is that $A_t^\sT A_t \preceq I_m / \alpha$. The second part of Theorem~\ref{thm:iamp_feasibility} follows immediately by combining this result and the fact that the set of $(\alpha, m)$-feasible distributions  achievable by our AMP algorithm is closed under weak limits, since we can approximate general $L^2$-integrable functions by continuous functions to arbitrary accuracy.
\end{proof}

\subsection{Auxiliary lemmas}
\begin{lem}\label{lem:matrix_cauchy}
	Let $A \in \R^{1 \times k}$ and $X \in \R^{1 \times m}$ be two random vectors. Then, we have
	\begin{equation*}
		\left\{ \E [ A^\sT X ]: \E [ X^\sT X ] \preceq I_m \right\} = \, \left\{ R \in \R^{k \times m}: R R^\sT \preceq \E [A^\sT A] \right\}.
	\end{equation*}
\end{lem}

\begin{proof}
	Define the following two subsets of $\R^{k \times m}$:
	\begin{equation*}
		\mathcal{Y}_A = \, \left\{ \E [ A^\sT X ]: \E [ X^\sT X ] \preceq I_m \right\}, \quad \mathcal{Z}_A = \, \left\{ R \in \R^{k \times m}: R R^\sT \preceq \E [A^\sT A] \right\}.
	\end{equation*}
	It is easy to see that both $\mathcal{Y}_A$ and $\mathcal{Z}_A$ are closed and convex. By Hahn-Banach theorem, it suffices to show that for any $B \in \R^{k \times m}$:
	\begin{equation}\label{eq:Y_Z_equal_by_B}
		\sup_{Y \in \mathcal{Y}_A} \Tr ( B^\sT Y ) = \sup_{Z \in \mathcal{Z}_A} \Tr ( B^\sT Z ).
	\end{equation}
	To this end, we compute these two suprema respectively. For the first supremum, we have
	\begin{align*}
		\sup_{Y \in \mathcal{Y}_A} \Tr ( B^\sT Y ) = \, & \sup_{\E [ X^\sT X ] \preceq I_m} \Tr \left( B^\sT \E [ A^\sT X ] \right) = \sup_{\E [ X^\sT X ] \preceq I_m} \Tr \left( \E \left[ (A B)^\sT X \right] \right) \\
		= \, & \sup_{\E [ X^\sT X ] \preceq I_m} \E \left[ \Tr \left( (A B)^\sT X \right) \right] = \sup_{\E [ X^\sT X ] \preceq I_m} \E \left[ X (A B)^\sT \right] \\
		= \, & \sup_{X} \min_{\Lambda \succeq 0} \left\{ \E \left[ X (A B)^\sT \right] - \Tr \left( \Lambda \left( \E [ X^\sT X ] - I_m \right) \right) \right\} \\
		= \, & \min_{\Lambda \succeq 0} \sup_{X} \left\{ \E \left[ X (A B)^\sT \right] - \E \left[ X \Lambda X^\sT \right] + \Tr (\Lambda) \right\} \\
		= \, & \min_{\Lambda \succeq 0} \left\{ \frac{1}{4} \Tr \left( \Lambda^{-1} \E \left[ (AB)^\sT AB \right] \right) + \Tr (\Lambda) \right\} \\
		= \, & \min_{\Lambda \succeq 0} \left\{ \frac{1}{4} \Tr \left( \Lambda \E \left[ (AB)^\sT AB \right] \right) + \Tr \left( \Lambda^{-1} \right) \right\} \\
		= \, & \Tr \left( \E \left[ (AB)^\sT AB \right]^{1/2} \right) = \Tr \left( B^\sT \E [ A^\sT A ] B \right)^{1/2}.
	\end{align*}
	For the second one, the caculation follows similarly:
	\begin{align*}
		\sup_{Z \in \mathcal{Z}_A} \Tr ( B^\sT Z ) = \, & \sup_{R R^\sT \preceq \E [A^\sT A]} \Tr (B^\sT R) = \sup_{R} \min_{\Lambda \succeq 0} \left\{ \Tr (B^\sT R) - \Tr \left( \Lambda \left( R R^\sT - \E [A^\sT A] \right) \right) \right\} \\
		= \, & \min_{\Lambda \succeq 0} \sup_{R} \left\{ \Tr (B^\sT R) - \Tr \left( \Lambda \left( R R^\sT - \E [A^\sT A] \right) \right) \right\} \\
		= \, & \min_{\Lambda \succeq 0} \sup_{R} \left\{ \Tr (B^\sT R) - \Tr ( \Lambda R R^\sT ) + \Tr \left( \Lambda \E [A^\sT A] \right) \right\}  \\
		= \, & \min_{\Lambda \succeq 0} \left\{ \frac{1}{4} \Tr \left( B^\sT \Lambda^{-1} B \right) + \Tr \left( \Lambda \E [A^\sT A] \right) \right\} \\
		= \, & \min_{\Gamma \succeq 0} \left\{ \frac{1}{4} \Tr \left( \Gamma \E [A^\sT A]^{1/2} B B^\sT \E [A^\sT A]^{1/2} \right) + \Tr \left( \Gamma^{-1} \right) \right\} \\
		= \, & \Tr \left( \E [A^\sT A]^{1/2} B B^\sT \E [A^\sT A]^{1/2} \right)^{1/2} = \Tr \left( B^\sT \E [ A^\sT A ] B \right)^{1/2}.
	\end{align*}
	This proves \cref{eq:Y_Z_equal_by_B}, and consequently $\mathcal{Y}_A = \mathcal{Z}_A$.
\end{proof}

\section{Appendix for \cref{sec:char_opt_ctrl}}

\subsection{Proof of \cref{thm:variation_compute}}\label{sec:proof_variation}

\begin{proof}
We begin with defining a family of SDEs related to $(X_t)_{t \in [r^2, 1]}$: For any fixed $(s, y, x) \in [0, 1] \times \R \times \R$, consider the SDE
\begin{equation*}
    \d X_t^{y, x} = \mu(t) \partial_x f_{y, \mu} (t, X_t^{y, x}) \d t + \d B_t, \, t \in [s, 1]
\end{equation*}
with initial condition $X_s^{y, x} = x$.
Similar to \cite{montanari2024exceptional}, we can show that the above SDE has a unique solution $\{ X_t ^{y, x} \}_{t \in [s, 1]}$, which follows from the boundedness of $\partial_x f_{y, \mu}$ and \cite[Proposition 1.10]{cherny2005singular}. As a consequence, we know that $X_t = X_t^{Y, rG}$ with $s = r^2$ exists and is unique.

Further, define for any $t \in [s, 1]$:
\begin{equation*}
    M_t^{y, x} = \, \frac{1}{\gamma(t)} \partial_x f_{y, \mu} (t, X_t^{y, x}) + X_t^{y, x}.
\end{equation*}
Then, applying \cite[Lemma 7.3]{montanari2024exceptional} and the martingale representation theorem, we know that there exists $(\phi_t^{y, x})_{t \in [s, 1]} \in D [s, 1]$ such that
\begin{equation*}
    M_t^{y, x} = \, M_s^{y, x} + \int_{s}^{t} \left( 1 + \phi_u^{y, x} \right) \d B_u, \, \forall t \in [s, 1].
\end{equation*}
This guarantees the existence of $(\phi_t)_{t \in [r^2, 1]}$ by letting $s = r^2$ and $\phi_t = \phi_t^{Y, rG}$.

\paragraph{Proof of (i).} Since $\mu \equiv 0$ on $[r^2, q]$, we know that
\begin{equation*}
    X_{q} = \, X_{r^2} + B_q - B_{r^2},
\end{equation*}
which has the same distribution as $Z_{r, q}$ given $(Y, G)$. Hence, $(Y, G, X_q) \stackrel{d}{=} (Y, G, Z_{r, q})$. To prove the second part, note that by \cref{prop:veri_arg_supervised} and our choice of $F$:
\begin{align*}
    f_{y, \mu} (q, x) + \frac{1}{2 \alpha} \int_{q}^{1} \gamma(s) \d s = \, & \sup_{u \in \R} \left\{ V_{\gamma} (q, y, x+u) - \frac{\gamma(q)}{2} u^2 \right\} \\
    = \, & V_{\gamma} \left( q, y, x + \frac{1}{\alpha} F (x, y) \right) - \frac{\gamma(q)}{2 \alpha^2} F(x, y)^2.
\end{align*}
Further, similar to the verification argument in \cite{montanari2024exceptional}, we know that the supremum in the definition of $V_{\gamma} ( q, y, x + F (x, y) / \alpha )$ is achieved at $(\phi_t^{y, x})_{t \in [q, 1]}$, i.e.,
\begin{align*}
    & V_{\gamma} \left( q, y, x + \frac{1}{\alpha} F (x, y) \right) \\
    = \, & \E \left[ h \left( y, x + \frac{1}{\alpha} F (x, y) + \int_{q}^{1} \left( 1 + \phi_t^{y, x} \right) \d B_t \right) - \frac{1}{2} \int_{q}^{1} \gamma(t) \left( (\phi_t^{y, x})^2 - \frac{1}{\alpha} \right) \d t \right].
\end{align*}
From the proof of \cref{thm:two_stage_strong_duality}, we get that
\begin{equation*}
    \mathsf{F} (\mu, c, r) = \, \E \left[ f_{Y, \mu} \left( q, Z_{r, q} \right) \right] + \frac{1}{2 \alpha} \int_{r^2}^{1} \gamma(t) \d t,
\end{equation*}
which leads to
\begin{align*}
    \mathsf{F} (\mu, c, r) = \, & \E \left[ V_{\gamma} \left( q, Y, X_q + \frac{1}{\alpha} F (X_q, Y) \right) - \frac{\gamma(q)}{2 \alpha} \left( \frac{F(X_q, Y)^2}{\alpha} - (q - r^2) \right) \right] \\
    = \, & \E \bigg[ h \left( Y, X_q + \frac{1}{\alpha} F (X_q, Y) + \int_{q}^{1} \left( 1 + \phi_t^{Y, X_q} \right) \d B_t \right) - \frac{1}{2} \int_{q}^{1} \gamma(t) \left( \left( \phi_t^{Y, X_q} \right)^2 - \frac{1}{\alpha} \right) \d t \\
    & \quad - \frac{\gamma(q)}{2 \alpha} \left( \frac{F(X_q, Y)^2}{\alpha} - (q - r^2) \right) \bigg].
\end{align*}
The proof is then completed by noting that $\phi_t^{Y, X_q} = \phi_t$ for $t \in [q, 1]$.

\paragraph{Proof of (ii).} For any fixed pair $(Y, G)$, the conclusion follows similarly as Proposition 5.4 (ii) and Proposition 8.2 of \cite{montanari2024exceptional}. Applying the law of total expectation then completes the proof.

\paragraph{Proof of (iii).}
For the proof of this part, we need the following proposition which computes the first-order variations of $f_{y, \mu} (t, x)$ with respect to $\mu$ and $c$ (note that $f_{y, \mu} (t, x)$ depends on $c$ through the terminal condition), whose proof is established by Proposition 5.4 (iii), Proposition 8.3 and Lemma 8.4 in \cite{montanari2024exceptional}\footnote{Although \cite[Lemma 8.4]{montanari2024exceptional} is only stated for the setting $1 / c \le \sup_{z \in \R} \partial_z^2 h (y, z)$, its proof actually works for all $c > 0$.}.
\begin{prop}\label{prop:variation_compute_1}
    Fix $(t, y, x) \in [0, 1] \times \R \times \R$. Assume that $\delta: [t, 1] \to \R $ is in $L^1 [t, 1]$ and $L^{\infty} [t, s]$ for any $s \in [t, 1)$. Then, we have
    \begin{equation*}
    \begin{split}
        \frac{\d}{\d u} f_{y, \mu + u \delta} (t, x) \bigg\vert_{u = 0} = \, & \frac{1}{2} \int_{t}^{1} \delta (s) \E \left[ \left( \partial_x f_{y, \mu} (s, X_s^{y, x}) \right)^2 \right] \d s, \\
        \frac{\d}{\d c} f_{y, \mu} (t, x) = \, & \E \left[ g_c \left(y, M_1^{y, x} \right) \right] + \frac{1}{2} \E \left[ \left( \partial_x f_{y, \mu} \left(1, X_1^{y, x} \right) \right)^2 \right].
    \end{split}
    \end{equation*}
\end{prop}
Based on the above proposition, we are now able to compute the first-order variations of $\mathsf{F}$ with respect to $\mu$ and $c$. Note that, for any fixed $(Y, G)$, \cref{prop:variation_compute_1} implies that
\begin{equation*}
\begin{split}
    \frac{\d}{\d u} f_{Y, \mu + u \delta} (r^2, rG) \bigg\vert_{u = 0} = \, & \frac{1}{2} \int_{r^2}^{1} \delta (s) \E \left[ \left( \partial_x f_{Y, \mu} (s, X_s^{Y, rG}) \right)^2 \right] \d s, \\
        \frac{\d}{\d c} f_{Y, \mu} (r^2, rG) = \, & \E \left[ g_c \left(Y, M_1^{Y, rG} \right) \right] + \frac{1}{2} \E \left[ \left( \partial_x f_{Y, \mu} \left(1, X_1^{Y, rG} \right) \right)^2 \right].
\end{split}
\end{equation*}
Using \cref{thm:Parisi_solution_supervised} and \cref{lem:bounded_g_c}, we obtain that
\begin{align*}
    \left\vert \frac{\d}{\d u} f_{Y, \mu + u \delta} (r^2, rG) \bigg\vert_{u = 0} \right\vert \le \, & \frac{1}{2} \norm{\delta}_{L^1 [r^2, 1]} \norm{\partial_x h(Y, \cdot)}_{\infty}^2, \\
    \left\vert \frac{\d}{\d c} f_{Y, \mu} (r^2, rG) \right\vert \le \, & \frac{5}{2} \norm{\partial_x h(Y, \cdot)}_{\infty}^2.
\end{align*}
Of course, the above estimates hold within a neighborhood of $(\mu, c)$.
Since $\E [ \norm{\partial_x h(Y, \cdot)}_{\infty}^2 ] < \infty$ by our assumption, we can apply the dominated convergence theorem to deduce that
\begin{equation*}
\begin{split}
    \frac{\d}{\d u} \E \left[ f_{Y, \mu + u \delta} (r^2, rG) \right] \bigg\vert_{u = 0} = \, & \frac{1}{2} \int_{r^2}^{1} \delta (s) \E \left[ \left( \partial_x f_{Y, \mu} (s, X_s) \right)^2 \right] \d s, \\
        \frac{\d}{\d c} \E \left[ f_{Y, \mu} (r^2, rG) \right] = \, & \E \left[ g_c \left(Y, M_1 \right) \right] + \frac{1}{2} \E \left[ \left( \partial_x f_{Y, \mu} \left(1, X_1 \right) \right)^2 \right],
\end{split}
\end{equation*}
which concludes the calculation of the first two variations, since the variations of the entropic term follow from a straightforward calculation.

Finally, we compute the first-order variation of $\mathsf{F}$ with respect to $r$. It suffices to show that
\begin{equation*}
    \frac{\d}{\d r} \E \left[ f_{Y, \mu} (r^2, rG) \right] = \, \frac{\gamma (r^2)}{\alpha} \frac{1}{q - r^2} \E \left[ \left( q G - r Z_{r, q} \right) F \left( Z_{r, q}, \varphi(G, \veps) \right) \right].
\end{equation*}
To this end, note that since $\mu \equiv 0$ on $[r^2, q]$, we have
\begin{equation*}
    f_{Y, \mu} (r^2, rG) = \, f_{Y, \mu} (q, X_q) \stackrel{d}{=} f_{Y, \mu} (q, Z_{r, q}).
\end{equation*}
For any fixed $(Y, G, Z)$,
\begin{align*}
    \frac{\d}{\d r} f_{Y, \mu} (q, Z_{r, q}) = \, \frac{\d}{\d r} f_{Y, \mu} \left( q, rG + \sqrt{q - r^2} Z \right) = \partial_x f_{Y, \mu} (q, Z_{r, q}) \left( G - \frac{r}{\sqrt{q - r^2}} Z \right),
\end{align*}
which is uniformly bounded by $\norm{\partial_x h(Y, \cdot)}_{\infty} \vert G - rZ / \sqrt{q - r^2} \vert \in L^1$, since $\norm{\partial_x h(Y, \cdot)}_{\infty} \in L^2$ by our assumption.
Applying again dominated convergence theorem, we obtain that
\begin{align*}
    \frac{\d}{\d r} \E \left[ f_{Y, \mu} (r^2, rG) \right] = \, & \frac{\d}{\d r} \E \left[ f_{Y, \mu} (q, Z_{r, q}) \right] = \E \left[ \partial_x f_{Y, \mu} (q, Z_{r, q}) \left( G - \frac{r}{\sqrt{q - r^2}} Z \right) \right] \\
    = \, & \frac{\gamma(r^2)}{\alpha} \frac{1}{q - r^2} \E \left[ \left( q G - r Z_{r, q} \right) F \left( Z_{r, q}, \varphi(G, \veps) \right) \right],
\end{align*}
where the last line follows from the definition of $F$. This completes the proof.
\end{proof}

\subsection{Auxiliary lemmas}

\begin{lem}\label{lem:bounded_g_c}
    For any fixed $y \in \R$, it holds
    \begin{equation*}
        \norm{g_c (y, \cdot)}_{\infty} \le \, 2 \norm{\partial_x h(y, \cdot)}_{\infty}^2.
    \end{equation*}
\end{lem}
\begin{proof}
    For any $c > 0$, \cite[Lemma C.2]{montanari2024exceptional} implies that
    \begin{equation*}
        h_c (y, x) = \, \operatorname{conc}  \left( h (y, x) - \frac{x^2}{2 c} \right) + \frac{x^2}{2 c} = \sup_{U \in L^2 (\Omega), \, \E[U] = 0} \E \left[ h(y, x+U) - \frac{U^2}{2 c} \right].
    \end{equation*}
    Denoting $L = \| \partial_x h(y, \cdot) \|_{\infty}$, we first show that the above supremum must be achieved at some $U$ such that $\E [U^2] \le 4 c^2 L^2$. Indeed, such $U$ must satisfy
    \begin{align*}
        h(y, x) \le \, \E \left[ h(y, x+U) - \frac{U^2}{2 c} \right] \le h(y, x) + L \E [\vert U \vert] - \frac{1}{2 c} \E [U^2] \le h(y, x) + L \E [U^2]^{1/2} - \frac{1}{2 c} \E [U^2],
    \end{align*}
    which implies that $\E [U^2] \le 4 c^2 L^2$. As a consequence,
    \begin{equation*}
        h_c (y, x) = \, \sup_{\E [U^2] \le 4 c^2 L^2, \, \E[U] = 0} \E \left[ h(y, x+U) - \frac{U^2}{2 c} \right].
    \end{equation*}
    For any $0 < c_1 \le c_2$, we have
    \begin{align*}
        & \left\vert h_{c_1} (y, x) - h_{c_2} (y, x) \right\vert \\
        = \, & \left\vert \sup_{\E [U^2] \le 4 c_1^2 L^2, \, \E[U] = 0} \E \left[ h(y, x+U) - \frac{U^2}{2 c_1} \right] - \sup_{\E [U^2] \le 4 c_2^2 L^2, \, \E[U] = 0} \E \left[ h(y, x+U) - \frac{U^2}{2 c_2} \right] \right\vert \\
        \le \, & \sup_{\E [U^2] \le 4 c_2^2 L^2, \, \E[U] = 0} \left\vert \frac{\E[U^2]}{2 c_2} - \frac{\E[U^2]}{2 c_1} \right\vert \le \frac{2 c_2 (c_2 - c_1)}{c_1} L^2,
    \end{align*}
    leading to
    \begin{equation*}
        \left\vert g_c (y, x) \right\vert = \, \left\vert \frac{\partial h_c (y, x)}{\partial c} \right\vert \le 2 L^2, \, \forall x \in \R.
    \end{equation*}
    This completes the proof.
\end{proof}

\end{document}